\documentclass[twoside,11pt]{article}

\usepackage{blindtext}

\usepackage{jmlr2e}

\usepackage[T1]{fontenc}

\usepackage[svgnames,x11names,table]{xcolor}

\usepackage{booktabs}       
\usepackage{amsfonts}       
\usepackage{nicefrac}       
\usepackage{amsmath}

\usepackage{times}
\usepackage{epsfig}
\usepackage{amssymb}

\usepackage{algorithm}
\usepackage{algorithmic}
\usepackage{dsfont}

\usepackage[mathscr]{eucal}

\usepackage{tabularx}
\usepackage{multicol}
\usepackage{multirow}
\usepackage{rotating}
\usepackage{diagbox}
\usepackage{graphicx}
\usepackage{colortbl}
\usepackage{mleftright}
\usepackage{bbold}
\usepackage{adjustbox}
\usepackage{tikz}

\definecolor{asparagus}{rgb}{0.53, 0.66, 0.42}
\definecolor{applegreen}{rgb}{0.55, 0.71, 0.0}

\newcommand{\RowColorb}{\rowcolor{blue!15}}
\newcommand{\RowColorc}{\rowcolor{orange!15}}

\newcommand{\lratio}[1]{\setlength{\hsize}{#1\hsize}} 
\newcolumntype{Y}{>{\lratio{0.65}}>{\centering\arraybackslash}X}

\DeclareMathOperator*{\argmin}{\arg\!\min}

\makeatletter
\newcommand*{\shifttext}[2]{%
  \settowidth{\@tempdima}{#2}%
  \makebox[\@tempdima]{\hspace*{#1}#2}%
}
\makeatother

\newtheorem{assumption}[theorem]{Assumption}

\newcommand\independent{\protect\mathpalette{\protect\independenT}{\perp}}
\def\independenT#1#2{\mathrel{\rlap{$#1#2$}\mkern2mu{#1#2}}}

\newcommand{\source}[1]{\centering #1}

\newenvironment{code}[1]{
    \, \textbf{Code}: #1
}{}

\usepackage{lastpage}
\jmlrheading{26}{2025}{1-\pageref{LastPage}}{10/23; Revised 10/24}{1/25}{23-1359}{Antoine de Mathelin, François Deheeger, Mathilde Mougeot, Nicolas Vayatis}

\ShortHeadings{Deep Out-of-Distribution Uncertainty Quantification via Weight Entropy Maximization}{de Mathelin, Deheeger, Mougeot and Vayatis}
\firstpageno{1}

\begin{document}

\title{Deep Out-of-Distribution Uncertainty Quantification via \\ Weight Entropy Maximization}

\author{\name Antoine de Mathelin$^{1, 2}$ \email{antoine.de\_mathelin@ens-paris-saclay.fr} \\
       \name François Deheeger$^{1}$ \email francois.deheeger@michelin.com \\
       \name Mathilde Mougeot$^{2}$ \email mathilde.mougeot@ens-paris-saclay.fr \\
       \name Nicolas Vayatis$^{2}$ \email nicolas.vayatis@ens-paris-saclay.fr \\
       \addr $^{1}$Manufacture Française des pneumatiques Michelin,
       Clermont-Ferrand, 63000, France \\
       \addr $^{2}$Centre Borelli, Université Paris-Saclay, CNRS, ENS Paris-Saclay, Gif-sur-Yvette, 91190, France
       }

\editor{Maya Gupta}

\maketitle

\begin{abstract}%
This paper deals with uncertainty quantification and out-of-distribution detection in deep learning using Bayesian and ensemble methods. It proposes a practical solution to the lack of prediction diversity observed recently for standard approaches when used out-of-distribution \citep{ovadia2019CanYouTrustYourModel, Liu2021PerilDeepOOD}. Considering that this issue is mainly related to a lack of weight diversity, we claim that standard methods sample in ``over-restricted'' regions of the weight space due to the use of ``over-regularization'' processes, such as weight decay and zero-mean centered Gaussian priors. We propose to solve the problem by adopting the maximum entropy principle for the weight distribution, with the underlying idea to maximize the weight diversity. Under this paradigm, the epistemic uncertainty is described by the weight distribution of maximal entropy that produces neural networks ``consistent'' with the training observations. Considering stochastic neural networks, a practical optimization is derived to build such a distribution, defined as a trade-off between the average empirical risk and the weight distribution entropy. We provide both theoretical and numerical results to assess the efficiency of the approach. In particular, the proposed algorithm appears in the top three best methods in all configurations of an extensive out-of-distribution detection benchmark including more than thirty competitors.
\end{abstract}

\begin{code}
    \url{https://github.com/antoinedemathelin/maxwent-expe}
\end{code}

\section{Introduction}
\label{intro}

In many practical deep learning scenarios, neural network models are deployed on unknown data distributions that can significantly differ from the training distribution. For instance, when building deep learning models of object detection for autonomous cars, the training dataset cannot cover any potential situation that the model can encounter, in terms of weather conditions, geography, or camera obstructions for examples. In this context, the learner aims at providing confidence guarantees on the model prediction for any data belonging to the whole input space, including data outside the support of the training distribution. This task is related to \textit{epistemic uncertainty} quantification and \textit{out-of-distribution} (OOD) detection for deep learning \citep{kendall2017EpistemicUncertainties, Shen2021OODSurvey}. In the epistemic uncertainty quantification framework, the learner aims at estimating the potential discrepancy between the estimated hypothesis and the optimal predictor. When dealing with neural networks, the set of hypotheses is typically very large and many of them provide a low empirical risk on the training observations. Informally, this collection of hypotheses that are \textit{consistent} with the training data form a subset of relevant candidates for the optimal predictor. The prediction uncertainty for a novel input observation is then described by the prediction diversity of the consistent hypotheses \citep{hullermeier2021aleatoricEpistemic}.

In the case of universal approximators such as neural networks, a proxy of the epistemic uncertainty can be estimated by computing the distance to the support of the training set. For example, if the considered set of hypotheses is the set of $k$-Lipschitz functions, the pointwise prediction discrepancy between two consistent hypotheses is bounded by a value proportional to the distance to the training inputs  \citep{sullivan2013optimalUncertaintyLipschitz, Malherbe2017GlobalOptim, deMathelin2021DBAL}. Methods developed under this paradigm are referred to as \textit{distance-based} uncertainty quantifiers, which include, for instance, derivative of Gaussian processes \citep{rasmussen2003gaussianprocess}, Deterministic Uncertainty Quantification (DUQ) \citep{vanAmersfoort2020DUQ}, Mahalanobis distance \citep{lee2018MahalanobisOODdetect} or Deep Nearest Neighbors \citep{Sun2022KNNOOD}. The main challenge faced by distance-based uncertainty approaches is to find a relevant notion of distance to use \citep{Liu2022DistanceAwarness}. For high-dimensional machine learning problems, using the Euclidean distance in the input space is generally irrelevant and one looks for geometric distances computed in encoded spaces. For instance, \citet{Liu2022DistanceAwarness} and \citet{Cao2022deepHybridModelOOD} develop distance preserving networks using spectral normalization. Finally, computing the distance to the training distribution support can also be performed by density estimation techniques, such as auto-encoders or GANs, which have been used for OOD detection \citep{zhou2022AutoEncoderOOD, ryu2018GANOOD}. The distance to the training set is then computed through the reconstruction error of the decoder or by the predicted likelihood of the discriminator.

The main alternative to the distance-based approach consists in directly looking for a set of hypotheses that are consistent with the training observations and to use the diversity of their predictions as uncertainties. It essentially includes ensemble and Bayesian methods \citep{Lakshminarayanan2017DeepEnsemble, mackay1992bayesianNetwork}. The ongoing challenge of this approach is to produce diversity in the ensemble of networks, i.e., to avoid sampling similar hypotheses. It has been observed, indeed, that most of the main baselines lead to a lack of prediction diversity, in particular outside the training support, i.e., for out-of-distribution data \citep{ovadia2019CanYouTrustYourModel, Liu2021PerilDeepOOD, Angelo2021BayesianNotSuited4OOD}. Facing this issue, several attempts propose to increase the prediction diversity by adding a penalizing term to the loss. For instance, negative correlation methods penalize the correlation between the outputs of the ensemble members on the training data \citep{liu1999negativecorrelation, shui2018negativecorrelation, zhang2020NegCorr}. Related methods, referred to as \textit{contrastive} approaches, penalize small output variances on synthetic OOD data produced by sampling uniformly in the input space \citep{Jain2020MOD, Mehrtens2022MODplus} or in the neighborhood of the training instances \citep{Lakshminarayanan2017DeepEnsemble, Segonne2022OODpseudoInputs}. The drawback of these methods is the lack of generalization to any OOD data that the model can encounter \citep{Cao2022deepHybridModelOOD}. Alternative approaches consist in penalizing the similarity between the ensemble members in the parameter space \citep{pearce2018AnchorNetwork, Angelo2021RepulsiveDeepEnsemble}, with the underlying assumption that an ensemble of neural networks with weights distant from each other produces diversified outputs. In this work, we advocate that the key feature for producing accurate uncertainty quantification for any input data point is to sample in the \textit{whole} space of consistent hypotheses. Indeed, we argue that standard Bayesian and ensemble methods often provide over-confident predictions for OOD data because the hypotheses they produce are sampled in restricted regions of the consistent hypothesis space due to weight decay regularization and hyper-parameters selection based on hold-out validation.




Considering stochastic neural networks with parameterized weight distribution \citep{jospin2022BNNTutorial}, we cast the problem as a trade-off between sampling in low empirical risk regions and increasing the weight diversity. We consider the entropy as a measure of weight diversity, and show that the optimization boils down to solving a maximum entropy problem \citep{jaynes1968priorprobability}, where we aim at selecting the weight distribution of maximal entropy under the constraint that the training loss is acceptable. We derive a practical optimization formulation to solve this problem, called Maximum Weight Entropy (\textbf{MaxWEnt}), and show that it can be tackled with stochastic variational inference \citep{hoffman2013stochasticVI} using the reparameterization trick \citep{Kingma2013VAE}. The proposed optimization consists in penalizing the training loss with a term imposing the \textit{increase} of the weight distribution entropy. We provide a theoretical framework to understand the dynamic of this approach and show that the spread of the weight distribution is inversely proportional to the neuron activation amplitude for the training data. Numerical experiments conducted on several regression and classification datasets demonstrate the strong benefit of this approach in OOD detection compared to state-of-the-art methods dedicated to this task (e.g., Figure \ref{toy-intro}).

\begin{figure}[h]
    \centering
    \begin{minipage}{0.24\linewidth}
        \centering
        \includegraphics[width=\linewidth]{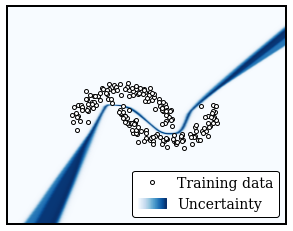} 
    \end{minipage}
    \begin{minipage}{0.24\linewidth}
        \centering
        \includegraphics[width=\linewidth]{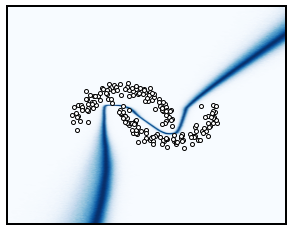} 
    \end{minipage}
    \begin{minipage}{0.24\linewidth}
        \centering
        \includegraphics[width=\linewidth]{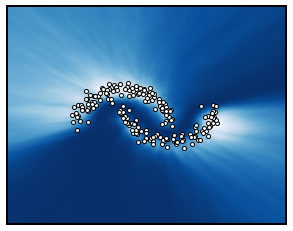} 
    \end{minipage}
        \begin{minipage}{0.24\linewidth}
        \centering
        \includegraphics[width=\linewidth]{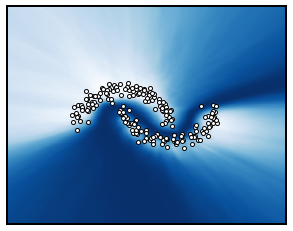} 
    \end{minipage} \\
    \begin{minipage}{0.24\linewidth}
        \centering
        \includegraphics[width=\linewidth]{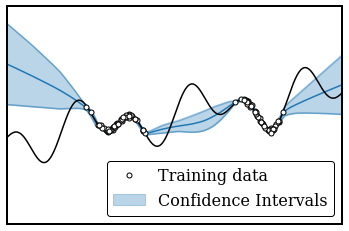} 
        \source{(a) Deep Ensemble}
    \end{minipage}
    \begin{minipage}{0.24\linewidth}
        \centering
        \includegraphics[width=\linewidth]{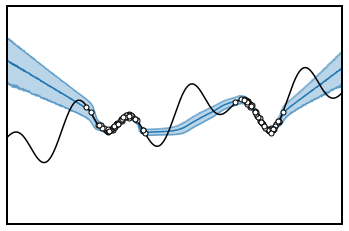} 
        \source{(b) MC-Dropout}
    \end{minipage}
    \begin{minipage}{0.24\linewidth}
        \centering
        \includegraphics[width=\linewidth]{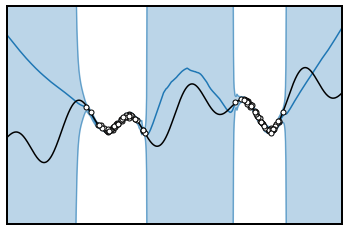} 
        \source{(c) MaxWEnt}
    \end{minipage}
    \begin{minipage}{0.24\linewidth}
        \centering
        \includegraphics[width=\linewidth]{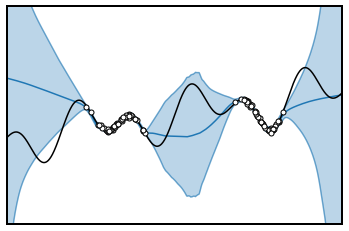} 
        \source{(d) MaxWEnt + Clip}
    \end{minipage}
    \caption{Uncertainty Estimation Comparison. Above: ``two-moons'' 2D classification dataset. Below: 1D-regression \citep{Jain2020MOD}. Deep Ensemble (a) and MC-Dropout (b) produce overconfident estimations outside the training support due to a lack of hypothesis diversity. In the classification experiment, the hypotheses produced by both methods are restricted to half-space separators. There is no prediction uncertainty in the upper left and lower right areas of the input space, despite the lack of training data in these regions. In contrast, MaxWEnt (c, d) provides a clear discrimination between the in-distribution and out-of-distribution domains in terms of prediction uncertainty. Figure (c) presents the result obtained with MaxWEnt when no regularity assumption is made on the labeling function. In this case, the uncertainty quickly increases when leaving the training support, which truly represents the epistemic uncertainty in the absence of prior knowledge about the labeling function. Figure (d) reports the MaxWEnt uncertainty estimation when considering Lipschitz constraints. This result can be obtained with a small modification of the previous MaxWEnt model in the form of weight clipping. The full description of these synthetic experiments is reported in Section \ref{synth-expe}.}
    \label{toy-intro}
\end{figure}



\section{Assessing Epistemic Uncertainty for OOD Prediction: Setup and Challenges}

\subsection{Notations}
\label{main-notation}

We consider the supervised learning framework provided with the input space $\mathcal{X}$ of finite dimension $b \in \mathbb{N}$, and the output space $\mathcal{Y}$. We denote by $p^{*}\!(x, y)$ the ``ground truth'' law defined over $\mathcal{X} \times \mathcal{Y}$. Furthermore, we distinguish the \textit{in-distribution} and \textit{out-of-distribution} domains by considering that only a subset $\mathcal{D}_{\mathcal{X}} \subset \mathcal{X}$ can be sampled. The subset $\mathcal{D}_{\mathcal{X}}$ is called ``training domain'' and any data from the complementary $\mathcal{X} \setminus \mathcal{D}_{\mathcal{X}}$ is considered to be ``out-of-distribution''. We assume that the learner has access to the training set $\mathcal{S} = \{(x_1, y_1), ..., (x_n, y_n) \} \in \mathcal{D}_{\mathcal{X}} \times \mathcal{Y}$ of size $n \in \mathbb{N}$ where the training instances $(x_i, y_i)$ are supposed independently identically distributed (iid) according to the joint distribution $p(x, y)$ defined over $\mathcal{D}_{\mathcal{X}} \times \mathcal{Y}$ and verifying $p(y|x) = p^{*}\!(y|x) \; \forall \, x \in \mathcal{D}_{\mathcal{X}}$. We consider a continuous loss function $\ell: \mathcal{Y} \times \mathcal{Y} \to \mathbb{R}_+$ and define the \textit{optimal predictor} $f^*: \mathcal{X} \to \mathcal{Y}$ as follows:
\begin{equation*}
\label{opt-predictor-eq}
    f^*(x) = \argmin_{y' \in \mathcal{Y}} \int_{y \in \mathcal{Y}} \ell(y', y) \, d p^*\!(y|x).
\end{equation*}
We denote $\mathcal{H}$ a set of neural networks of a specified architecture, mapping $\mathcal{X}$ to $\mathcal{Y}$. The set $\mathcal{H}$ is assumed to be ``large''. We denote $\mathcal{W} \subset \mathbb{R}^d$ ($d \in \mathbb{N}$) the set of weights corresponding to the hypotheses in $\mathcal{H}$. The best hypothesis in the class, $h^*$, referred to as the \textit{optimal hypothesis}, is defined according to the following expression:
\begin{equation*}
    h^* = \argmin_{h \in \mathcal{H}} \int_{x \in \mathcal{X}} \int_{y \in \mathcal{Y}} \ell(h(x), y) \, d p^{*}\!(x, y) .
\end{equation*}
Finally, for any $h \in \mathcal{H}$, we define the empirical risk as follows:
\begin{equation*}
\mathcal{L}_{\mathcal{S}}(h) = \frac{1}{n} \sum_{(x, y) \in \mathcal{S}} \ell(h(x), y),
\end{equation*}
denoted indifferently $\mathcal{L}_{\mathcal{S}}(w)$, when considering the weights $w \in \mathcal{W}$ associated to the hypothesis $h \in \mathcal{H}$, also referred as $h_w$.

\subsection{The epistemic uncertainty is described by the set of consistent hypotheses}
\label{sec-consistent-hypo}

The general purpose of uncertainty quantification is to provide, for any $x \in \mathcal{X}$, a distribution in the label space $\mathcal{Y}$, describing the potential outputs that can be associated with $x$, knowing the training observations $\mathcal{S}$ and prior information about $p^*(x, y)$. In this work, we distinguish the following four sources of uncertainty:

\begin{enumerate}
    \item \textbf{Aleatoric uncertainty}: the intrinsic random noise of the data, i.e., $p^*(y|x)$. This uncertainty cannot be reduced, even with an infinite number of observations (e.g., the outcome of a coin flip).
    \item \textbf{Model uncertainty}: the discrepancy between $f^*$ and $h^*$. The model uncertainty is related to the choice of the hypothesis set $\mathcal{H}$. It can be reduced by increasing the size of $\mathcal{H}$ or by acquiring prior knowledge about $f^*$ (e.g., Lipschitz constraints).
    \item \textbf{Statistical uncertainty}: the partial knowledge about $p(x, y)$ given by the finite number of data $\mathcal{S}$. This uncertainty, also referred to as \textit{approximation uncertainty} \citep{hullermeier2021aleatoricEpistemic} or \textit{data variability} \citep{huang2021quantifyingEpistemic}, is linked to the discrepancy between $h^*$ and its estimation. It can be reduced by the acquisition of novel observations drawn according to $p(x, y)$ or by prior knowledge about the intrinsic random noise (e.g., a Gaussian homoscedastic noise of known variance).
    \item \textbf{Out-of-distribution uncertainty}: the absence of observation over the out-of-distribution domain $\mathcal{X} \setminus \mathcal{D}_{\mathcal{X}}$. This uncertainty can remain large even with an infinite number of training observations. Indeed, for complex hypotheses as neural networks, different hypotheses can match $h^*(x)$ on $\mathcal{D}_{\mathcal{X}}$ but produce different outputs on $\mathcal{X} \setminus \mathcal{D}_{\mathcal{X}}$.
\end{enumerate}


The first three sources of uncertainty are described in detail in \citet{hullermeier2021aleatoricEpistemic}, while source (4) is an additional distinction of the epistemic uncertainty similar to the framework introduced in \citet{Liu2022DistanceAwarness}. The distinction between these four sources of uncertainty is useful for positioning our contribution. In this work, we focus essentially on the estimation of the statistical and OOD uncertainty. We assume, indeed, that $f^*$ is close to $\mathcal{H}$, i.e., $h^* \simeq f^*$, and then neglect the model uncertainty, which is a common assumption when considering a large set of hypotheses such as a set of deep neural networks \citep{hullermeier2021aleatoricEpistemic}. In the following, the term \textit{epistemic uncertainty} will refer to the combination of the statistical and OOD uncertainty.

It should be noted that, on the one hand, as we consider model uncertainty to be negligible, our approach may not effectively capture the epistemic uncertainty if this assumption does not hold. On the other hand, since our method is designed to address OOD uncertainty, it remains useful even when the statistical uncertainty is low (e.g., when a very large amount of training data is available). Indeed, multiple hypotheses can perform optimally on the training distribution while providing different predictions on OOD data. Our method specifically aims to capture this variability.


Because of lack of complete knowledge, the learner cannot perfectly determine the optimal hypothesis $h^*$ and then the optimal predictions $h^*(x)$. If no data is available, the prediction uncertainty for $x \in \mathcal{X}$ is given by the distribution of the predicted values $h(x)$ for all hypotheses $h \in \mathcal{H}$. When acquiring more observations, the learner can discriminate between relevant and irrelevant candidates for $h^*$, i.e., between ``consistent'' and ``inconsistent'' hypotheses with respect to the observations $\mathcal{S}$ (assuming that a notion of ``consistency'' can be formally defined). By denoting $\mathcal{H}_\mathcal{S}$ the set of \textit{consistent hypotheses}, the epistemic uncertainty for the prediction of the model for $x$ is then given by the distribution of predictions $h(x)$ with $h \sim \mathcal{H}_\mathcal{S}$.

The notion of consistency depends on the underlying assumptions that the user considers about the data sample $\mathcal{S}$. A strong assumption is the ``no noise'' framework, where the user assumes that the optimal hypothesis necessarily verifies $h^*(x) = y$ for any $(x, y) \in \mathcal{S}$. In this case, the set of consistent hypotheses is the set: $\mathcal{H}_\mathcal{S} = \{ h \in \mathcal{H}; \; h(x) = y \}$ \citep{mitchell1977versionSpace}. In general, the user assumes a moderated noise level. Then, the notion of consistency is related to the empirical risk $\mathcal{L}_{\mathcal{S}}(h)$, such that consistent hypotheses provide ``low'' empirical risk on $\mathcal{S}$. For instance, if the user is only interested in deploying models with greater accuracy than $\tau = 0.99$, then the set of consistent hypotheses is defined as $\mathcal{H}_\mathcal{S} = \{ h \in \mathcal{H}; \; \mathcal{L}_{\mathcal{S}}(h) \leq 1-\tau \}$ (assuming that $\ell$ is the 0-1 loss). In the Bayesian setting, a noise model, $p(y | x, h)$ is generally assumed (e.g., a Gaussian noise of unknown mean and variance), then a gradual notion of consistency is obtained through the likelihood of the hypothesis $h \in \mathcal{H}$ given the sample $\mathcal{S}$, i.e., $p(h | \mathcal{S})$ \citep{Angelo2021RepulsiveDeepEnsemble}.

\subsection{The main limitation of epistemic uncertainty estimation for deep learning}

\label{main-limit}

Based on the previous considerations, the epistemic uncertainty estimation is then considered accurate when the learner is able to determine the whole set of consistent hypothesis $\mathcal{H}_\mathcal{S}$ (or to determine the likelihood of any hypotheses in the Bayesian framework). However, as $\mathcal{H}$ is an infinite set, computing the empirical risk for any hypothesis from $\mathcal{H}$ to determine which hypothesis belong to $\mathcal{H}_\mathcal{S}$ is impossible. Moreover, with deep neural network hypotheses, determining the subspace $\mathcal{H}_\mathcal{S}$ is generally intractable, because of the non-linear relationship between the neural network parameters and the empirical error.

To overcome this issue, common practice consists in using empirical risk minimization algorithms to produce a sample or a distribution of consistent hypotheses. To avoid sampling always the same empirical risk minimizer, deep ensemble methods use random initialization and random batch order with early stopping \citep{Lakshminarayanan2017DeepEnsemble}, while Bayesian neural networks algorithms learn a weight distribution \citep{kendall2017EpistemicUncertainties}. Although such approaches foster hypothesis diversity, they cannot guarantee to produce a representative sample of the \textit{whole} set of consistent hypotheses. Moreover, common practices in deep learning training induce important biases which narrow the sampling in a restricted region of the consistent hypotheses' subspace. For instance, the use of weight decay ($\ell_2$ penalization) and random weights initialization of relatively small variance (e.g., equal to the inverse of the number of neurons in the layer, \citealt{Glorot10GlorotUniform}) drive the sample in low weight regions. Consistent hypotheses with high weights are then excluded, even though they can explain the observations as well, but in a different way, which would contribute to increase the potential prediction diversity. Similarly, in the Bayesian framework, it has been recently observed that the most commonly used prior, i.e., the Gaussian centered prior, is ``unintentionally informative'' \citep{wenzel2020howgoodBayesPosterior}. Finally, the evaluation of uncertainty quantification methods and their hyper-parameters selection is traditionally driven by the negative-log-likelihood metric (NLL) computed over a validation dataset belonging to the training domain \citep{liu1999negativecorrelation, pearce2018AnchorNetwork, Jain2020MOD}. However, such practice does not account for the epistemic uncertainty out-of-distribution and then does not foster methods which accurately estimate it. This issue is illustrated by the four bottom graphics of Figure \ref{toy-intro}, the four methods provide almost the same prediction uncertainty on the training domain, their validation NLL is then similar, but their OOD epistemic uncertainty estimation is very different.

Therefore, we identify the inability of standard approaches to produce a representative sample of consistent hypotheses as their main limitation. We argue that this limitation is the principal cause of their lack of prediction diversity for OOD data, observed recently \citep{ovadia2019CanYouTrustYourModel, Liu2021PerilDeepOOD, Angelo2021BayesianNotSuited4OOD} (cf. Section \ref{discuss-overfitting}).

\section{Weight Entropy Maximization}

The main contribution of this work is the development of a practical algorithm to produce a sample of hypotheses that tends to be representative of the whole space of consistent hypotheses. Considering stochastic neural networks, we propose to learn the scale parameters of a distribution over the network weights, centered on a hypothesis of low empirical risk, with the double objective of minimizing the average empirical risk and maximizing the distribution diversity, measured through the weight entropy.

\subsection{A principle to sample the whole space of consistent hypotheses}
\label{optim-form}

We consider the stochastic neural network approach, where samples of hypotheses are produced through a parameterized weight distribution $q_{\phi}$ in the set $\Phi = \{ q_{\phi} \}_{\phi \in \mathbb{R}^D}$ composed of several distributions over $\mathcal{W}$ parameterized by $\phi \in \mathbb{R}^D$, with $D \in \mathbb{N}$ the parameter dimension. We propose to penalize the average training risk over $q_{\phi}$ with the entropy of the weight distribution, leading to the following optimization formulation:

\begin{equation}
\label{tradeoff-optim}
    \min_{\phi \in \mathbb{R}^D} \; {\mathbb{E}}_{q_{\phi}}\left[ \mathcal{L}_{\mathcal{S}}(w) \right] - \lambda {\mathbb{E}}_{q_{\phi}}\left[ -\log(q_{\phi}(w)) \right],
\end{equation}
with  $\lambda \in \mathbb{R}_+$ the trade-off parameter.

\begin{itemize}
    \item The first term: ${\mathbb{E}}_{q_{\phi}}[\mathcal{L}_{\mathcal{S}}(w)]$ of the optimization objective in Equation (\ref{tradeoff-optim}) is the average empirical risk over the weight distribution. This term induces the increase of the probability mass $q_{\phi}(w)$ in regions where the weights $w \in \mathcal{W}$ produce accurate hypotheses on the training dataset, i.e., where $\mathcal{L}_{\mathcal{S}}(w)$ is small.
    \item The second term: $-\lambda {\mathbb{E}}_{q_{\phi}}\left[ -\log(q_{\phi}(w)) \right]$ in Equation (\ref{tradeoff-optim}) is a penalty that induces the increase of the weight entropy, which is generally related to expand the support of the weight distribution $q_{\phi}$ as broad as possible.
\end{itemize}


It should be underlined that both terms in Equation (\ref{tradeoff-optim}) evolve in opposite direction with respect to the weight distribution: the first term induces a weight distribution concentrated around any set of weights with minimal error on $\mathcal{S}$, while the second term induces a uniform distribution over the whole weight space. To solve this trade-off, the weight distribution tends to flatten in regions of little impact on the empirical risk, while remaining concentrated in directions where a small weight perturbation causes an important risk increase. The theoretical analysis in Section \ref{analysis} shows, indeed, that the distribution spread of the weights is inversely proportional to the neuron activation amplitude. The weight variance is then larger for weights in front of neurons weakly activated by the training data. This theoretical result is supported by numerical results observed on synthetic datasets in Section \ref{synth-expe} which provide a direct illustration of this link between the neuron activation and the weight variance (cf. Figure \ref{layer-amplitude}). 

Objective (\ref{tradeoff-optim}) can be understood as a maximum entropy problem \citep{Jaynes1957InfoTheory}, where, in the presence of partial information about the optimal weight, the uncertainty is best described by the distribution of low-risk hypotheses with maximal entropy. In the Bayesian neural network setting, a similar objective can be derived through the ELBO formulation by using the prior of maximum entropy \citep{jaynes1968priorprobability}, which, in this case, is the uniform distribution over $\mathcal{W}$ (see Section \ref{bayesian-discussion}). To highlight the link between our proposed approach and the maximum entropy principle, we call the method: Maximum Weight Entropy (\textbf{MaxWEnt}) in reference to the general maximum entropy modeling framework, commonly named MaxEnt \citep{berger1996maximumEntropyNLP}. Formulating the epistemic uncertainty quantification as a maximum entropy problem offers a natural classification among the weight distributions $q_{\phi} \in \Phi$. Between two weight distributions that provide the same level of empirical risk on the training data, the user should select the one with the largest entropy. The maximum entropy paradigm also offers an interesting guideline to drive the selection of the weight distribution family $\Phi$: the user should foster weight parameterization that enables larger increases of the entropy, as the SVD-parameterization described in Section \ref{svd-param-section}.

\subsection{The maximum weight entropy algorithm under general weight parameterization}
\label{impl}

Equation (\ref{tradeoff-optim}) is solved through stochastic gradient descent with mini-batches. To compute the expectation over $q_{\phi}$, we use the reparameterization trick \citep{Kingma2013VAE, rezende2014stochasticbackpropBayes}. We introduce a sampling variable $z \sim \mathcal{Z}$ with $\mathcal{Z}$ a distribution over $\mathbb{R}^d$ and a parameterization function $\omega: \mathbb{R}^d \times \mathbb{R}^d \to \mathbb{R}^d$ such that: $w = \omega(z, \phi)$. Typically, $z$ follows a distribution that can be numerically sampled as the normal or uniform distribution. In case of simple parameterization, the weight entropy can be directly derived from the weight parameters $\phi$, such that there exists a function $H: \mathbb{R}^d \to \mathbb{R}$ verifying $H(\phi) = {\mathbb{E}}_{\mathcal{Z}}\left[ -\log(q_{\phi}(\omega(z, \phi))) \right]$. This leads to the following objective function, computed on a mini-batch of data $\mathcal{S}_b \subset \mathcal{S}$ of size $B>0$:

\begin{equation}
\label{objective}
    G(\phi, \mathcal{S}_{b}) = {\mathbb{E}}_{\mathcal{Z}}\left[\mathcal{L}_{\mathcal{S}_{b}}(\omega(z, \phi)) \right] - \lambda \, H(\phi).
\end{equation}

By sampling $z^{(1)}, ..., z^{(N)}$ iid according to $\mathcal{Z}$, we can compute an estimation of the objective function gradient for each mini-batch as follows:

\begin{equation}
\label{gradient}
    \nabla_{\phi} G(\phi, \mathcal{S}_{b}) \simeq \nabla_{\phi} \left[ \frac{1}{N} \sum_{j=1}^N \mathcal{L}_{\mathcal{S}_{b}}(\omega(z^{(j)}, \phi))  - \lambda \, H(\phi) \right].
\end{equation}

Note that choosing $N=1$ appears to be sufficient, in practice, to obtain efficient results \citep{Kingma2013VAE}. Several gradient updates are performed until convergence to obtain the estimated parameters $\hat{\phi}$. The training part of the algorithm is summarized in Algorithm \ref{alg-training}. For inference on $x \in \mathcal{X}$, a set of $P$ predictions ($P \in \mathbb{N}^*$) is obtained by sampling multiple $z^{(j)} \sim \mathcal{Z}$ with $j \in [|1, P|]$, and computing the corresponding outputs $\{ h_{w_j}(x); \; w_j = \omega(z^{(j)}, \hat{\phi}) \}_{j \in [|1, P|]}$ (cf. Algorithm \ref{alg-infer})

\noindent
\begin{minipage}[t]{0.51\textwidth}
\begin{algorithm}[H]
	\caption{MaxWEnt Training}
	\label{alg-training}
    	\begin{algorithmic}[1]
    	    \STATE \textbf{Inputs}: Training set $\mathcal{S}$, learning rate $\nu$, trade-off $\lambda$, batch size $B$, parameterization $\omega$ 
            \STATE \textbf{Outputs}: Scaling vector $\phi$
            \STATE \textbf{Init}: $\phi \in \mathbb{R}^d$
            \smallskip
    	\WHILE {stopping criterion is not reached}
    	\STATE $z \sim \mathcal{Z}$,  $\mathcal{S}_b \sim \mathcal{U}(\mathcal{S}^B)$
    	\STATE $\phi \! \leftarrow \! \phi - \nu \nabla_{\phi} \left[ \mathcal{L}_{\mathcal{S}_b}(\omega(z, \phi)) - \lambda H(\phi) \right]$
    		\smallskip
    		\ENDWHILE
	    \end{algorithmic}
\end{algorithm}
\end{minipage}
\hfill
\begin{minipage}[t]{0.46\textwidth}
\begin{algorithm}[H]
	\caption{MaxWEnt Inference}
	\label{alg-infer}
    	\begin{algorithmic}[1]
    	    \STATE \textbf{Inputs}: Input data $x$, parameterization $\omega$, scaling vector $\phi$, sample size $P$ 
            \STATE \textbf{Outputs}: Prediction sample $(\hat{y}_1, ..., \hat{y}_P)$
    	\FOR {$1 \leq i \leq P$}
    	\STATE $z \sim \mathcal{Z}$;
            \STATE $w \leftarrow \omega(z, \phi)$
    	\STATE $\hat{y}_{i}  \leftarrow  h_{w}(x)$
    		\smallskip
    	\ENDFOR
	    \end{algorithmic}
\end{algorithm}
\end{minipage}

\subsection{Examples of weight parameterization}
\label{weight-param}

Obviously, the choice of the weight parameterization $\omega$ has an important impact on the resulting weight distribution. In line with the purpose of the MaxWEnt approach, the guidelines for choosing $\omega$ should follow these three principles: enable the sampling in regions of accurate hypotheses, foster the increase of the weight entropy and be practical to use. Moreover, one should consider weight parameterizations that provide a tractable formulation of the weight entropy $H(\phi)$.

\subsubsection{Scaling Parameterization}

Following the aformentionned guidelines, we consider the sampling variable $z \sim \mathcal{Z}$ such that ${\mathbb{E}}[z] = 0, {\mathbb{V}}[z] = \text{Id}_d$ and propose the ``scaling'' parameterization defined as follows:
\begin{equation}
\label{weight-reparam}
    \omega(z, \phi) = \overline{w} + \phi \odot z.
\end{equation}
Where $\odot$ is the element-wise product between two vectors, such that $\phi \odot z = (\phi_1 z_1, ..., \phi_d z_d)$ with $\phi = (\phi_1, ..., \phi_d) \in \mathbb{R}^d$ and $z = (z_1, ..., z_d) \in \mathbb{R}^d$. The weight vector $\overline{w} \in \mathbb{R}^d$ is the weight mean $\mathbb{E}_{q_{\phi}}[w] = \overline{w}$. It is typically defined as the weights of a pretrained network $h_{\overline{w}}$ fitted on the training data. For $\mathcal{Z}$ defined as a normal $\mathcal{N}(0, \text{Id}_d)$ or uniform distribution $\mathcal{U}([-\sqrt{3}, \sqrt{3}]^d)$, the parameters $\phi = (\phi_1, ..., \phi_d)$ act as scaling factors: the higher $\phi_k$, the wider the distribution $w_k \sim \overline{w}_k + \phi_k z_k$.

The scaling parameterization (\ref{weight-reparam}) meets the three previous requirements for a relevant choice of stochastic model. The mean of the weight distribution verifies $\mathbb{E}_{q_{\phi}}[w] = \overline{w}$ with $\overline{w}$ the weights of a pretrained network fitted on $\mathcal{S}$, the weight distribution is then centered in a region of the weight space of low empirical risk. If $\phi \simeq 0$, the resulting weight distribution is equivalent to a peaked distribution around $\overline{w}$, which meets the first objective to provide samples of accurate hypotheses. Moreover, the weight entropy is directly controlled by the parameters $\phi$ : when $\phi$ increases, the weight distribution becomes wider and the entropy increases. We show, indeed, in the next section, that the weight entropy $H(\phi)$ can be expressed directly as a function of $\phi$. Finally, it can be noticed that the scaling parameterization only involves element-wise multiplications, which makes it practical to compute.

\subsubsection{SVD Parameterization}
\label{svd-param-section}

We show, through the theoretical analysis developed in Section \ref{analysis}, that the increase of the $\phi$ parameters is inversely proportional to the neuron activation amplitude. Indeed, if a neuron is weakly activated by the training data, all the weights $w_k$ in front of this neuron have little impact on the network predictions in the training domain. Therefore, the parameters $\phi_k$ can be enlarged without degrading the average empirical risk ${\mathbb{E}}_{q_{\phi}}\left[ \mathcal{L}_{\mathcal{S}}(w) \right]$. In the extreme case, if the neuron is never activated by the training data (it always returns $0$), then the parameters $\phi_k$ can go to infinity without impacting the network outputs on the training domain. Based on this theoretical observation, we argue that the weight entropy can be further increased without impacting the training risk by taking into account the correlation between neurons. Indeed, let's consider, for instance, two neurons of the same hidden layer, totally correlated, both with activation amplitude $a > 0$ on average on the training data. The scales of the weights $w_k$ in front of these neurons will verify $\phi_k \propto 1/a$. However, by expressing the outputs of these neurons in their singular value decomposition basis, the novel representation is now composed of one component of average amplitude $a$ and the other of null amplitude. In that case, some parameters $\phi_k$ can be further increased without impacting the training risk. Motivated by these arguments, we propose the ``SVD'' parameterization described in the following subsection.

Let's consider a pretrained neural network $h_{\overline{w}}$ of $L$ hidden layers. We denote $\psi_{(l)}(X) \in \mathbb{R}^{n \times b_l}$ the hidden representation of the input data $X \in \mathbb{R}^{n \times b}$ in the $l^{th}$ layer of $h_{\overline{w}}$, with $b_l$ the hidden layer dimension (i.e., the number of neurons). The singular values decomposition of $\psi_{(l)}(X)$ is written: $\psi_{(l)}(X) = U_{(l)} S_{(l)} V_{(l)}$ with $U_{(l)} \in \mathbb{R}^{n \times n}, S_{(l)} \in \mathbb{R}^{n \times b_l}$ and $V_{(l)} \in \mathbb{R}^{b_l \times b_l}$. We propose the SVD parameterization, which consists in ``aligning'' the weight distribution with the principal components of $\psi_{(l)}(X)$ such that:
\begin{equation}
\label{SVD-param-nocompact}
 w_{(l)} = \overline{w}_{(l)} + V_{(l)}^T (\phi_{(l)} \odot z_{(l)}),
\end{equation}
for any $l \in [|0, L|]]$, where $w_{(l)}, \overline{w}_{(l)}, \phi_{(l)}, z_{(l)} \in \mathbb{R}^{b_l \times b_{l+1}}$ are respectively the matrix of weights, average weights, scaling parameters and sampling variables between the $l^{th}$ layer and the next layer. A compact formulation of the parameterization can be written as follows:

\begin{equation}
\label{SVD-param}
    \omega(z, \phi) = \overline{w} + V(\phi \odot z).
\end{equation}
Where $V$ denotes the block matrix: $V = \big[V_{(1)}^T, ..., V_{(1)}^T, V_{(2)}^T, ..., V_{(L)}^T \big]$ of dimension $\sum b_l \times b_{l+1}$.

Similar to the previous one, the SVD parameterization fulfills the guidelines. Indeed, the weight distribution is still centered on $\overline{w}$, which ensures to sample in a weight space region of low empirical risk. Moreover, the weight entropy can be increased by enlarging the $\phi$ parameters. This can be done more efficiently compared to the previous approach due to the integration of the neurons' correlations (cf. Section \ref{analysis-weight-param}). The SVD parameterization requires additional computational time compared to the scaling one, due to the SVD decomposition and the matrix multiplication. It should be noticed that the SVD decomposition for each layer is computed only once. Before the stochastic gradient descent, a forward pass of the training data in $h_{\overline{w}}$ is required to compute each hidden representation $\psi_{(l)}(X)$, then the SVD decomposition of $\psi_{(l)}(X)$ is performed to compute the matrix $V_{(l)}$. However, the matrix multiplications between $V_{(l)}$ and $\phi_{(l)} \odot z_{(l)}$ are performed at each gradient update, which requires an additional computational burden during the gradient descent compared to the scaling parameterization (cf. Section \ref{discuss-svd} for the complexity calculation). Finally, we show in the next section, that a similar expression of the weight entropy $H(\phi)$ can be written in function of $\phi$ for both parameterizations.

\subsubsection{Weight Entropy Formulation Under Scaling and SVD Parameterizations}

The following proposition states that the previous weight parameterizations provide a closed-form expression of the weight entropy $H(\phi)$:

\begin{proposition}[Closed-form expression of the weight entropy]
\label{entropy-closeform}
Let $q_{\phi}$ be a weight distribution described by Equation (\ref{weight-reparam}) or (\ref{SVD-param}) with $z \sim \mathcal{Z}$ and $\phi > 0$. If $\mathcal{Z}$ is defined as the normal $\mathcal{N}(0, \textnormal{Id}_d)$ or the uniform distribution $\mathcal{U}([-\sqrt{3}, \sqrt{3}]^d)$, there exists two constants $C_1, C_2$ such that the weight entropy $H(\phi)$ is expressed as follows:
    \begin{equation}
    \label{entropy}
    H(\phi) = C_1 \sum_{k=1}^d \log(\phi^2_k) + C_2,
    \end{equation}
with $\phi = (\phi_1, ..., \phi_d) \in \mathbb{R}^d$, the scaling parameters of the weight distribution $q_{\phi}$.
\end{proposition}
\begin{proof}
    The full proof is reported in Appendix \ref{proof-entropy-closeform}. The proof consists in considering that, for a normal distribution $\mathcal{N}(0, \Sigma)$ or for a uniform distribution defined over a parallelotope described by $\Sigma$, the entropy verifies $H(\phi) \propto \log(|\textnormal{det}(\Sigma)|)$. Then, by showing that for both parameterizations $\textnormal{det}(\Sigma) \propto \textnormal{det}(\textnormal{diag}(\phi))$, the above result can be derived.
\end{proof}

Note that, the $C_2$ constant can be removed in the objective function of Equation (\ref{objective}) as it does not impact the optimization and the $C_1$ constant can be integrated in the trade-off parameter $\lambda$. This expression of the entropy function is easy to implement. It highlights the direct link between the scale parameter $\phi_k$ and the weight entropy. When $\phi_k$ grows, the weight distribution becomes wider and the entropy increases.

\section{Theoretical Analysis}
\label{analysis}

In this section, we develop a theoretical framework to understand the MaxWEnt approach in the specific case where the loss function is defined by the mean squared error. We first develop theoretical results in the linear regression case, and further extend these results to deep fully-connected neural networks.

\subsection{Linear regression}
\label{linear-theory}

Linear regression can be seen as a particular case of deep fully-connected neural networks where the networks are composed of exactly two layers: the input layer of $b$ neurons and the output layer of $1$ neuron with linear activation function. The linear regression case is not representative of the framework considered in this work, as the hypotheses $h \in \mathcal{H}$ can no longer be considered as universal approximators. However, the following study provides valuable insights on what happen between the neurons of one hidden layer and one neuron of the next layer. In particular, we highlight the link between the scale parameters $\phi$ and the amplitude of the input features.

\subsubsection{Notations}

We consider the linear regression framework, where the learner has access to an input dataset $X \in \mathbb{R}^{n \times b}$ composed of $n$ row data $x_i \in \mathbb{R}^b$ drawn iid according to the distribution $p(x)$ and an output vector $y \in \mathbb{R}^n$ such that $y = (y_1, ..., y_n)$. Each input $x_i$ is associated to the scalar output $y_i \in \mathbb{R}$ drawn according to $p(y|x_i)$. We denote $\mathcal{S} = \{(x_1, y_1), ..., (x_n, y_n)\}$ the set of training observations. We consider the set $\mathcal{H} = \{x \to \sum_{k=1}^b x_k w_k; \; w \in \mathbb{R}^b\}$ of linear hypotheses. The loss function is the mean squared error, and we define the empirical risk for any weight $w \in \mathbb{R}^b$ as $\mathcal{L}_{\mathcal{S}}(w) = \frac{1}{n} ||X w - y||^2_2$. We denote by $a = (a_1, ..., a_b) \in \mathbb{R}_+^b$ the \textit{amplitude} of the input features of the training set, such that $a_j^2 = \frac{1}{n} ||X_j||^2_2$ for any $j \in [|1, b|]$, with $X_j$ the $j^{\text{th}}$ column of $X$. We assume that $a_j > 0$ for any $j \in [|1, b|]$.

\subsubsection{Scaling Weight Parameterization}

We first consider the weight parameterization defined in Equation (\ref{weight-reparam}) such that $q_{\phi} \sim \overline{w} + \phi \odot z$ with $z \sim \mathcal{Z}$ such that $\mathcal{Z} \sim \mathcal{N}(0, \textnormal{Id}_b)$ or $\mathcal{Z} \sim \mathcal{U}([-\sqrt{3}, \sqrt{3}]^b)$. The weight vector $\overline{w} \in \mathbb{R}^b$ is the weight mean: $\mathbb{E}_{q_{\phi}}[w] = \overline{w}$. Finally, we consider the entropy penalty $H(\phi)$ defined by $H(\phi) = \sum_{k=1}^b \log(\phi_k^2)$. The optimization problem (\ref{tradeoff-optim}) can then be written:

\begin{equation}
\label{linear-optim}
    \min_{\phi \in \mathbb{R}^b} {\mathbb{E}}_{\mathcal{Z}}\left[\frac{1}{n} \left|\left| X (\overline{w} + \phi \odot z) - y \right|\right|_2^2\right] - \lambda \sum_{k=1}^b \log(\phi_k^2).
\end{equation}

We show that the MaxWEnt optimization problem of Equation (\ref{linear-optim}) has a unique solution, which can be expressed with the following closed-form expression:

\begin{proposition}[Closed-form solution for the scaling parameterization]
\label{thm-lin-indep}
Equation (\ref{linear-optim}) has a \linebreak unique solution $\phi^* \in \mathbb{R}^b$ verifying for any $k \in [|1, b|]$:
\begin{equation*}
\label{thm-lin-indep-eq}
    {\phi_k^*}^2 = \frac{\lambda}{a_k^2}.
\end{equation*}
\end{proposition}
\begin{proof}
The proof consists in first developing the average risk as follows:
\begin{equation*}
    {\mathbb{E}}_{\mathcal{Z}}\left[ \frac{1}{n} || X (\overline{w} + \phi \odot z) - y ||_2^2 \right] = \sum_{k=1}^b \, a_k^2 \phi_k^2 + \frac{1}{n} || X \overline{w} - y ||_2^2.
\end{equation*}
Optimization (\ref{linear-optim}) can then be written:
\begin{equation}
\label{log-determinant}
    \min_{\phi \in \mathbb{R}^b} \sum_{k=1}^b \, a_k^2 \phi_k^2 - \lambda \sum_{k=1}^b \, \log(\phi_k^2).
\end{equation}
This is a convex problem, for which the derivative of the objective function with respect to $\phi^2$ is null for:
\begin{equation*}
    a_k^2 - \lambda / \phi_k^2 = 0.
\end{equation*}
A detailed proof is provided in Appendix \ref{appendix-thm-lin-indep}.
\end{proof}

This closed-form solution of $\phi^*$ is particularly insightful: ${\phi_k^*}$ is inversely proportional to $a_k^2$, which means that the optimal scale parameters ${\phi_k^*}$ are larger for weights in front of low amplitude features $a_k^2$. Applied to the hidden layers of a neural network, Proposition (\ref{thm-lin-indep}) states that the weight distribution is wider in front of neurons weakly activated by the training data. As a consequence, if an OOD data activates these neurons, large values are propagated through the network, which produces an important output variance. These statements are formalized in Section \ref{theory-multi} when considering deep fully connected neural networks. 


It can be further noticed that Equation (\ref{log-determinant}) is equivalent to a log determinant optimization problem \citep{boyd2006convex}. The maximum entropy optimization can then be interpreted as a maximum ellipsoid volume problem, where the volume $\prod \phi_k^2$ is maximized under the linear constraint $\sum_k a_k^2 \phi_k^2 \leq \lambda b$. If $\mathcal{Z}$ is a uniform distribution, this boils down to maximizing the support of the weight distribution while maintaining the average empirical risk on the training data under an acceptable threshold. This is in line with the purpose of the approach to find the weight distribution that covers as many consistent weights as possible.

\subsubsection{SVD Weight Parameterization}
\label{analysis-weight-param}

According to Proposition (\ref{thm-lin-indep}), the optimal scale parameters verify ${\phi^*}^2 = \lambda / a^2$. When injecting this solution in the entropy formulation, we obtain: $H(\phi) = - \sum \log(a_k^2) + \text{cste}$. Considering this formula, it appears clearly that the weight entropy is particularly important if some $a_k^2$ are small, i.e., if some input features have a low amplitude. However, in the presence of correlated features, all amplitudes $a_k^2$ may be high while the input training data may present small variation in some directions of the input space. The SVD parameterization (\ref{SVD-param}) proposes to exploit these directions of small variation by aligning the weight distribution with the singular value components of the input data. For this purpose, we now consider $V \in \mathbb{R}^{b \times b}$, the matrix of eigenvectors of $\frac{1}{n} X^T X$ and $s^2 = (s_1^2, ..., s_b^2) \in \mathbb{R}_+^b$ the vector of eigenvalues, and assume that $s_j>0$ for any $j \in [|1, b|]$. The SVD weight parameterization is written $w = \overline{w} + V (\phi \odot z)$ with $z \sim \mathcal{Z}$ and the MaxWEnt optimization problem (\ref{tradeoff-optim}) becomes:
\begin{equation}
\label{svd-optim}
\min_{\phi \in \mathbb{R}^b} {\mathbb{E}}_{\mathcal{Z}}\left[\frac{1}{n} \left|\left| X (\overline{w} + V (\phi \odot z) - y \right|\right|_2^2\right] - \lambda \sum_{k=1}^b \log(\phi_k^2).
\end{equation}

In comparison to the previous optimization problem in Equation (\ref{linear-optim}), there is now the presence of the matrix $V$ between $X$ and $\phi \odot z$. By definition of $V$, the matrix $X V$ is the expression of $X$ in its singular values basis. Thus, the vector $\phi \odot z$ is now aligned with the singular value components. As for the previous parameterization, the optimal parameter vector $\phi^*$ admits a closed-form expression as follows:
\begin{proposition}[Closed-form solution for the SVD parameterization]
\label{thm-lin-notindep}
Equation (\ref{svd-optim}) has a \linebreak unique solution $\phi^* \in \mathbb{R}^b$ verifying for any $k \in [|1, b|]$:
\begin{equation*}
\label{thm-lin-notindep-eq}
    {\phi_k^*}^2 = \frac{\lambda}{s_k^2}.
\end{equation*}
\end{proposition}
\begin{proof}
The proof consists in developing the average risk, such that:
\begin{equation*}
    {\mathbb{E}}_{\mathcal{Z}}\left[ \frac{1}{n} || X (\overline{w} + V (\phi \odot z)) - y ||_2^2 \right] = \sum_{k=1}^b \, s_k^2 \phi_k^2 + \frac{1}{n} || X \overline{w} - y ||_2^2.
\end{equation*}
Optimization (\ref{svd-optim}) is then written:
\begin{equation*}
    \min_{\phi \in \mathbb{R}^b} \sum_{k=1}^b \, s_k^2 \phi_k^2 - \lambda \sum_{k=1}^b \, \log(\phi_k^2),
\end{equation*}
which is similar to Equation (\ref{log-determinant}) with $s_k^2$ instead of $a_k^2$ (see Appendix \ref{appendix-thm-lin-notindep} for a detailed proof).
\end{proof}

Proposition (\ref{thm-lin-notindep}) states that the optimal parameters $\phi^*$ are now inversely proportional to the singular values of the training data instead of the feature amplitudes. We show, with the next Proposition, that this difference implies a larger weight entropy for the same level of average empirical risk.

\begin{proposition}[Comparison between scaling and SVD parameterization]
\label{thm-wf}
Let $q^{(1)}_{\phi^*}$, $q^{(2)}_{\phi^*}$ be the respective optimal weight distributions for the scaling and the SVD parameterization. The following propositions hold:
\begin{gather*}
    {\mathbb{E}}_{q^{(1)}_{\phi^*}}\left[ \mathcal{L}_{\mathcal{S}}(w) \right] = {\mathbb{E}}_{q^{(2)}_{\phi^*}}\left[ \mathcal{L}_{\mathcal{S}}(w) \right] \\
    {\mathbb{E}}_{q^{(1)}_{\phi^*}}\left[ -\log(q^{(1)}_{\phi^*}(w)) \right] \leq {\mathbb{E}}_{q^{(2)}_{\phi^*}}\left[ -\log(q^{(2)}_{\phi^*}(w)) \right] .
\end{gather*}
\end{proposition}
\begin{proof}
The average empirical risk equality can be derived as follows:
\begin{equation*}
    {\mathbb{E}}_{q^{(1)}_{\phi^*}}\left[ \mathcal{L}_{\mathcal{S}}(w) \right] = \lambda \sum_{k=1}^b \frac{a_k^2}{a_k^2} + \epsilon = \lambda \, b + \epsilon = \lambda \sum_{k=1}^b \frac{s_k^2}{s_k^2} + \epsilon = {\mathbb{E}}_{q^{(2)}_{\phi^*}}\left[ \mathcal{L}_{\mathcal{S}}(w) \right],
\end{equation*}
with $\epsilon = \frac{1}{n} ||X \overline{w} - y ||_2^2$. The weight entropy inequality is derived from Hadamard's inequality. The detailed proof is reported in Appendix \ref{appendix-thm-compar}.
\end{proof}
In light of Proposition (\ref{thm-wf}), it appears that the SVD parameterization leads to a more efficient weight distribution according to the maximum entropy principle. Indeed, for the same level of explanation of the observations (same average empirical risk), the SVD parameterization provides more entropy. Experiments conducted on both synthetic and real datasets show that this last weight parameterization provides, indeed, a better evaluation of the epistemic uncertainty (cf. Section \ref{experiments}) which advocates in favor of the use of the entropy as a measure of weight distribution quality.

\subsection{Deep fully connected neural network}
\label{theory-multi}

In this subsection, we extend the previous result to deep fully connected networks under the mean squared error loss. In particular, we formally derive the connection between the neuron activation amplitude and the optimal scaling parameters suggested by Proposition (\ref{thm-lin-indep}).

\subsubsection{Notations}

We consider fully-connected neural networks $h_w \in \mathcal{H}$ of $L$ hidden layers with $w \in \mathcal{W}$. For the sake of simplicity, we assume that every hidden layer is composed of $b$ neurons with $b$ the dimension of the input data, the last layer  is composed of $1$ neuron such that the neural networks produce scalar outputs. For any $x \in \mathcal{X}$ and for any $l \in [|1, L|]$, $\psi_{(l)}(x) \in \mathbb{R}^b$ denotes the hidden representation of the input data $x$ in the $l^{th}$ layer; $\psi_{(0)}(x) \in \mathbb{R}^b$ and $\psi_{(L+1)} \in \mathbb{R}$ are respectively the input and output layer representation, such that $\psi_{(0)}(x) = x$ and $\psi_{(L+1)}(x) = h_w(x)$. Notice that the hidden representations depend on $w$; the notation $\psi_{(l)}(x)$ is a contraction of $\psi_{(l)}(x, w)$ or ${\psi_{(l)}}_w(x)$. The set of network weights verifies $\mathcal{W} \subset \mathbb{R}^d$, with $d = L b^2 + b$ the number of weights in the network (bias parameters are not considered here). For any weights $w \in \mathcal{W}$, $w_{(l, j)} \in \mathbb{R}^b$ denotes the weights between the layer $l$ and the $j^{th}$ components of the layer $l+1$ for $l \in [|0, L|]$ and $j \in [|1, b_l|]$, with $b_l = 1$ if $l = L$ and $b_l = b$ otherwise. We consider the activation function $\zeta: \mathbb{R} \to \mathbb{R}$ such that, for any $x \in \mathcal{X}$, any $l \in [|0, L-1|]$ and any $j \in [|1, b|]$, $\psi_{(l+1, j)}(x) = \zeta\left(\psi_{(l)}(x)^T w_{(l, j)} \right)$ with $\psi_{(l+1, j)}(x)$ the $j^{th}$ component of the hidden representation $\psi_{(l+1)}(x)$. The weight distributions are denoted $q_{\phi}$ with $\phi \in \mathbb{R}^d$. The loss function $\ell$ is the mean squared error and the problem to be solved is written:

\begin{equation}
\label{multi-layer}
    \min_{\phi \in \mathbb{R}^d} \; {\mathbb{E}}_{q_{\phi}}\left[\mathcal{L}_{\mathcal{S}}(w) \right] - \lambda \sum_{k=1}^d \log(\phi_k^2).
\end{equation}
We assume that Problem (\ref{multi-layer}) has a unique solution, denoted $\phi^* \in \mathbb{R}^d$.

\subsubsection{Scaling Weight Parameterization}
\label{multi-sec-scaling}

We focus our deep neural networks analysis on the scaling parameterization (\ref{weight-reparam}) such that $q_{\phi} \sim \overline{w} + \phi \odot z$ with $z \sim \mathcal{Z}$ where $\mathcal{Z} \sim \mathcal{N}(0, \textnormal{Id}_d)$ or $\mathcal{Z} \sim \mathcal{U}([-\sqrt{3}, \sqrt{3}]^d)$ and $\overline{w}$ the weight of a pretrained network $h_{\overline{w}}$. In the following, we aim at extending the results of Proposition (\ref{thm-lin-indep}) to the hidden layers of deep neural networks and show that the MaxWEnt optimization leads to scaling parameters inversely proportional to the neuron activation amplitude. For this purpose, we consider the following assumption on the activation function $\zeta$. Assumption (\ref{activ-func}) states that the order of the first and second moment of the neuron activation are preserved by $\zeta$. This assumption is verified, for instance, for most of the common activation functions, as ReLU or Leaky-ReLU, if the neuron activation follows a centered independent Gaussian distribution.

\begin{assumption}[Moments preserving property of the activation function]
\label{activ-func}
For any $\phi_1, \phi_2 \in \Phi$, $l \in [|0, L-1|]$ and any $j \in [|1, b|]$, the activation function $\zeta$ verifies:
\begin{gather*}
\sum_{i=1}^n {\mathbb{E}}_{q_{\phi_1}}\left[ U_{ij} \right] \leq \sum_{i=1}^n {\mathbb{E}}_{q_{\phi_2}}\left[ U_{ij} \right] \implies \sum_{i=1}^n {\mathbb{E}}_{q_{\phi_1}}\left[ \zeta \left( U_{ij} \right) \right] \leq \sum_{i=1}^n {\mathbb{E}}_{q_{\phi_2}}\left[ \zeta \left( U_{ij} \right)\right] \\
\sum_{i=1}^n {\mathbb{E}}_{q_{\phi_1}}\left[ U_i U_i^T \right] \preccurlyeq \sum_{i=1}^n  {\mathbb{E}}_{q_{\phi_2}}\left[ U_i U_i^T \right] \! \implies \! \sum_{i=1}^n  {\mathbb{E}}_{q_{\phi_1}}\left[ \zeta \left( U_i \right) \zeta \left( U_i \right)^T \right] \preccurlyeq \sum_{i=1}^n  {\mathbb{E}}_{q_{\phi_2}}\left[  \zeta \left( U_i \right) \zeta \left( U_i \right)^T \right]
\end{gather*}
Where $U_i = (U_{i1}, ..., U_{ib})$ and $U_{ij} = \psi_{(l)}(x_i)^T w_{(l, j)} \; \forall i \in [|1, n|], \; \forall j \in [|1, b|]$. For two matrices $A, B$, the notation $A \preccurlyeq B$ states that $B-A$ is a positive semi-definite matrix.
\end{assumption}

\begin{proposition}[Optimal scaling parameters]
\label{thm-multi-closeform}
Let $\phi^* \in \mathbb{R}^d$ be the unique solution of Problem (\ref{multi-layer}), then $\phi^*$ verifies:
    \begin{equation}
    \label{thm-multi-cond-phi-eq}
    \begin{gathered}
        \phi^* = \overset{L}{\underset{l=0}{\bigotimes}} \overset{b_l}{\underset{j=1}{\bigotimes}} \, \left( \phi^*_{(l, j, 1)}, ..., \phi^*_{(l, j, b)} \right) \\
        {\phi^*_{(l, j, k)}}^2 = \frac{\sigma_{(l, j)}^2}{b \, a_{(l, k)}^2} \quad \forall \, l \in [|1, L|]; \, j \in [|1, b_l|]; \, k \in [|1, b|].
    \end{gathered}
    \end{equation}
    Where $\bigotimes$ is the concatenation operator and for any $l \in [|0, L|]$, $j \in [|1, b_l|]$ and $k \in [|1, b|]$:
    \begin{gather*}
        a_{(l, k)}^2 = \frac{1}{n} \sum_{i=1}^n {\mathbb{E}}_{q_{\phi^*}} \left[ \psi_{(l, k)}(x_i)^2 \right] \\
        \sigma_{(l, j)}^2 = \frac{1}{n} \sum_{i=1}^n {\mathbb{V}}_{q_{\phi^*}} \left[ \psi_{(l)}(x_i)^T \left( w_{(l, j)} - \overline{w}_{(l, j)} \right) \right].
    \end{gather*}
\end{proposition}

\begin{proof}
     The full proof is reported in Appendix \ref{proof-thm-multi-closeform}. The main idea of the proof consists in first dividing Problem (\ref{multi-layer}) by layer and output neurons. The parameters $\phi_{(l, j, k)}$ defined in Equation (\ref{thm-multi-cond-phi-eq}) provide the solution for each sub-problem. Then, considering Assumption ($\ref{activ-func}$) on the activation function and the uniqueness of the solution, it can be shown that $\phi = \phi^*$.
\end{proof}

Proposition (\ref{thm-multi-closeform}) states that the solution $\phi^*$ of the MaxWEnt optimization (\ref{multi-layer}) is the inverse of the average neuron activation amplitude over the training data. We emphasize that the aim of Proposition (\ref{thm-multi-closeform}) is not to provide an exact solution (as the quantities $a_{(l, k)}^2$ and $\sigma_{(l, j)}^2$ are intractable) but to offer a theoretical understanding of MaxWEnt in the case of deep fully connected neural networks. Numerical observations described in Section \ref{theory-verif} confirm this ``inverse proportionality'' relationship between the scaling parameters and the neuron activation amplitude. This means that maximizing the weight entropy leads to put more emphasis on the activation of neurons that are weakly activated by the training data. Thus, it can be considered that these neurons act as ``detectors'' for the out-of-distribution data that activate them.

\section{Discussion}
\label{discussion}

\subsection{Overfitting, weight diversity and evaluation}
\label{discuss-overfitting}

In Section \ref{main-limit}, we identify the main limitation of standard ensemble and Bayesian approaches as their inability to produce a representative sample of the whole consistent hypothesis set. We argue that this limitation is related to the use of weight decay regularization and hyper-parameters selection driven by hold-out validation. Indeed, the use of weight decay for deep neural networks is first designed as a tool to avoid overfitting \citep{krogh1991simpleweightdecayL2reg}, with the underlying idea that large weights induce the over-specification of the network on the observations. This technique has proven to improve the model accuracy in most cases. However, when applied in ensemble and Bayesian learning, it induces the counter effect of penalizing the diversity of the resulting sample of neural networks. On the contrary,
our approach aims at jointly regularizing the sample of neural networks by increasing the entropy of their underlying weight distribution. Additionally, the use of a broad weight distribution avoids overfitting thanks to the marginalization process \citep{wilson2020caseOfBayesianDL}.


Regarding the use of hold-out validation for hyper-parameters selection, we claim that such a technique fosters narrowed weight distributions. As $q_{\phi}$ cannot model complex distribution, the covering of a large portion of consistent hypotheses generally comes with the inclusion of inconsistent weights in the support of the weight distribution. As a consequence, the in-distribution performance for distributions of high entropy is usually degraded (confirmed numerically in our experiments). Moreover, for a large number of training data, the in-distribution epistemic uncertainty becomes negligible in front of the aleatoric uncertainty. Its accurate estimation is then not required to obtain good validation NLL. However, for out-of-distribution data, the main source of uncertainty is epistemic, and its estimation is critical. Then, narrowed weights distributions, although improving the validation NLL, fail to produce relevant uncertainty quantification out-of-distribution \citep{ovadia2019CanYouTrustYourModel, Liu2021PerilDeepOOD, Angelo2021BayesianNotSuited4OOD}.

It should be underlined that, although MaxWEnt tends to enlarge the weight distribution, it cannot fully guarantee to capture the whole set of consistent hypotheses due to the technical limitation of the stochastic model $q_{\phi}$. However, the MaxWEnt approach is an important step in this direction. It already provides significant improvements compared to the baselines, as demonstrated by our numerical experiments.

\subsection{Bayesian neural networks}
\label{bayesian-discussion}

In the Bayesian variational inference framework, the learner aims at approximating the posterior distribution $p(w | \mathcal{S})$ with a parameterized distribution $q_{\phi}$ defined over $\mathcal{W}$. The minimization of the Kullback-Leibler (KL) divergence between $p(w | \mathcal{S})$ and $q_{\phi}$ leads to the maximization of the \textit{evidence lower bound} (ELBO) expressed as follows \citep{wenzel2020howgoodBayesPosterior}:
\begin{equation}
\label{elbo}
    \max_{\phi \in \mathbb{R}^D} \; {\mathbb{E}}_{q_{\phi}}\left[ \sum_{(x, y) \in \mathcal{S}} \log(p(y | h_w(x)) \right] - D_{\textnormal{KL}} \left( q_{\phi}(w), p(w) \right).
\end{equation}
Where $p(y | h_w(x))$ is the log likelihood of $y$ with respect to $h_w(x)$, $D_{\textnormal{KL}}$ is the Kullback-Leibler divergence and $p(w)$ is the prior distribution defined over $\mathcal{W}$.

If we consider a uniform prior over the whole weight space: $p(w) \sim \mathcal{U}(\mathcal{W})$ (assuming $\mathcal{W}$ bounded), the second term of the ELBO maximization: $D_{\textnormal{KL}} \left( q_{\phi}(w), p(w) \right)$, is equal to the negative entropy of $q_{\phi}$ (up to a constant). Therefore, if the empirical risk $\mathcal{L}_{\mathcal{S}}(w)$ can be written as a quantity proportional to the negative log-likelihood, the ELBO maximization (\ref{elbo}) is equivalent to the MaxWEnt optimization problem (\ref{tradeoff-optim}). This is in line with the application of the maximum entropy principle to the Bayesian framework \citep{jaynes1968priorprobability}, which states that the prior should be selected as the distribution of maximal entropy that integrates prior information. In our case, without any regularity assumption about the optimal hypothesis, the maximum entropy principle then leads to consider a uniform prior over the whole weight space $\mathcal{W}$ (bounded), i.e., $p(w) \sim \mathcal{U}(\mathcal{W})$. 

The use of ``uninformative'' parameter priors is considered as the guideline to model epistemic uncertainty in the Bayesian framework \citep{wilson2020caseOfBayesianDL}. In practice, however, the most commonly used priors for Bayesian neural networks are Dropout \citep{gal2016MCdropout, kendall2017EpistemicUncertainties, boluki2020learnableDropout} which has been shown to produce over-confident predictions for out-of-distribution data \citep{Liu2021PerilDeepOOD} and the isotropic Gaussian prior $p(w) \sim \mathcal{N}(0, \sigma_0^2 \, \text{Id}_d)$ \citep{zhang2018NoisyNaturalGradientVI, osawa2019practicalDLwithBayesian, jospin2022BNNTutorial}, which is recently considered to be often ``non-optimal'' or ``unintentionally informative'' \citep{wenzel2020howgoodBayesPosterior, fortuin2021bayesianPriorRevisited}.

When considering a Gaussian isotropic prior $p(w) \sim \mathcal{N}(0, \sigma_0^2 \, \text{Id}_d)$ with $\sigma_0 \in \mathbb{R}$ and an independent multivariate Gaussian stochastic model $q_{\phi} \sim \mathcal{N}(\mu, \text{diag}(\sigma^2))$ with $\mu, \sigma \in \mathbb{R}^d$ the mean and scale parameters such that $\phi = (\mu, \sigma)$, the following expression can be derived for the KL divergence between the approximate posterior and the prior \citep{duchi2007KLtwoGaussian}:
\begin{equation*}
\label{gaussian-kl}
    D_{\textnormal{KL}} \left( q_{\phi}(w), p(w) \right) =\frac{||\mu||_2^2}{2 \sigma_0^2} + \frac{1}{2} \sum_{k=1}^d  \left( \frac{\sigma_k^2}{\sigma_0^2}  - \log \left( \frac{\sigma_k^2}{\sigma_0^2} \right)  \right) - \frac{d}{2}.
\end{equation*}
From this expression, it appears that the KL divergence operates a ``double'' regularization regime on the scale parameters $\sigma$. When $\sigma_k^2$ is below $\sigma_0^2$, the term $-\log(\sigma_k^2 / \sigma_0^2)$ dominates $\sigma_k^2 / \sigma_0^2$, which induces the increase of the $\sigma_k^2$ parameter similar to the MaxWEnt penalization. Whereas, for $\sigma_k^2$ above $\sigma_0^2$, the dominant term becomes $\sigma_k^2  / \sigma_0^2$ which stops the increase of the scaling parameter. Then, for $\sigma_0 \to + \infty$, the regularization over $\sigma^2$ induced by the KL divergence converges to the maximum entropy penalization. However, as a side effect, the term $||\mu||_2^2 / 2 \sigma_0^2$ is reduced to zero and no regularization on the mean is operated, which is generally avoided. In many previous works which consider isotropic Gaussian priors, the commonly considered prior bandwidth $\sigma_0^2$ are relatively small \citep{zhang2018NoisyNaturalGradientVI, osawa2019practicalDLwithBayesian, ashukha2019pitfalls}, or at least, not designed in a maximum entropy perspective. Moreover, a trade-off parameter $\lambda < 1$ is often added between the log likelihood and the KL divergence in optimization (\ref{elbo}) \citep{wenzel2020howgoodBayesPosterior} which further tempers the KL divergence regularization. Our interpretation is that the hyper-parameter selection is often driven by the in-distribution performances (computed on a validation set for instance) which fosters narrowed posterior distributions. Indeed, extending the weight distribution to any consistent weight, generally penalizes the test performances as observed in our experiments (cf. Sections \ref{synthetic-discuss} and \ref{uci-results-sec}). However, we argue that such penalization could be accepted when considering OOD detection. Further insights on the key components that enable a Bayesian Neural Network (BNN) with a Gaussian prior to exhibit behaviors similar to those of MaxWEnt and MaxWEnt-SVD are provided in Appendix \ref{app-bnn-setting}.

\subsection{SVD-parameterization}
\label{discuss-svd}

The SVD-parameterization has been introduced in Section \ref{svd-param-section} (cf. Equation \ref{SVD-param}) with the aim of allowing a larger increase of the weight entropy while limiting the average empirical risk penalty. We argue, indeed, that using independent weight components in the stochastic model sets the directions of weight distribution expansion to the canonical basis of $\mathbb{R}^d$, which seems intuitively sub-optimal. We could include correlations between weight components as additional parameters to optimize in $\phi$. However, this solution would require the optimization of $\mathcal{O}(d^2)$ parameters which may become intractable, especially for large neural networks such as ResNet \citep{He2016ResNet}, for instance, for which $d > 10^6$. Through the SVD-parameterization, we propose to set the correlation between weight components, at each hidden layer, according to the singular value decomposition of the neuron activation on the training data. Our theoretical analysis in Section \ref{analysis-weight-param} shows, in the case of linear regression, that this weight parameterization provides the same level of average empirical risk as independent weight components but with larger weight entropy.

Previous works consider the use of weight correlations in stochastic model in the form of matrix Gaussian distribution \citep{louizos2016matrixGaussianPrior, sun2017MatrixGaussian} or through more sophisticated models such as weight distributions defined over ``well-chosen'' subspace of $\mathbb{R}^d$ \citep{izmailov2020subspaceInferenceBNN}, as well as normalizing flows \citep{louizos2017multiplicativeNormFlow} and implicit weight models \citep{pawlowski2017HyperNetBayes}. A notable use of correlation between weights is the Laplace approximations \citep{mackay1992practicalBayesian, foong2019inbetweenUncertainty, ritter2018scalableLaplace}, where the correlation matrix for a Gaussian model is given by a ``closed-form'' solution which can be computed using one forward and backward step through the network. Similarities can be observed between the Kronecker Laplace approximation \citep{ritter2018scalableLaplace} and the SVD-parameterization, as both method involve the correlation matrix of the neuron activation, but identifying the link between both methods would require further investigation. In our case, the parameters $\phi$ are still optimized through stochastic variational gradient descent, whereas the Laplace approximation does not require multiple gradient updates. As we manage to find a closed-form expression for $\phi^*$ in the linear case (cf. Propositions \ref{thm-lin-indep} and \ref{thm-lin-notindep}), interesting future work directions include ``Laplace-like'' approximation in the MaxWEnt framework, which can potentially speed up the computation of the parameters $\phi^*$.

Regarding the complexity of the SVD parameterization, we can consider the case of a fully connected neural network with $L$ layers of $b$ neurons each. Computing the SVD decomposition matrix $V$ (cf. Section \ref{svd-param-section}) requires one forward pass of the training inputs and the computation of the SVD decomposition at each layer with complexity $\mathcal{O}(L b^3)$ \citep{pan1999complexityEigenValueDecompo}. Storing the matrices adds $\mathcal{O}(L b^2)$ of memory burden, which is equivalent to $\mathcal{O}(d)$ with $d \in \mathbb{N}$ the dimension of the network weight vector. During the variational gradient descent, the matrix multiplication between the matrix $V$ and the vector $\phi \odot z$ has a complexity of order $\mathcal{O}(L b^3)$. For comparison, a forward pass with a batch of size $B$, for the scaling parameterization, is of complexity $\mathcal{O}(L B b^2)$. If we consider that $b \simeq B$ with $B$ the batch size, we can say that the SVD parameterization requires twice as much computational time as the scaling one, which corresponds approximately to what we observed in our experiments.

\subsection{Entropy function}
\label{entropy-func-discuss}

In the case of scaling (Equation \ref{weight-reparam}) or SVD parameterization (Equation \ref{SVD-param}), we manage to provide an expression of the entropy $H(\phi)$ function of $\phi$ (cf.  Equation \ref{entropy}), which is a convenient property to speed up the MaxWEnt optimization. For other weight parameterizations, one may not be able to derive such a closed-form expression. If the probability density function $q_{\phi}(w)$ can be computed, one can estimate the entropy through sampling, as done for the empirical risk. An alternative solution is to use a proxy of the entropy which is directly linked to the parameters $\phi$. If the entropy is a growing function of $\phi_k$ for any $k \in [|1, d|]$, we propose to consider the following general expression for the penalization term related to the entropy:
\begin{equation}
\label{entropy-general}
H(\phi) = \sum_{k=1}^d g_k(\phi^2_k),
\end{equation}
with $g_k: \mathbb{R}_+ \to \mathbb{R}$ predefined growing functions such that $\phi^2_k$ grows with $H(\phi)$. Typical choices are $g_k(u) = \log(u)$ or $g_k(u) = \sqrt{u}$. In the case, $g_k(u) = \log(u)$, Equation (\ref{entropy-general}) matches the entropy expression derived in Proposition (\ref{entropy-closeform}) within a constant factor. Equation (\ref{entropy-general}) can be seen as a ``proxy'' of the weight ``entropy'' as it increases with $\phi_k$ as the entropy.

We observe in our numerical experiments that alternative functions $g_k$ lead to improved results compared to the logarithm. Our interpretation is that the logarithm function over-penalizes small components $\phi_k$. Then, at initialization, the gradient norm of the objective function is large (as $\phi_k$ are initialized close to $0$) and the first gradient steps push the $\phi_k$ too far. We then advocate for the use of proxy functions, as $g_k(u) = \sqrt{u}$, to improve the optimization stability in practice.


\subsection{Assessing epistemic uncertainty quantification methods}

Our work focuses on better quantifying epistemic uncertainty using Bayesian and ensemble approaches, as we identify the underestimation of epistemic uncertainty as the key reason why standard methods tend to provide overconfident predictions on OOD data.

As discussed in Section \ref{sec-consistent-hypo}, epistemic uncertainty is characterized by the set of consistent hypotheses. We argue that an optimal estimation of epistemic uncertainty should account for the entire set of these consistent hypotheses. This makes evaluating uncertainty quantification methods inherently challenging, as it requires knowledge of the full set of consistent hypotheses, which is the very goal these methods aim to uncover. Additionally, commonly used metrics for evaluating uncertainty—such as negative log-likelihood (NLL), coverage, expected calibration error (ECE)—do not adequately assess epistemic uncertainty. These metrics rely on the ground truth label and evaluate the probability assigned to it or whether it is covered by a predictive interval. However, epistemic uncertainty is more accurately evaluated by determining if all potential labels, given the observation set $\mathcal{S}$, are identified, rather than checking if the correct label is included in the confidence intervals.

Therefore, a commonly used ``proxy'' to evaluate the epistemic uncertainty quantification methods is OOD detection scores, with the underlying assumption that epistemic uncertainty should be larger for OOD data than ID data. This boils down to compute uncertainty scores for test data coming from the same distribution as the training data and an OOD dataset from a different distribution. If epistemic uncertainty is well estimated, the uncertainty scores should be larger for OOD data than test data. AUROC and FPR95 metrics can be used on the uncertainty scores to evaluate the OOD detection ability of the methods \citep{Yang2022openoodBenchmark}. 

A good illustration of the difficulty to evaluate epistemic uncertainty quantification methods with metrics such as NLL or coverage is provided in Figure \ref{citycam-weather-quali} presenting the results of uncertainty quantifiaction for a counting vehicle task. We can see that, after 17:30, because of water drops landing on the camera, many cars are hidden and are difficult to count. The human annotator has only counted visible cars, resulting in an under-estimation of the true number of vehicles present at this time. We can see that Deep Ensemble is rather confident in the number of vehicles and almost matches the number given by the human annotator. On the contrary, MaxWEnt detects the domain shift and triggers large prediction variances. If we evaluate the NLL or the coverage, Deep Ensemble will likely provide better scores than MaxWEnt. However, in this case, MaxWEnt is closer to the ``true'' epistemic uncertainty, as it is not possible to accurately count the actual number of vehicles.

\section{Related Work}

The main related works in distance-based and ensemble-based uncertainty quantification are presented in Section \ref{intro}. The vast uncertainty estimation literature also includes notable methods as conformal prediction \citep{vovk2005conformal, lei2018conformal, angelopoulos2020conformal}, calibration \citep{guo2017calibration, kuleshov2018CalibrationRegression} and evidential learning \citep{sensoy2018evidentialClassif, amini2020deepEvidentialReg}. Our focus in this present work is on the Bayesian and ensemble approaches, for which we propose a specific improvement through the MaxWEnt algorithm. Readers interested in the alternative approaches will find further details in the following surveys \citep{abdar2021UncertaintyQuantificationSurvey, Shen2021OODSurvey}.

\subsection{Deep ensembles and prediction diversity out-of-distribution}

The main challenge, faced by Bayesian and ensemble methods, is the lack of explicit correlation between the prediction diversity and the distance to the training domain, leading to the observation that standard methods in this category often produce over-confident predictions for OOD data  \citep{Angelo2021BayesianNotSuited4OOD, ovadia2019CanYouTrustYourModel, Liu2021PerilDeepOOD}.

As described in Section \ref{intro}, two main approaches are considered to increase the prediction diversity of deep ensemble, especially out-of-distribution: the first approach works on the diversity of the network outputs, gradients or hidden representations \citep{liu1999negativecorrelation, shui2018negativecorrelation, zhang2020NegCorr, ross2020EnsemblesLocallyIndependant,rame2021diceUncertainty, Sinha2021DIBS}. In this category, contrastive approaches make use of auxiliary real or synthetic OOD data \citep{pagliardini2022DBAT, Tifrea2022semiSupervisedOOD, kristiadi2022beingAbitFrequantistBNN, Jain2020MOD, Mehrtens2022MODplus, yu2019OODdetectMaxClassifDisc, wang2022partialContrastiveLongTail}. The second approach works on the hypothesis diversity through random initialization and different architectures \citep{Lakshminarayanan2017DeepEnsemble, Wen2020batchensemble, Wenzel2020hyperDeepEnsemble, Zaidi2021NIPSNES} or by imposing the weight diversity \citep{pearce2018AnchorNetwork, tagasovska2019singleModelUncertainty, Angelo2021RepulsiveDeepEnsemble, deMathelin2023DARE}.

These last methods particularly relate to MaxWEnt. In particular, the DARE algorithm \linebreak \citep{deMathelin2023DARE} produces a deep ensemble at the edge of the consistent hypothesis set by enlarging the network weights while maintaining the loss under an acceptable threshold. However, the DARE approach is limited to the use of the mean squared error loss function with linear end activation. Moreover, the DARE training requires the control of the penalization term to avoid numerical issues when the weights become too large. With the MaxWEnt approach, the training is more stable, as the weight distribution is centered on the weights $\overline{w}$ of a pretrained network. It also works with softmax activation because of the symmetric expansion of the weight distribution.

\subsection{Bayesian neural network priors and stochastic models}

Since the seminal work of Jaynes on Bayesian priors \citep{jaynes1968priorprobability}, an ongoing discussion has been opened about the use of the maximum entropy method for assigning priors in Bayesian modeling. This method, considered ``thought-provoking'' \citep{mackay2003informationTheory}, is generally not recommended \citep{gelman2020Stanwiki}. With the proposed MaxWEnt approach, we do not plan to further extend this discussion. We do not argue that the maximum entropy method is the ``optimal'' way to select a prior, as such a statement depends on the considered notion of optimality. Actually, we advocate for the use of MaxWEnt for OOD detection but do not recommend this method to improve the test accuracy. Enlarging the weight entropy may, indeed, induce a loss of test accuracy due to the large weight variance. However, we show in our experiments that one can always use a ``shrunk'' version of the weight distribution learned by MaxWEnt when looking for accurate inference while sampling over the whole distribution for OOD detection (cf. Section \ref{uci-results-sec}).

The question of the prior choice has been extensively discussed in the Bayesian literature, a recent review provides the main considered approaches \citep{fortuin2021bayesianPriorRevisited}. For Bayesian neural networks, two main groups of priors can be distinguished: weight-space priors and function-space priors. The latter includes priors defined in function space, i.e., over $\mathcal{H}$. Many recent works consider this approach \citep{sun2018functionalVarBayesian, louizos2019functionalNeuralProcess, tran2022allyouneedisGoodFunctionalPrior, fortuin2022PriorsBayesianReview, rudner2023functionSpaceInNN}, which mainly use Gaussian process priors. These methods can be related to the distance-based uncertainty approach, as they make explicit the link between uncertainty and distance to training data through Gaussian processes. The former group corresponds to priors defined over the weights of the neural network, i.e., over $\mathcal{W}$. Our work relates particularly to this approach, as discussed in Section \ref{bayesian-discussion}. The main considered priors in this category are Dropout \citep{gal2016MCdropout, gal2017concreteDropout, boluki2020learnableDropout, nguyen2022OODdropoutEnmbedding}, isotropic Gaussians \citep{zhang2018NoisyNaturalGradientVI, osawa2019practicalDLwithBayesian, jospin2022BNNTutorial}, mixture of Gaussians \citep{blundell2015bayesbybackprop}, hierarchical \citep{wu2018HierarchicalPrior} and horseshoe priors \citep{ghosh2019horshoePrior}. Some methods also propose to define the prior based on empirical observation of the weight distribution of non-Bayesian networks \citep{atanov2018deepWeightPrior, fortuin2021bayesianPriorRevisited}.

Regarding the stochastic model of the weight distribution, previous works have considered the use of diagonal Gaussian \citep{graves2011practicalVINN} and matrix Gaussian to include the weight correlations \citep{louizos2016matrixGaussianPrior, sun2017MatrixGaussian}. In the case of multivariate Gaussian model with fixed mean, approximation methods can be used to derive the posterior distribution without using gradient descent as Laplace approximations \citep{mackay1992practicalBayesian, foong2019inbetweenUncertainty, ritter2018scalableLaplace, kristiadi2020beingAbitBayesian} and tractable approximate Gaussian inference (TAGI) \citep{goulet2021TAGI}. More sophisticated stochastic models have been developed with techniques as normalizing flows \citep{rezende2015variationalNormFlow, louizos2017multiplicativeNormFlow}, implicit distribution \citep{pawlowski2017HyperNetBayes} or distribution defined over subspaces of $\mathcal{W}$ \citep{izmailov2020subspaceInferenceBNN}.

\section{Experiments}
\label{experiments}

We conduct several experiments on both synthetic and real datasets. We primarily focus on OOD detection performances to compare the methods. The implementation details for the MaxWEnt algorithm are presented in Section \ref{impl-choice}. The source code of the experiments is available on GitHub.\footnote{\url{https://github.com/antoinedemathelin/maxwent-expe}}

\subsection{Synthetic experiments}
\label{synth-expe}

In this section, we provide a qualitative analysis of the MaxWEnt algorithm on low dimensional synthetic datasets. Specifically, we compare the uncertainty estimation produced by MaxWEnt and standard ensemble and Bayesian methods.

\subsubsection{Setup}
\label{expe-synthetic-setup}

We consider both classification and regression experiments, performed respectively on the two following datasets:

\begin{itemize}
    \item \textbf{Two Moons Classification}: We consider the \textit{two-moons} classification dataset from scikit-learn\footnote{\url{https://scikit-learn.org/stable/modules/generated/sklearn.datasets.make\_moons.html}} which simulates a two-dimensional binary classification task with moons like distributed classes. The training set is composed of $200$ data points generated from the \textit{two-moons} generator; $50$ additional instances are generated to form a validation dataset. The noise level of the generator is set to $0.1$.
    \item \textbf{1D Regression} : We reproduce the synthetic univariate regression experiment from \citet{Jain2020MOD} with $100$ training and $20$ validation instances. The input instances are drawn in $\mathcal{X} \subset \mathbb{R}$ according to the mixture of two Gaussians centered respectively in $-0.5$ and $0.75$ with standard deviation $0.1$. The outputs $y \in \mathcal{Y} \subset \mathbb{R}$ are drawn according to the conditional distribution: $p(y|x) \sim f^*(x) + \epsilon$ with $\epsilon \sim \mathcal{N}(0, 0.02)$ the noise variable and $f^*(x)$ the ``ground truth'' defined as:
    \begin{equation*}
        f^*(x) = 0.3 \left(x + \sin(2 \pi x) + \sin(4 \pi x) \right).
    \end{equation*}
\end{itemize}

In both experiments, the base estimator is a fully-connected neural network with three layers of $100$ neurons, each with ReLU activations. For classification, the end layer is composed of one layer with sigmoid activation to produce probabilistic outputs. The end layer for regression is made of two neurons which respectively encode for the conditional mean $\mu_w(x)$ and conditional standard deviation $\sigma_w(x)$ of the univariate Gaussian $\mathcal{N}(\mu_w(x), \sigma_w(x))$ as suggested in \citet{nix1994ProbabilisticNetwork} to produce probabilistic outputs in the regression setting. We consider the five following uncertainty quantification methods:
\begin{itemize}
    \item \textbf{Vanilla Network}, the baseline, which produces uncertainty estimation based on the network probabilistic outputs $h_w(x) \in [0, 1]$ for classification and $\sigma_w(x) \in \mathbb{R}_+$ for regression. Notice that an ensemble of Vanilla Networks corresponds to the \textbf{Deep Ensemble} method.
    \item \textbf{MC-Dropout} \citep{gal2016MCdropout}, with dropout rate selected through hold-out validation NLL, computed using the validation data, among $[0.05, 0.1, 0.2, 0.3, 0.5]$;
    \item  standard \textbf{BNN} (Bayesian Neural Network) \citep{mackay1992practicalBayesian, graves2011practicalVINN}, trained with \linebreak stochastic variational inference and reparameterization trick \citep{hoffman2013stochasticVI, Kingma2014Adam}, we use an independent multivariate Gaussian stochastic model $q_{(\mu, \sigma)} \sim \mathcal{N}(\mu, \text{diag}(\sigma^2))$ and a Normal prior $p(w) \sim \mathcal{N}(0, \text{Id})$. Following common practices for variational Bayes approach to BNNs, we consider a trade-off parameter $\lambda$ between the NLL and the KL divergence \citep{wenzel2020howgoodBayesPosterior}. The trade-off parameter is selected in $\{ 10^k \}_{k \in [|-3, 3|]}$ through hold-out validation NLL.
    \item \textbf{MaxWEnt}, with an independent multivariate uniform stochastic model centered on the resulting weights of the Vanilla Network. 
    \item \textbf{MaxWEnt-SVD}, which uses the ``SVD'' parameterization of Equation (\ref{SVD-param}) in addition to the previous MaxWEnt settings.
\end{itemize}

We use the Adam optimizer \citep{Kingma2014Adam} with learning rate $0.001$ and batch size $32$. $10$k iterations are used to train the Vanilla Network and $20$k iterations for other methods, as the stochastic variational inference requires more iterations to converge. For both tasks, the loss function is the Negative Log Likelihood (NLL). It can be written for the respective classification and regression settings as follows:
\begin{gather}
\label{loss-classification}
    \mathcal{L}_{\mathcal{S}}(w) = - \frac{1}{n} \sum_{(x, y) \in \mathcal{S}} y \log\!\left(h_w(x)\right) + (1-y) \log\!\left(1 - h_w(x)\right) \quad \text{(Classification)}\\
\label{loss-regression}
    \mathcal{L}_{\mathcal{S}}(w) = \frac{1}{n} \sum_{(x, y) \in \mathcal{S}} \frac{1}{2} \left( \log(\sigma_w(x)^2) + \frac{\left(y - \mu_w(x)\right)^2}{\sigma_w(x)^2} \right) \quad \text{(Regression)},
\end{gather}
with $h_w \in \mathcal{H}$ the neural network of weights $w \in \mathcal{W}$ such that, for any $x \in \mathcal{X}$, $h_w(x) = (\mu_w(x), \sigma_w(x))$ for the regression setting (cf. \citealt{Lakshminarayanan2017DeepEnsemble}).

To compute uncertainty estimates, we use the entropy metric for classification and the standard deviation of the ``Gaussian mixture approximation'' introduced in \citep{Lakshminarayanan2017DeepEnsemble} for regression. All uncertainty quantification methods except the Vanilla Network produce stochastic outputs, i.e., for any $x \in \mathcal{X}$, $h_w(x)$ is a random variable as $w$ follows a stochastic model. To produce uncertainty estimates at inference, we then compute $P=50$ predictions $\{ h_{w_i}(x) \}_{i \in [|1, P|]}$ with $w_i$ drawn iid according to the learned weight distribution. Then, the uncertainty estimation for each setting becomes, for any $x \in \mathcal{X}$:
\begin{gather}
\label{uncertainty-classification}
    u(x) = - \overline{h}_{w}(x) \log\!\left(\overline{h}_{w}(x)\right) - \left(1-\overline{h}_{w}(x)\right) \log\!\left(1 - \overline{h}_{w}(x)\right) \quad \text{(Classification)}\\
\label{uncertainty-regression}
    u(x) = \frac{1}{P} \sum_{i=1}^P \left(\sigma_{w_i}(x)^2 +  \mu_{w_i}(x)^2 \right) - \overline{\mu}_{w}(x)^2  \quad \text{(Regression)} ,
\end{gather}
with $\overline{h}_w(x), \overline{\mu}_w(x)$, the average of the respective sets $\{ h_{w_i}(x) \}_i$ and $\{ \mu_{w_i}(x) \}_i$. It should be underlined that the uncertainty metric for classification in Equation (\ref{uncertainty-classification}) is the entropy metric applied to the average predicted output over the $P$ stochastic inferences, while the uncertainty metric for regression in Equation (\ref{uncertainty-regression}) is the variance formula for the Gaussian mixture composed of $P$ Gaussians of mean $\mu_{w_i}(x)$ and variance $\sigma_{w_i}(x)^2$ \citep{Lakshminarayanan2017DeepEnsemble}. Notice also that, for the Vanilla Network, the estimated uncertainty is independent of $P$ as the method produces the deterministic outputs $h_w(x)$. In the regression case, the Vanilla Network uncertainty is $u(x) = \sigma_w(x)^2$.

To complete the experiments, we also consider ensembles of the previously mentioned uncertainty quantification methods. We build ensembles of $N = 5$ networks trained independently with different random weight initializations. In this case, the uncertainty metrics are computed in the same way as in the single-network setting through Equation (\ref{uncertainty-classification}) and (\ref{uncertainty-regression}) with $P$ predictions for each network in the ensemble, i.e., with a total of $N P = 250$ predictions.

\subsubsection{Results}

The regression experiment results are reported in Figure \ref{toy-reg-comparison}. Predicted uncertainties for each method are presented in the form of confidence intervals in light blue. We observe that the Deep Ensemble, MC-Dropout and BNN methods provide larger uncertainty estimates out-of-distribution than in-distribution, which offers an efficient way to detect OODs in this case. However, the three methods fail to capture the full epistemic uncertainty, as a significant part of the ground-truth lies outside the confidence intervals. In contrast, MaxWEnt provides relevant confidence intervals outside the training support when extrapolating on the right and left sides of the domain. Although the predicted uncertainties between the two separated parts of the training domain are still under-estimated. This behavior is corrected by MaxWEnt-SVD which fully manages to produce tight confidence intervals in-distribution and uncertainties as large as possible out-of-distribution.

\begin{figure}[h]
    \centering
    \begin{minipage}{0.03\linewidth}
        \begin{tabular}{c}
            \rotatebox[origin=c]{90}{\small Single Net}
        \end{tabular}
    \end{minipage}
    \begin{minipage}{0.185\linewidth}
        \centering
        \includegraphics[width=\linewidth]{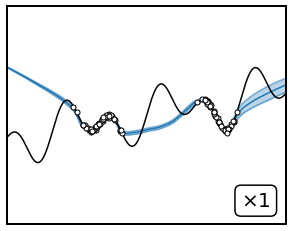} 
    \end{minipage}
    \begin{minipage}{0.185\linewidth}
        \centering
        \includegraphics[width=\linewidth]{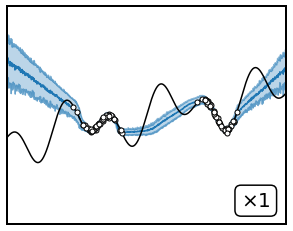} 
    \end{minipage}
    \begin{minipage}{0.185\linewidth}
        \centering
        \includegraphics[width=\linewidth]{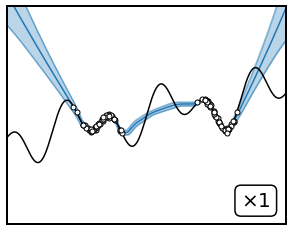} 
    \end{minipage} 
    \begin{minipage}{0.185\linewidth}
        \centering
        \includegraphics[width=\linewidth]{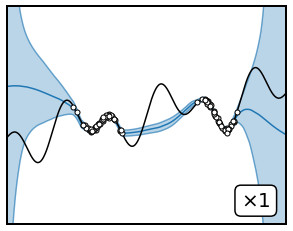} 
    \end{minipage}
    \begin{minipage}{0.185\linewidth}
        \centering
        \includegraphics[width=\linewidth]{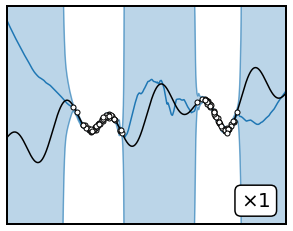} 
    \end{minipage} \\

    \centering
    \begin{minipage}{0.03\linewidth}
        \begin{tabular}{c}
            \rotatebox[origin=c]{90}{\small Ensemble}
        \end{tabular}
        \source{}
    \end{minipage}
    \begin{minipage}{0.185\linewidth}
        \centering
        \includegraphics[width=\linewidth]{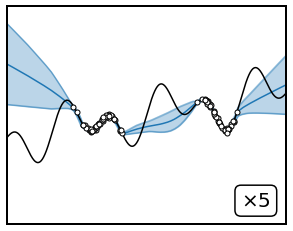}
        \source{Deep Ensemble}
    \end{minipage}
    \begin{minipage}{0.185\linewidth}
        \centering
        \includegraphics[width=\linewidth]{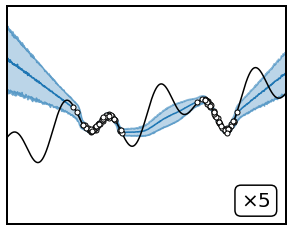}
        \source{MC-Dropout}
    \end{minipage}
    \begin{minipage}{0.185\linewidth}
        \centering
        \includegraphics[width=\linewidth]{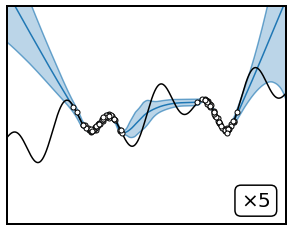}
        \source{BNN}
    \end{minipage} 
    \begin{minipage}{0.185\linewidth}
        \centering
        \includegraphics[width=\linewidth]{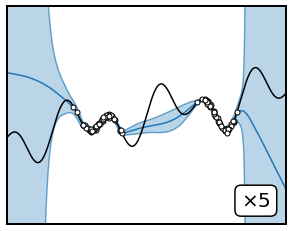}
        \source{MaxWEnt}
    \end{minipage}
    \begin{minipage}{0.185\linewidth}
        \centering
        \includegraphics[width=\linewidth]{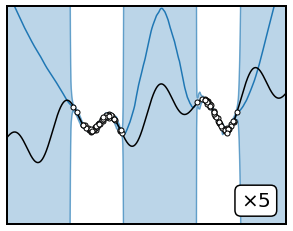} \source{MaxWEnt-SVD}
    \end{minipage}
    \caption{1D-Regression Uncertainty Estimation. The horizontal and vertical axes correspond respectively to the 1D input space $\mathcal{X}$ and the 1D output space $\mathcal{Y}$. The black line represents the ground truth $f^*(x)$ and the blue line the average predictions $\overline{\mu}_w(x)$. Training instances are reported as white dots. Uncertainty estimations are reported in the form of confidence intervals centered around the average prediction (in light blue). The length of the intervals is equal to $4 \sqrt{u(x)}$ with $u(x)$ defined according to Equation (\ref{uncertainty-regression}).}
    \label{toy-reg-comparison}
\end{figure}

The results of the classification experiment are reported in Figure \ref{toy-classif-comparison}. As for the regression experiment, we observe that Deep Ensemble, MC-Dropout and BNN fail to provide relevant uncertainty estimations whereas MaxWEnt and MaxWEnt-SVD are close to the expected behavior of an ideal uncertainty quantifier. Moreover, in this experiment, the first three methods do not offer proper discrimination between out-of-distribution and in-distribution data. The produced uncertainties are concentrated in the margin between classes and do not increase in the OOD areas behind the training instances. We observe that MaxWEnt and MaxWEnt-SVD manage to increase the uncertainty outside the margin between classes.

\begin{figure}[h]
    \centering
    \begin{minipage}{0.03\linewidth}
        \begin{tabular}{c}
            \rotatebox[origin=c]{90}{\small Single Net}
        \end{tabular}
    \end{minipage}
    \begin{minipage}{0.185\linewidth}
        \centering
        \includegraphics[width=\linewidth]{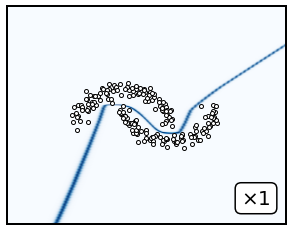} 
    \end{minipage}
    \begin{minipage}{0.185\linewidth}
        \centering
        \includegraphics[width=\linewidth]{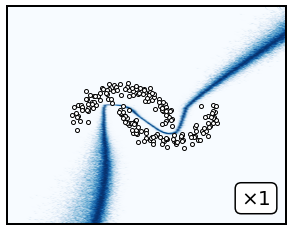} 
    \end{minipage}
    \begin{minipage}{0.185\linewidth}
        \centering
        \includegraphics[width=\linewidth]{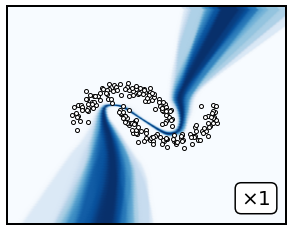} 
    \end{minipage} 
    \begin{minipage}{0.185\linewidth}
        \centering
        \includegraphics[width=\linewidth]{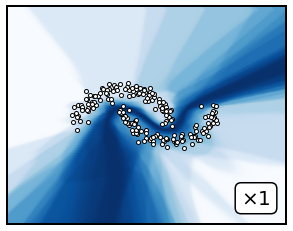} 
    \end{minipage}
    \begin{minipage}{0.185\linewidth}
        \centering
        \includegraphics[width=\linewidth]{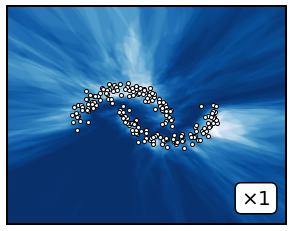} 
    \end{minipage} \\

    \centering
    \begin{minipage}{0.03\linewidth}
        \begin{tabular}{c}
            \rotatebox[origin=c]{90}{\small Ensemble}
        \end{tabular}
        \source{}
    \end{minipage}
    \begin{minipage}{0.185\linewidth}
        \centering
        \includegraphics[width=\linewidth]{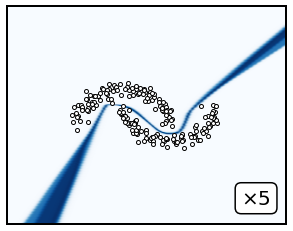}
        \source{Deep Ensemble}
    \end{minipage}
    \begin{minipage}{0.185\linewidth}
        \centering
        \includegraphics[width=\linewidth]{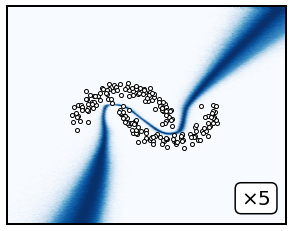}
        \source{MC-Dropout}
    \end{minipage}
    \begin{minipage}{0.185\linewidth}
        \centering
        \includegraphics[width=\linewidth]{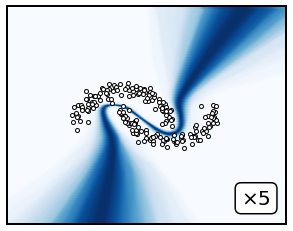}
        \source{BNN}
    \end{minipage} 
    \begin{minipage}{0.185\linewidth}
        \centering
        \includegraphics[width=\linewidth]{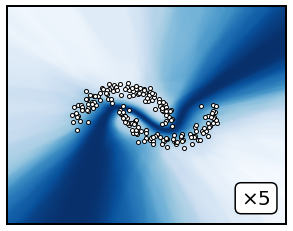}
        \source{MaxWEnt}
    \end{minipage}
    \begin{minipage}{0.185\linewidth}
        \centering
        \includegraphics[width=\linewidth]{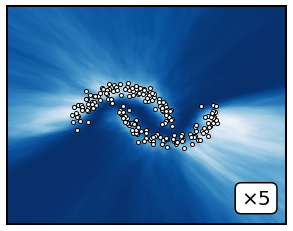} \source{MaxWEnt-SVD}
    \end{minipage}
    \caption{Two-Moons Classification Uncertainty Estimation. The horizontal and vertical axes correspond to both dimensions of the input space $\mathcal{X}$. Training instances are represented by white dots. The two ``moons'' formed by the training instances correspond to two different classes. Predicted uncertainties $u(x)$, computed through Equation (\ref{uncertainty-classification}), are reported in shades of blue (darker areas correspond to larger uncertainties).}
    \label{toy-classif-comparison}
\end{figure}

\subsubsection{Discussion}
\label{synthetic-discuss}

Both experiments on synthetic data strongly highlight the benefit of using MaxWEnt for uncertainty quantification over standard Bayesian and ensemble methods. As discussed in Section \ref{bayesian-discussion}, the MaxWEnt implementation is related to BNN algorithms, however, the predicted uncertainties of MaxWEnt and BNN are very different (cf. Figures \ref{toy-reg-comparison} and \ref{toy-classif-comparison}). These observed discrepancies between the two methods can be explained by their different paradigms. In standard BNN optimization, the main objective is to produce relevant uncertainty estimation inside the training domain. In this perspective, the prior distribution and the trade-off parameters are selected in order to minimize the validation NLL. Consequently, the expansion of the weight distribution is generally limited. In MaxWEnt optimization, the primary goal is to maximize the entropy of the weight distribution as long as the sampled weights are consistent. Although this approach induces a slight penalization of the validation NLL as suggested in Figure \ref{toy-classif-comparison} (predicted uncertainties in the training domain are larger for MaxWEnt and MaxWEnt-SVD than for BNN), it significantly improves the predicted epistemic uncertainties outside the training domain. Notice that one can sample from the whole MaxWEnt weight distribution to detect OOD and then from ``shrunk'' weight distribution to provide more accurate prediction for data identified as in-distribution (cf. Figure \ref{toy-clipping}).

When considering the MaxWEnt-SVD results for both experiments (cf. right side of Figures \ref{toy-reg-comparison} and \ref{toy-classif-comparison}), we might judge that the produced out-of-distribution uncertainties are over-estimated; especially in the regression experiment, where the predicted uncertainties become very large almost instantly at the borders of the training domain. However, this behavior is optimal according to the notion of epistemic uncertainty considered in this work. Indeed, epistemic uncertainty is defined through the set of potential candidates for the optimal hypothesis $h_{w^*}$. Then, as soon as there exists a neural network $h$ in $\mathcal{H}$ which fits the training instances and produces very high outputs out-of-distribution, the user has no reason, in the absence of further regularity consideration, to exclude that the optimal hypothesis can be modeled by $h$. If, for some reason, the user wants to add some prior information on $h_{w^*}$, such as Lipschitz constraints on the network output, this can be achieved, for example, by clipping the scaling variable $\phi \odot z$ during the MaxWEnt inference as done for the weights of the Wasserstein-GAN to impose the $1$-Lipschitz constraint \citep{Arjovsky2017WGAN}. This boils down to considering a reduced hypothesis space $\mathcal{H}$, which de facto reduces the epistemic uncertainty estimation, but potentially increases the discrepancy between $\mathcal{H}$ and $f^*$. We present in Figure \ref{toy-clipping} the impact of weight clipping on the predicted uncertainties of MaxWEnt-SVD on the regression dataset. We observe that the clipping parameter enables the interpolation between the behavior of the vanilla probabilistic network and the MaxWEnt-SVD behavior. Notice that clipping is performed at ``test time'', i.e., after the MaxWEnt optimization, which is convenient as the clipping parameter can be selected ``a posteriori'', without further training.

\begin{figure}[ht]
    \centering
    \begin{minipage}{0.185\linewidth}
        \centering
        \includegraphics[width=\linewidth]{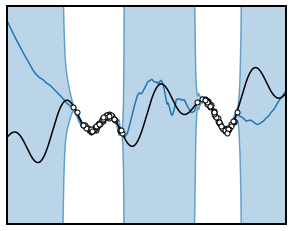} 
        \source{$C = + \infty$}
    \end{minipage}
    \begin{minipage}{0.185\linewidth}
        \centering
        \includegraphics[width=\linewidth]{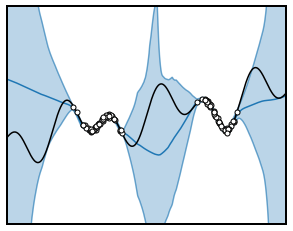} 
        \source{$C = 1.5$}
    \end{minipage}
    \begin{minipage}{0.185\linewidth}
        \centering
        \includegraphics[width=\linewidth]{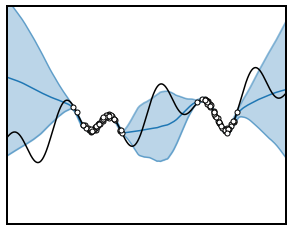} 
        \source{$C = 1$}
    \end{minipage} 
    \begin{minipage}{0.185\linewidth}
        \centering
        \includegraphics[width=\linewidth]{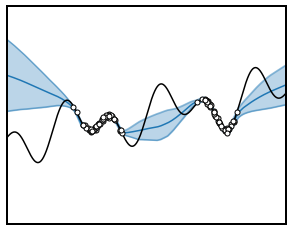} 
        \source{$C = 0.5$}
    \end{minipage}
    \begin{minipage}{0.185\linewidth}
        \centering
        \includegraphics[width=\linewidth]{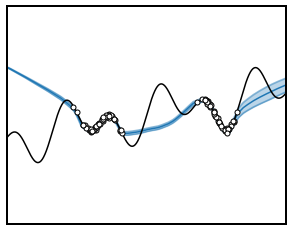} 
        \source{$C = 0$}
    \end{minipage}
    \caption{MaxWEnt-SVD Uncertainties for different clipping parameters. Clipping is performed ``a posteriori'' on the scaling variable $\hat{\phi} \odot z$ (cf. Equation \ref{SVD-param}) with $\hat{\phi}$ the parameters of the fitted MaxWEnt-SVD network, such that $q_{\hat{\phi}} \sim \overline{w} + V \min(\hat{\phi} \odot z, C)$, with $C$ the clipping parameter.}
    \label{toy-clipping}
\end{figure}

The comparison between the regression and classification results suggests that out-of-distribution detection is a more difficult task in the classification setting. Indeed, in this setting, the uncertainty quantification methods do not fully manage to increase uncertainty for OOD data behind the training instances of each class. This behavior can be explained by the use of the sigmoid activation at the end-layer, which hardens the epistemic uncertainty estimation as different large outputs are reduced in the same probabilistic output (close to $1$ if positive or $0$ if negative). In fact, recent out-of-distribution detection methods often abandon the use of softmax and sigmoid activation functions at the end layer in favor of distance-based approaches where class assignment is computed through distance to class prototypes \citep{vanAmersfoort2020DUQ}. Notice that, we do not consider distance-based uncertainty methods in these synthetic experiments. For these low dimensional problems, using the Euclidean distance to the training instances would provide an almost perfect OOD detector. However, for high dimensional datasets, ensemble-based approaches generally provide better performances \citep{Yang2022openoodBenchmark}.

In both experiments, we observe that MaxWEnt-SVD produces uncertainty estimates of better quality than MaxWEnt. The theoretical analysis in Section \ref{analysis-weight-param} suggests that this improvement is related to the weight entropy increase. To evaluate this theoretical claim, we report the evolution of the predicted uncertainties and the weight entropy $H(\phi)$ through the epochs for both methods in the regression setting (cf. Figure \ref{toy-epochs}). We observe, for both methods, a strong correlation between the increase of the weight diversity (measured by $H(\phi)$) and the increase of the uncertainty estimates out-of-distribution. Moreover, the predicted uncertainties of MaxWEnt-SVD quickly increase around epoch $100$ as well as its distribution entropy $H(\phi)$, which becomes higher than the MaxWEnt entropy ($H(\phi) = -0.03$ at epoch 125 for MaxWEnt-SVD while $H(\phi) = -2.51$ for MaxWEnt). After this stage, the predicted OOD uncertainties are better for MaxWEnt-SVD than for MaxWEnt, especially in the interpolation regime between the two parts of the training domain. These observations support the idea that higher weight diversity for the same level of in-distribution risk produces better uncertainty quantification out-of-distribution.

\begin{figure}[ht]
    \centering
    \begin{minipage}{0.03\linewidth}
        \begin{tabular}{c}
            \rotatebox[origin=c]{90}{\small MaxWEnt}
        \end{tabular}
    \end{minipage}
    \begin{minipage}{0.235\linewidth}
        \centering
        \includegraphics[width=\linewidth]{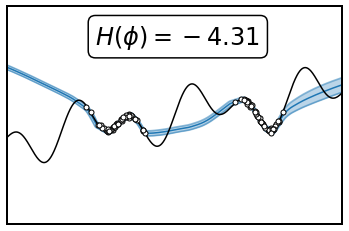} 
    \end{minipage}
    \begin{minipage}{0.235\linewidth}
        \centering
        \includegraphics[width=\linewidth]{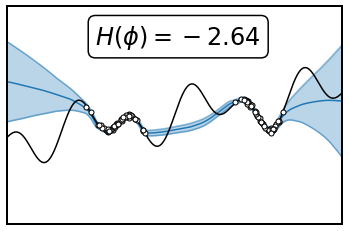} 
    \end{minipage}
    \begin{minipage}{0.235\linewidth}
        \centering
        \includegraphics[width=\linewidth]{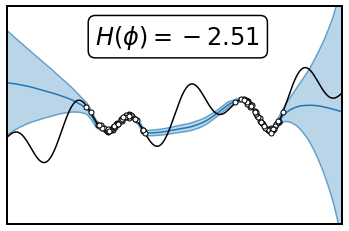} 
    \end{minipage} 
    \begin{minipage}{0.235\linewidth}
        \centering
        \includegraphics[width=\linewidth]{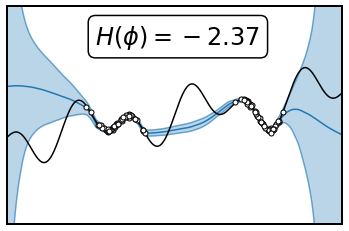} 
    \end{minipage} \\

    \centering
    \begin{minipage}{0.03\linewidth}
        \begin{tabular}{c}
            \rotatebox[origin=c]{90}{\small MaxWEnt-SVD}
        \end{tabular}
        \source{}
    \end{minipage}
    \begin{minipage}{0.235\linewidth}
        \centering
        \includegraphics[width=\linewidth]{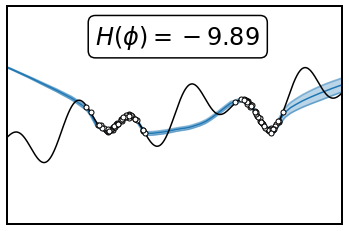}
        \source{Epoch 10}
    \end{minipage}
    \begin{minipage}{0.235\linewidth}
        \centering
        \includegraphics[width=\linewidth]{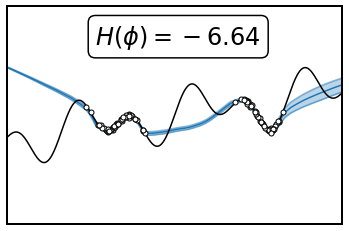} 
        \source{Epoch 75}
    \end{minipage}
    \begin{minipage}{0.235\linewidth}
        \centering
        \includegraphics[width=\linewidth]{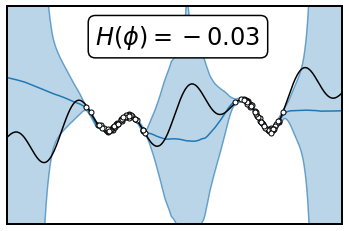}
        \source{Epoch 125}
    \end{minipage} 
    \begin{minipage}{0.235\linewidth}
        \centering
        \includegraphics[width=\linewidth]{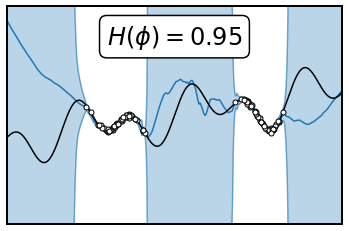}
        \source{Epoch 200}
    \end{minipage}
    \caption{MaxWEnt Uncertainties Evolution through Epochs. One epoch corresponds to $100$ iterations. The entropy term $H(\phi)$ is defined here as $H(\phi) = \frac{1}{d} \sum_{k=1}^d \log(\phi_k^2)$ with $\phi$ the scale parameters of the weight distribution. Notice that MaxWEnt and MaxWEnt-SVD have different parameter initialization, respectively: $\phi_{\text{init}} = \text{softplus}(-5)$ and $\phi_{\text{init}} = \text{softplus}(-10)$.}
    \label{toy-epochs}
\end{figure}

\subsubsection{Neuron Activation Amplitude and Scaling Parameters}
\label{theory-verif}

In the theoretical analysis in Section \ref{analysis}, we show, in the case of fully-connected neural networks, that the scaling parameters $\phi_k$ are inversely proportional to the neuron activation amplitude on the training data denoted $a^2_{(l, k)}$ for the $l^{th}$ layer (cf. Proposition \ref{thm-multi-closeform}). We aim at supporting this theoretical result with empirical observations. For this purpose, we estimate the activation amplitudes in each layer of the MaxWEnt neural network and compare their values with the average of their corresponding scaling parameters $(1/b_l) \sum_{j=1}^{b_l} \phi_{(l, j, k)}$. We report the result in Figure \ref{layer-amplitude}. The top three graphics present the scaling parameters as a function of the activation amplitudes in the three layers of the MaxWEnt neural network trained on the two moons dataset. We observe a clear relation of inverse proportionality between the two quantities, in line with the theoretical outcomes. The three graphics below present the results for the standard BNN method. We observe the inverse proportionality relationship for the first layer but to a lesser extent than for MaxWEnt. This relationship is diminished in the two next layers. Moreover, we observe that the scaling parameters in the two last layers are globally larger for MaxWEnt than for BNN.

\begin{figure}[ht]
    \centering
    \begin{minipage}{0.03\linewidth}
        \begin{tabular}{c}
            \rotatebox[origin=c]{90}{\small MaxWEnt}
        \end{tabular}
        \source{}
    \end{minipage}
    \begin{minipage}{0.31\linewidth}
        \centering
        \includegraphics[width=\linewidth]{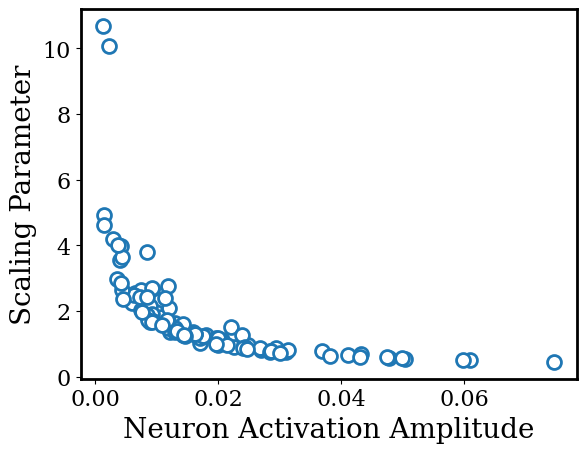} 
    \end{minipage}
    \begin{minipage}{0.29\linewidth}
        \centering
        \includegraphics[width=\linewidth]{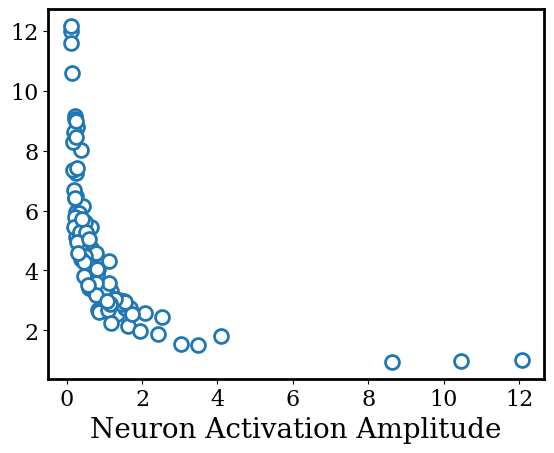} 
    \end{minipage} 
    \begin{minipage}{0.3\linewidth}
        \centering
        \includegraphics[width=\linewidth]{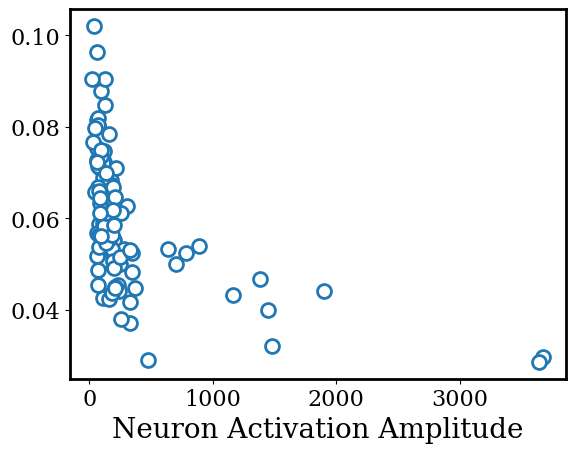} 
    \end{minipage} \\

    \centering
    \begin{minipage}{0.03\linewidth}
        \begin{tabular}{c}
            \rotatebox[origin=c]{90}{\small BNN}
        \end{tabular}
        \source{}
    \end{minipage}
    \begin{minipage}{0.31\linewidth}
        \centering
        \includegraphics[width=\linewidth]{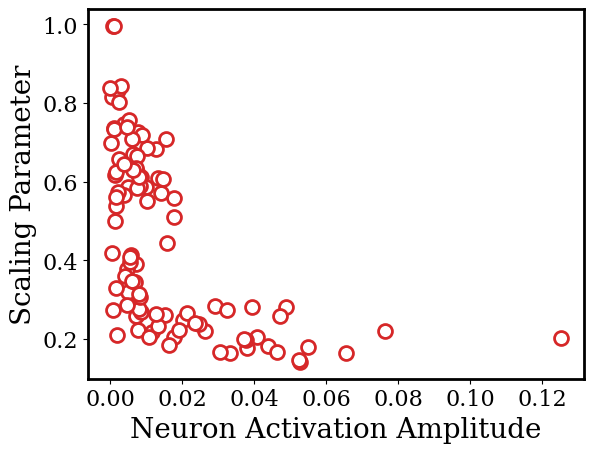} 
        \source{Layer 1}
    \end{minipage}
    \begin{minipage}{0.29\linewidth}
        \centering
        \includegraphics[width=\linewidth]{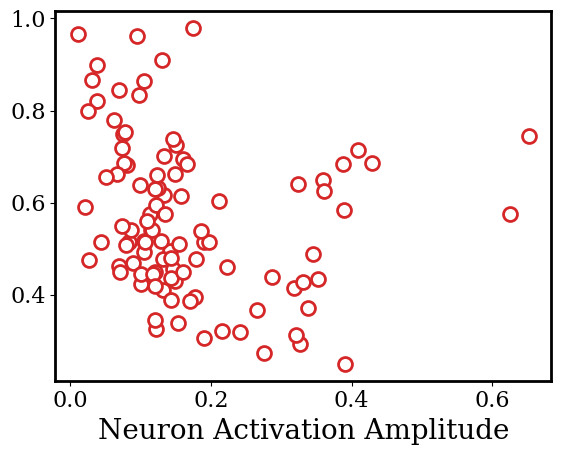} 
        \source{Layer 2}
    \end{minipage} 
    \begin{minipage}{0.3\linewidth}
        \centering
        \includegraphics[width=\linewidth]{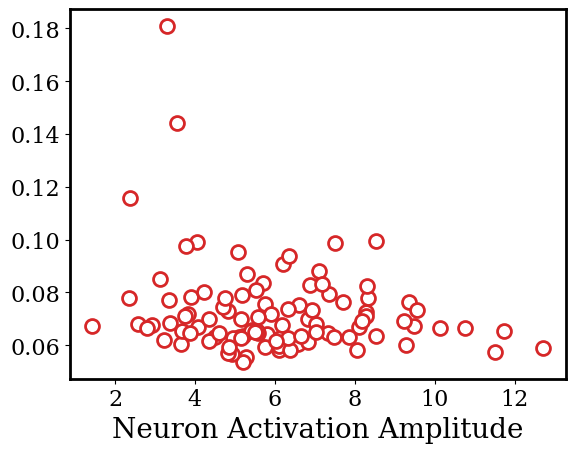} 
        \source{Layer 3}
    \end{minipage}

    \caption{Neuron activation amplitude for the three hidden layers of MaxWEnt and BNN for the synthetic classification experiment: The top three graphics correspond to MaxWEnt while the three bottom to BNN. Each graphic reports the value of the average scaling parameter as a function of the training neuron activation amplitude in the three layers of the neural network.}
    \label{layer-amplitude}
\end{figure}

\subsection{UCI regression datasets}
\label{exp-uci-reg-dataset}

\subsubsection{Setup}

In this section, we consider the most common UCI regression datasets used to evaluate uncertainty quantification methods. Most previous works evaluate the methods based on the in-distribution NLL computed on a test set drawn from the same distribution as the training set \citep{Lakshminarayanan2017DeepEnsemble}. In this work, we focus on the method's ability to detect whether a data point is outside the training support or not. For this purpose, we build OOD detection problems by splitting each dataset into two distinct parts, with one part modeling the training domain and the other part the OOD data. Inspired by \citet{foong2019inbetweenUncertainty} and \citet{Jain2020MOD}, which propose OOD splits for UCI datasets, we split the dataset along the first component of the input PCA: we define the \textit{internal} domain with the data between the 25\% and 75\% percentiles of the first component of the input PCA, while the rest of the data forms the \textit{external} domain. We then consider the two following experimental setups:
\begin{itemize}
    \item \textbf{Extrapolation}: The training data are defined by the \textit{internal} domain, while the data from the \textit{external} domain are considered as OOD.
    \item \textbf{Interpolation}: The training data are defined by the \textit{external} domain, while the data from the \textit{internal} domain are considered as OOD.
\end{itemize}

In all experiments, we consider a fully-connected network with three hidden layers of $100$ neurons each and ReLU activation for the base estimator. The end-layer is composed of two neurons, which respectively predict the conditional mean and standard deviation $\mu_w(x), \sigma_w(x)$ (cf. Section \ref{expe-synthetic-setup}). We consider $13$ different uncertainty quantification approaches: five deep ensemble methods: \textbf{Deep Ensemble} \citep{Lakshminarayanan2017DeepEnsemble}, \textbf{Negative Correlation} \citep{liu1999negativecorrelation, shui2018negativecorrelation}, \textbf{Maximize-Overall-Diversity (MOD)} \citep{Jain2020MOD}, \textbf{Anchored-Networks} \citep{pearce2018AnchorNetwork}, \textbf{Repulsive-Deep-Ensemble (RDE)} \citep{Angelo2021RepulsiveDeepEnsemble} and four ``Bayesian'' methods: \textbf{MC-Dropout}, \textbf{BNN}, \textbf{MaxWEnt}, \textbf{MaxWEnt-SVD} (described in Section \ref{expe-synthetic-setup}), and ensemble version of these four previous Bayesian methods. The competitor characteristics are summarized in Table \ref{table-uci-competitors}. We use the Gaussian NLL loss for regression, as defined in Equation (\ref{loss-regression}) and the Adam optimizer \citep{Kingma2014Adam} with learning rate $0.001$ and batch size $128$. The number of iterations is chosen such that the minimum validation NLL is generally reached by every method on every dataset. We then consider $10$k iterations for ensemble methods and $50$k iterations for Bayesian and Bayesian ensemble methods, as stochastic variational inference converges slower than stochastic gradient descent. A callback process is used to monitor the validation NLL of the model every $100$ iterations, the network weights corresponding to the iteration of best validation NLL are restored at the training end. For MaxWEnt, the scale parameters are saved if the validation NLL is below the threshold defined in Section \ref{impl-choice}. Notice that the OOD dataset is not used during the network training, as well as a test dataset, selected uniformly in the training domain. $10\%$ of the in-distribution data are selected to form the test set and $5\%$ of the remaining data to form the validation set.

\begin{table}[ht]
\footnotesize
\begin{tabularx}{\linewidth}{>{\lratio{1.55}}X>{\lratio{0.4}}Y>{\lratio{0.7}}Y>{\lratio{0.25}}Y>{\lratio{0.25}}Y>{\lratio{0.25}}Y>{\lratio{0.4}}Y>{\lratio{3.2}}Y}
Methods & Abrv. & Kind & Nets & Preds & Total & Parallel & Hyper-parameters \\
\midrule
Deep Ensemble & DE & Ensemble & 5 & 1 & 5 & \checkmark & None \\
Negative Correlation & NC & Ensemble & 5 & 1 & 5 & \checkmark & $\lambda \in \{10^{-k}; \; k \in [|0, 5|]\}$ \\
Maximize-Overall-Diversity & MO & Ensemble & 5 & 1 & 5 & $\times$ & $\lambda \in \{10^{-k}; \; k \in [|0, 5|]\}$ \\
Anchored-Networks & AN & Ensemble & 5 & 1 & 5 & \checkmark & $\Sigma, \lambda \in [0.1, 1, 10]$ \\
Repulsive-Deep-Ensemble & RE & Ensemble & 5 & 1 & 5 & $\times$ & $\sigma = \text{median heuristic}$ \\
\midrule
MC-Dropout $(\times 1)$ & MD & Bayesian & 1 & 50 & 50 & N/A & rate $\in [0.05, 0.1, 0.2, 0.3, 0.5]$ \\
Bayesian Neural Net $(\times 1)$ & BN & Bayesian & 1 & 50 & 50 & N/A & $\lambda \in \{10^k; \; k \in [|-3, 3|]\}$ \\
MaxWEnt $(\times 1)$ & ME & Bayesian & 1 & 50 & 50 & N/A & $\lambda = 10$, $\phi_{\text{init}} = \text{soft}(-5)$ \\
MaxWEnt-SVD $(\times 1)$ & ME+ & Bayesian & 1 & 50 & 50 & N/A & $\lambda = 10$, $\phi_{\text{init}} = \text{soft}(-10)$ \\
\midrule
MC-Dropout $(\times 5)$ & MD & Bay Ens & 5 & 50 & 250 & \checkmark & rate $\in [0.05, 0.1, 0.2, 0.3, 0.5]$ \\
Bayesian Neural Net $(\times 5)$ & BN & Bay Ens & 5 & 50 & 250 & \checkmark & $\lambda \in \{10^k; \; k \in [|-3, 3|]\}$ \\
MaxWEnt $(\times 5)$ & ME & Bay Ens & 5 & 50 & 250 & \checkmark & $\lambda = 10$, $\phi_{\text{init}} = \text{soft}(-5)$ \\
MaxWEnt-SVD $(\times 5)$ & ME+ & Bay Ens & 5 & 50 & 250 & \checkmark & $\lambda = 10$, $\phi_{\text{init}} = \text{soft}(-10)$ \\
\bottomrule
\end{tabularx}
\caption{Competitors Summary. The columns ``Nets'' and ``Preds'' respectively report the number of networks in the ensemble and the number of predictions at inference for one network. ``Total'' is the total number of predictions (Nets $\times$ Preds). The ``Parallel'' column reports whether the ensemble can be trained in parallel or not. When a list is given in the ``Hyper-parameters'' section, the value is selected based on hold-out validation NLL. ``soft'' is the abbreviation for the ``softplus'' function: $\text{soft}(x) = \log(1 + \exp(x))$.}
\label{table-uci-competitors}
\end{table}

\subsubsection{Results}
\label{uci-results-sec}

\begin{table}[ht]
\scriptsize
\centering
\begin{tabular}{l|l|ccccc|cccc|cccc}
\toprule
& \multirow{2}{*}{\backslashbox{Data}{Meth}} & \multicolumn{5}{c|}{\textbf{Ensemble}} & \multicolumn{4}{c|}{\textbf{Bayesian}} & \multicolumn{4}{c}{\textbf{Bayesian Ensemble}} \\
  &  {} &             DE &        NC &            MO &      AN &            RE & BN & MD &  ME & ME+ & BN & MD &  ME & ME+ \\
\midrule
\midrule
\multirow{12}{*}{\rotatebox[origin=c]{90}{\textbf{Extrapolation}}}  & yacht    & 98.9          & 99.1          & 99.1 & 98.1          & \textbf{99.5} & 89.5 & 78.2 & 97.1          & \textbf{99.4} & 95.3 & 83.1 & \textbf{99.6} & 99.1          \\
 & energy   & 81.0          & \textbf{93.6} & 91.3 & 79.9          & 92.9          & 88.2 & 55.6 & 74.3          & \textbf{99.6} & 92.0 & 81.9 & 91.7          & \textbf{99.9} \\
 & concrete & 78.4          & \textbf{89.8} & 88.9 & 83.8          & 87.7          & 75.7 & 68.7 & 74.3          & \textbf{90.8} & 79.5 & 72.1 & 81.8          & \textbf{95.6} \\
 & wine     & 38.8          & \textbf{48.7} & 36.8 & 45.8          & 39.3          & 70.9 & 62.3 & 79.1          & \textbf{85.9} & 66.7 & 64.2 & 83.8          & \textbf{88.4} \\
 & power    & \textbf{84.9} & 78.5          & 75.3 & 75.1          & 79.4          & 82.1 & 79.8 & 78.4          & \textbf{93.0} & 82.4 & 86.7 & 93.1          & \textbf{93.3} \\
 & naval    & 97.5          & 97.7          & 85.3 & \textbf{99.7} & 96.0          & 89.5 & 96.1 & 96.9          & \textbf{97.2} & 96.8 & 96.8 & 98.9          & \textbf{99.6} \\
 & protein  & 82.5          & \textbf{83.0} & 82.8 & 78.0          & 79.7          & 82.9 & 74.7 & \textbf{81.6} & 79.6          & 84.0 & 79.9 & \textbf{89.3} & 87.6          \\
 & kin8nm   & 45.4          & 45.0          & 45.0 & 45.9          & \textbf{46.1} & 52.5 & 51.4 & 39.1          & \textbf{60.3} & 54.5 & 52.8 & 49.1          & \textbf{78.2} \\
\cmidrule{2-15}
\RowColorb \cellcolor{white} & \textbf{Avg AUC}   &           75.9 &           \textbf{79.4} &           75.6 &           75.8 &      77.6      &    78.9  &     70.8 &           77.6 &             \textbf{88.2} &     81.4 &           77.2 &           85.9 &             \textbf{92.7} \\
\RowColorb \cellcolor{white} &  \textbf{Rank}  &           9 &           \textbf{5} &          11 &           10 &           7 &     6 &          12 &           7 &             \textbf{2} &     4 &          8 &     3 &             \textbf{1} \\
\midrule
\midrule
\multirow{12}{*}{\rotatebox[origin=c]{90}{\textbf{Interpolation}}}   & yacht    & 77.5 & 78.4          & \textbf{80.6} & 76.1          & 79.7          & 46.7 & 48.4 & 71.4 & \textbf{98.9} & 51.9 & 48.0 & 90.1 & \textbf{98.6}  \\
 & energy   & 99.2 & \textbf{99.7} & 99.6          & 98.7          & 99.5          & 95.5 & 78.8 & 88.5 & \textbf{99.5} & 98.2 & 96.4 & 98.8 & \textbf{100.0} \\
 & concrete & 60.8 & 72.7          & 72.4          & 46.2          & \textbf{73.6} & 48.6 & 57.4 & 46.7 & \textbf{93.4} & 60.3 & 60.6 & 62.5 & \textbf{95.0}  \\
 & wine     & 43.3 & 42.7          & 42.7          & 41.1          & \textbf{43.8} & 34.8 & 41.6 & 32.2 & \textbf{52.5} & 33.9 & 41.0 & 37.0 & \textbf{62.1}  \\
 & power    & 43.5 & 42.8          & 17.8          & \textbf{67.3} & 48.5          & 38.1 & 42.7 & 58.6 & \textbf{94.7} & 57.2 & 47.1 & 65.5 & \textbf{96.0}  \\
 & naval    & 81.8 & 73.5          & 73.6          & \textbf{83.6} & 71.2          & 22.8 & 83.8 & 54.6 & \textbf{98.6} & 46.9 & 91.5 & 88.9 & \textbf{98.4}  \\
 & protein  & 70.6 & 72.5          & 65.9          & \textbf{73.7} & 73.3          & 66.0 & 71.6 & 64.2 & \textbf{83.8} & 71.0 & 76.1 & 71.9 & \textbf{80.6}  \\
 & kin8nm   & 63.7 & 63.2          & 63.3          & 64.7          & \textbf{64.8} & 53.1 & 56.2 & 55.8 & \textbf{67.1} & 54.2 & 58.1 & 56.9 & \textbf{67.7}  \\
\cmidrule{2-15}
\RowColorb \cellcolor{white} &  \textbf{Avg AUC}  &           67.6 &           68.2 &           64.5 &           68.9 &           \textbf{69.3} &     50.7 &          60.1 &           59.0 &             \textbf{86.1} &     59.2 &          64.8 &           71.4 &             \textbf{87.3} \\
\RowColorb \cellcolor{white} & \textbf{Rank}  &   7 &   6 &           9 &           5 &           \textbf{4} &   13 &          10 &           12 &             \textbf{2} &     11 &          8 &           3 &             \textbf{1} \\
\bottomrule
\bottomrule
\end{tabular}
\caption{UCI experiments OOD detection results. AUROC scores for OOD detection are reported. The best score for each category is emphasized in bold (higher scores are better). The two last rows for the extrapolation and interpolation settings report the average AUROC over the eight datasets (Avg AUC) and the rank of the method among all methods according to the average score (Rank).}
\label{uci-results}
\end{table}

To evaluate the model performances, we use the metric defined in Equation (\ref{uncertainty-regression}) which defines an uncertainty score for each data point, this score is used to compute the AUROC metric between in-distribution and OOD data which is a commonly used metric in the OOD detection setting \citep{Yang2022openoodBenchmark}. All results are reported in Table \ref{uci-results}. Each experiment is performed only once to reduce the computational time of the experiments. As many different datasets are used, this is sufficient to obtain statistically significant results. We report the results by kind of methods: ensemble, Bayesian and Bayesian ensemble. The best result for each dataset in each category is emphasized in bold. We report the average AUROC among extrapolation and interpolation experiments and the rank of the methods.  Our observations can be summarized as follows:

\begin{itemize}
    \item \textbf{MaxWEnt-SVD (ME+) outperforms all other approaches}, with or without ensembling. The second-best non-MaxWEnt approach is 11.3 points behind in extrapolation and 18 points behind in interpolation in terms of average AUROC. Ensembling improves from 4.5 points in extrapolation and 1.2 points in interpolation.
    \item \textbf{The ensemble version of MaxWEnt (ME) is third best} behind the two versions of \linebreak MaxWEnt-SVD. The single-network MaxWEnt, however, provides poor performances,  \linebreak which advocates for the use of ensembling or SVD parameterization.
    \item \textbf{AUROC scores are higher in extrapolation than in interpolation}, suggesting that the second task is more difficult. This seems reasonable, as the network is conditioned on both sides of the OOD domain in the interpolation case while being conditioned only on one side of the OOD domain in extrapolation.
    \item \textbf{Ensembling of Bayesian methods generally improves the results compared to the single-net from 7 points on average}. However, using Bayesian combined in ensemble increases the training and inference time by the number of members as well as the required memory size. Note that, for these methods, the ensemble training can be conducted in parallel, which can alleviate the training time burden.
\end{itemize}

\begin{table*}[ht]
\scriptsize
\centering
\begin{tabular}{l|l|ccc|cccc|cccc}
\toprule
& \multirow{3}{*}{Metric} & \multicolumn{3}{c|}{\textbf{Baselines}} & \multicolumn{4}{c|}{\textbf{MaxWEnt}} & \multicolumn{4}{c}{\textbf{MaxWEnt + Clip}} \\
  &  {} &  DE &   BN1 &  BN5  &  ME1 & ME1+ & ME5 & ME5+ &  ME1 & ME1+ & ME5 & ME5+  \\
\midrule
 & Avg NLL &         -0.69 & -0.61 &	\textbf{-0.75}  &             -0.44 &                -0.33 &      -0.41 &          -0.04 &             -0.61 &             -0.71 &                -0.59 &                -0.71 \\
& Avg ECE &          0.37 &  0.37 &	0.35 &              0.36 &                 0.33 &        0.36 &           0.39 &              0.31 &              0.36 &                 \textbf{0.29} &                 0.35 \\
 & Cov. & 0.61 & 0.58 & 0.65 & 0.73 & 0.75 & 0.79 & \textbf{0.82} & 0.65 & 0.68 & 0.63 & 0.64 \\
 & Int. Width & \textbf{1.00} & 1.33 & 1.85 & $>50$ & $>50$ & $>50$ & $>50$ & 3.07 & 1.19 & 1.07 & 1.12 \\
 & Coverage Conform & 0.91 & 0.93 & 0.86 & 0.93 & 0.92 & 0.91 & \textbf{0.94} & 0.93 & 0.93 & 0.90 & 0.90 \\
 \multirow{-6}{*}{\rotatebox[origin=c]{90}{\textbf{Extrapolation}}} & Int. Width Conform & 1.00 & 1.41 & 1.68 & $>50$ & $>50$ & $>50$ & $>50$ & 3.34 & 1.42 & \textbf{0.99} & \textbf{0.99} \\
\midrule
 & Avg NLL &         -0.45 & -0.49 &	\textbf{-0.54} &             -0.26 &                -0.12 &       -0.23 &          -0.06 &             -0.27 &             -0.45 &                -0.28 &                -0.45 \\
& Avg ECE &          0.33 &  0.32 &	\textbf{0.30} &              0.32 &                 \textbf{0.30} &        \textbf{0.30} &           0.34 &              0.33 &              0.33 &                 0.33 &                 0.33 \\
& Coverage & 0.65 & 0.57 & 0.64 & 0.74 & 0.78 & 0.81 & \textbf{0.84} & 0.59 & 0.59 & 0.65 & 0.65 \\
& Int. Width & 1.00 & 1.21 & 1.74 & $>50$ & $>50$ & $>50$ & $>50$ & 0.90 & \textbf{0.88} & 1.00 & 1.02 \\
& Coverage Conform & 0.91 & 0.90 & 0.90 & 0.93 & \textbf{0.96} & \textbf{0.94} & 0.92 & 0.93 & \textbf{0.94} & 0.91 & 0.91 \\
\multirow{-6}{*}{\rotatebox[origin=c]{90}{\textbf{Interpolation}}} & Int. Width Conform & 1.00 & 1.25 & 1.68 & $>50$ & $>50$ & 3.54 & $>50$ & 1.14 & 1.20 & 1.00 & \textbf{0.99} \\
\bottomrule
\end{tabular}
\caption{UCI experiments In-distribution performances. The average Negative Log Likelihood (NLL) and Expected Calibration Error (ECE) over the eight datasets are reported. The scores are 
computed on the test set, the lower the score the better. The number at the end of the acronyms corresponds to the number of networks (ME1 refers to a MaxWEnt single network and ME5 to an ensemble of 5 MaxWEnt networks). The coverage target level is 95\%. ``Int. Width'' indicates the average confidence interval width and ``Conf'' denotes conformalized metrics, i.e., the metrics computed after conformalization of the intervals on the validation dataset.}
\label{uci-results-id}
\end{table*}

Finally, to evaluate the in-distribution performance of the methods, we compute, on the test set, the Negative Log Likelihood (NLL) as well as the Expected Calibration Error (ECE) \citep{levi2022ECEregression}. To further assess the in-distribution performance of the methods, we also introduce two additional metrics: coverage and average interval width. Coverage measures the proportion of labels that fall within the predicted uncertainty intervals, while interval width evaluates the average length of these intervals. The average metrics computed over the eight datasets are reported in Table \ref{uci-results-id}. To evaluate the impact of clipping on the in-distribution performance, we also report the average metrics for the ``clipped'' MaxWEnt weight distribution: $q_{\phi} \sim \overline{w} + \min(\phi \odot z, C)$ (independent) and $q_{\phi} \sim \overline{w} + V \min(\phi \odot z, C)$ (SVD), with $C$ the clipping parameter selected in $[+\infty, 10, 5, \allowbreak 2, 1, 0.5, 0.2, 0.1, 0]$ according to the validation NLL performance. We observe that the MaxWEnt algorithms generally penalize the test NLL and ECE compared to the baselines. In particular, the average NLL of MaxWEnt-SVD (x5) is larger than the ones produced by the other methods, suggesting that stronger OOD detection results come with weaker test performances. However, we observe that the use of weight clipping improves the MaxWEnt test performances, which become comparable to those of the baselines. Our results reveal that, in the absence of clipping, the predicted uncertainty intervals tend to be overestimated, which, while resulting in better coverage, leads to excessively large interval widths. After applying clipping, we observe that both the length of the intervals and the coverage improve, becoming comparable to those achieved by the baseline methods, especially for the MaxWEnt-SVD (x5) and MaxWEnt-SVD (x5) (Me5 and Me5+). These results suggest that the user should use the ``unclipped'' MaxWEnt predicted uncertainties to perform OOD detection and the ``clipped'' MaxWEnt inferences to provide predictions for data identified as in-distribution. This requires two different inferences: one for OOD detection and one for prediction.

We emphasize that clipping is not used during training; but only in inference after the model has been trained. Additionally, clipping is not essential for out-of-distribution (OOD) detection, but rather for in-distribution (ID) prediction. Therefore, the choice of clipping parameter does not affect the training nor the OOD detection score. Regarding the impact of clipping on the entropy, clipping intuitively reduces the weight entropy as suggested by the reduction of the prediction uncertainty shown in Figure \ref{toy-clipping} for the synthetic experiment. Deriving the formula of the weight entropy after clipping is not straightforward because the clipped distribution is a mixture of continuous and discrete distributions. Insights to compute the entropy in such cases can be found in \citep{nair2006entropyMixture}.



\subsection{CityCam regression datasets}

\subsubsection{Setup}

This section is dedicated to uncertainty quantification on the real-world dataset CityCam \citep{Zhang2017WebCamT}. This dataset is composed of images gathered from several cameras monitoring the traffic in a city. Each camera records between $1$k and $6$k images dispatched over several days and hours. The task consists in counting the number of vehicles in the image using a neural network. This task is useful, for instance, to monitor the traffic in the city. To produce in-distribution vs OOD splits, we consider the three following experiments:
\begin{itemize}
    \item \textbf{Camera-Shift}: Images coming from ten different cameras are selected for this experiment. At each round, five cameras are randomly selected to form the training dataset, while the five remaining cameras are used as OOD dataset. On average, both datasets contain around $20$k images.
    \item \textbf{BigBus-Shift}: Images from five cameras are considered in this experiment. Some of them are marked as ``big-bus'' if a large vehicle masks a significant part of the image (cf. \citealt{Zhang2017WebCamT}). These images are selected to form the OOD dataset, while the remaining ones compose the training set. The in-distribution and OOD datasets respectively contain around $17$k and $1$k images.
    \item \textbf{Weather-Shift}: For this experiment, we consider the images gathered from three cameras recorded during February the $23^{\text{th}}$ from 9 am to 6 pm. On this particular day, weather conditions changed considerably between the beginning and end of the day. The dataset is split into two subsets: images recorded before $2$ pm are considered in-distribution, while the others are out-of-distribution. After $4$ pm, water drops landed on the cameras blur the images, causing a clear domain shift (cf. Table \ref{citycam-setup}).
\end{itemize}

\begin{table}[h]
\small
\centering
\begin{tabular}{lccc}
\toprule
& \textbf{Camera-Shift} & \textbf{BigBus-Shift} & \textbf{Weather-Shift} \\
\midrule
\RowColorb \rotatebox[origin=l]{90}{\quad \textbf{In-Dist}} & \includegraphics[width=0.2\linewidth]{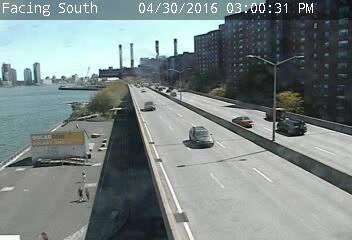} & \includegraphics[width=0.2\linewidth]{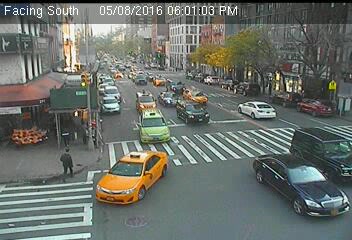} & \includegraphics[width=0.2\linewidth]{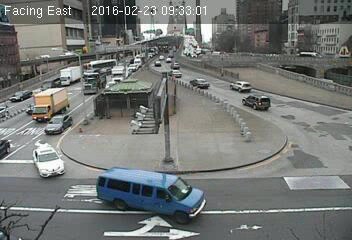} \\
\midrule
\RowColorc \rotatebox[origin=l]{90}{\quad \textbf{OOD}} & \includegraphics[width=0.2\linewidth]{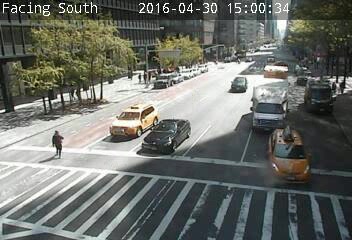} & \includegraphics[width=0.2\linewidth]{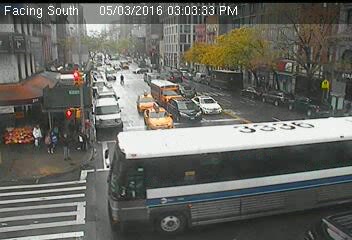} & \includegraphics[width=0.2\linewidth]{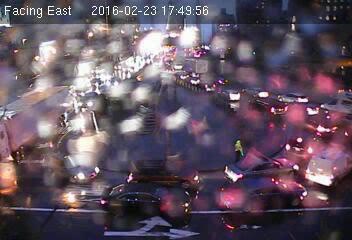} \\
\bottomrule
\end{tabular}
\caption{CityCam Experiments setup. An example of a webcam image is given for each domain for the three settings: Camera-Shift, BigBus-Shift and Weather-Shift.}
\label{citycam-setup}
\end{table}

The three previous experiments model different out-of-distribution scenarios. OOD data for the BigBus-Shift and Weather-Shift experiments can be considered as ``anomalies''. When a large vehicle masks an important part of the image or when the images become too blurry due to raindrops, it becomes very difficult to produce accurate predictions even for a human (cf. Table \ref{citycam-setup}). In this case, the user may expect uncertainty quantification methods to provide large prediction uncertainty in order to detect such abnormal events. The paradigm slightly differs for the Camera-Shift experiment. In this setting, the domain shift essentially lies in the background differences between cameras. Since the model is trained on five different cameras, the user might expect the model to ``generalize'' and to provide accurate predictions for the images of the novel cameras. 

As preprocessing, we use the features of the last layer of a ResNet50 \citep{He2016ResNet} pretrained on ImageNet \citep{deng2009imagenet}. We consider the same setting as for the UCI experiments in terms of base estimator, optimization parameters, callbacks and competitors. 


\subsubsection{Results}

For each experiment, we compute the AUROC metric and the False Positive Rate at 95 percent (FPR@95) \citep{Yang2022openoodBenchmark}, using the uncertainty scores given in Equation (\ref{uncertainty-regression}). The computed metrics are reported in Table \ref{citycam-results}. We observe an important discrepancy between the scores produced by MaxWEnt-SVD and the ones of other methods. The gap is particularly large for the Camera-Shift experiments, where every method produces an average FPR@95 score around 97\% while MaxWEnt-SVD provides a false positive rate of 29.4\% in the single-net setting and 15.3\% with ensembling. Similarly, MaxWEnt-SVD outperforms every other method for the BigBus-Shift and Weather-Shift experiments. The MaxWEnt algorithm without SVD parameterization provides the second-best results in the Bayesian and ensemble category, however, the performance gains compared to the baselines are much smaller than the ones obtained with the SVD parameterization. Notice, however, that MaxWEnt-SVD requires more computational time because of the additional matrix multiplication caused by the SVD alignment (cf. Section \ref{discuss-svd}).

\begin{table}[h]
\scriptsize
\centering
\begin{tabular}{l|cc|cc|cc}
\toprule
\multirow{2}{*}{Method} & \multicolumn{2}{c|}{Camera-Shift} & \multicolumn{2}{c|}{BigBus-Shift} & \multicolumn{2}{c}{Weather-Shift} \\
{} & FPR@95 & AUROC & FPR@95 & AUROC & FPR@95 & AUROC \\
\midrule
\midrule
DE             &          98.3 (1.4) &          52.1 (4.9) &         82.0 (2.2) &         77.9 (1.3) &          79.7 (2.1) &          77.5 (1.0) \\
NegCorr        &          \textbf{95.6 (0.6)} &          56.5 (4.3) &         78.4 (3.6) &         \textbf{79.9 (1.0)} &          80.0 (1.1) &          78.5 (1.9) \\
MOD            &          97.0 (1.7) &          \textbf{57.2 (2.2)} &         78.0 (4.0) &         79.0 (2.1) &          76.7 (2.4) &          78.5 (1.9) \\
AnchorNet      &          99.4 (0.4) &          51.0 (5.9) &         84.0 (1.7) &         78.2 (0.9) &          \textbf{73.4 (7.2)} &          \textbf{80.9 (3.2)} \\
RDE            &          97.4 (0.4) &          54.6 (3.9) &         \textbf{78.0 (1.4)} &         78.4 (1.1) &          77.1 (1.0) &          78.0 (0.6) \\
\midrule
BNN (x1)       &          98.0 (2.8) &          51.0 (2.3) &         93.3 (2.1) &         62.3 (7.6) &          76.6 (1.4) &          76.7 (1.8) \\
MCDropout (x1) &    99.9 (0.1) &	43.5 (4.4) &	92.2 (1.4) &	71.7 (1.8) &	77.1 (4.3) &	77.7 (1.9) \\
MaxWEnt (x1)   &          95.4 (0.0) &          51.2 (0.0) &         86.6 (0.0) &         78.7 (0.0) &          70.4 (0.0) &          77.3 (0.0) \\
MaxWEnt-SVD (x1) &          \textbf{29.4 (6.3)} &          \textbf{92.3 (2.5)} &         \textbf{57.5 (5.9)} &         \textbf{87.0 (2.5)} &          \textbf{61.1 (3.0)} &          \textbf{85.7 (0.7)} \\
\midrule
BNN (x5)       &          98.1 (2.5) &          53.5 (2.9) &         94.1 (1.4) &         64.0 (7.9) &          75.3 (1.9) &          80.2 (1.1) \\
MCDropout (x5) &     99.8 (0.1) &	56.5 (1.8) &	87.4 (1.4) &	78.0 (0.1) &	76.1 (2.5) &	80.6 (2.7) \\
MaxWEnt (x5)   &          93.6 (2.1) &          58.5 (5.9) &         79.1 (4.9) &         80.5 (1.2) &          67.8 (2.7) &          80.8 (0.3) \\
MaxWEnt-SVD (x5) &          \textbf{15.3 (6.3)} &          \textbf{96.9 (1.5)} &         \textbf{53.5 (3.4)} &         \textbf{88.6 (0.7)} &          \textbf{59.8 (7.6)} &          \textbf{86.7 (2.5)} \\
\bottomrule
\end{tabular}
\caption{CityCam Experiments: OOD Detection Results. Average AUROC and FPR@95 over three repetitions of the experiment are reported for each dataset and each method.}
\label{citycam-results}
\end{table}

A visualization of the MaxWEnt uncertainty evolution on the Weather-Shift experiment is presented in Figure \ref{citycam-weather-quali}. We compare the evolution of the confidence intervals produced by Deep Ensemble and MaxWEnt (x1) along the day. The left part of Figure \ref{citycam-weather-quali} corresponds to the images recorded between 2:00 pm and 2:30 pm which are the closest OOD data to the training domain. We observe that, in this time interval, both methods produce tight uncertainty intervals which well cover the ground-truth. The right part of Figure \ref{citycam-weather-quali} corresponds to the time interval 4:00 pm to 6:00 pm. During this period of time, raindrops progressively land on the camera objective and blur the image. At some point around 5:30 pm, the deterioration of the image becomes critical for the vehicles' counting. We observe that, in this case, the size of the confidence intervals produced by Deep Ensemble does not increase. Paradoxically, the Deep Ensemble method seems to produce more confident predictions around 5:30 pm than before 2:30 pm. Conversely, the MaxWEnt predicted uncertainty progressively grows after 5:00 pm in correlation with the increasing task difficulty. Notice that, at some point, even the ground-truth is not reliable anymore, as the human annotator was not able to accurately count the actual number of vehicles.

\begin{figure}[ht]
\center
\includegraphics[width=0.9\linewidth]{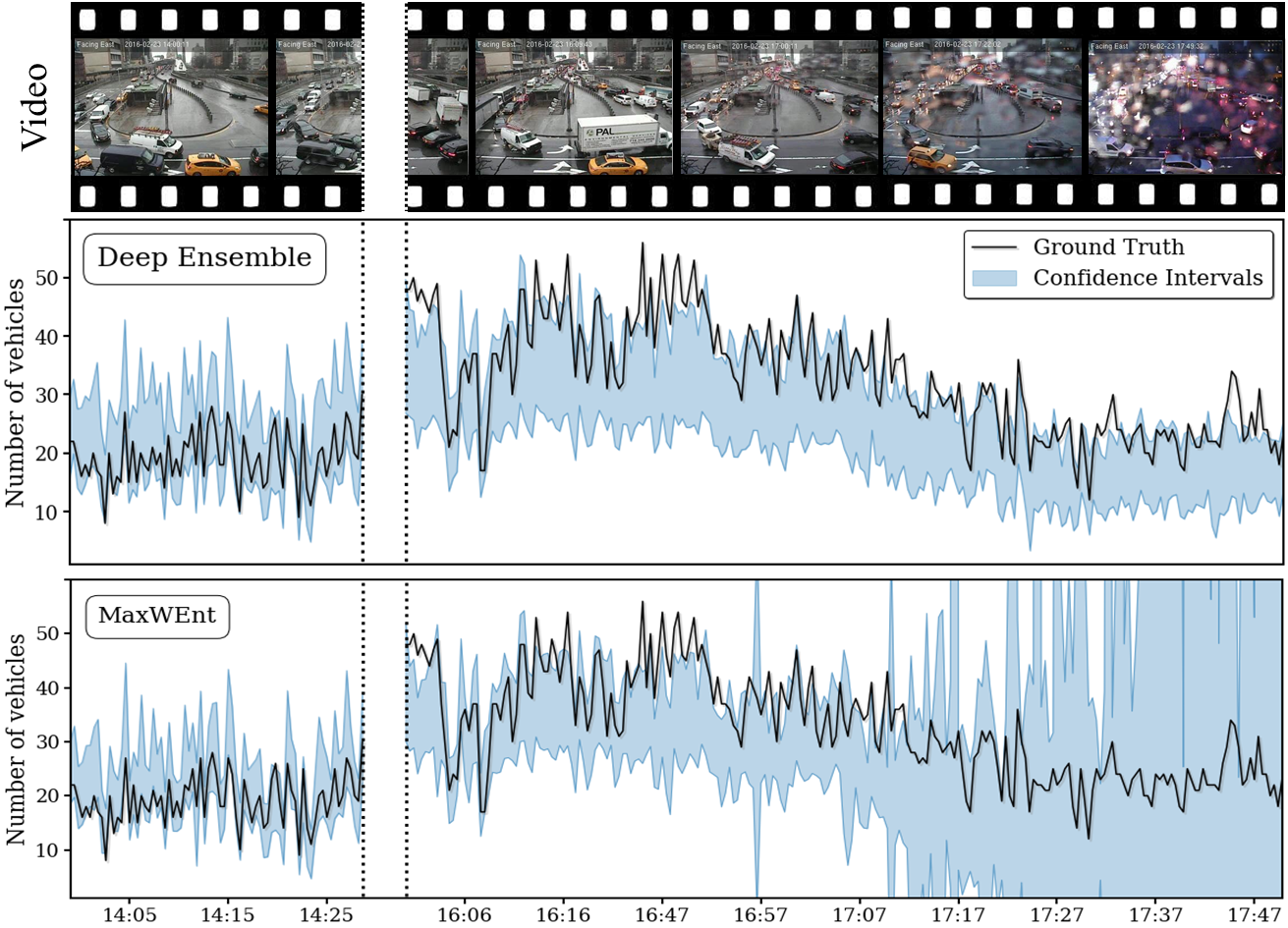}
\caption{Comparison of the uncertainty evolution across time for Deep Ensemble and MaxWEnt on the Weather-Shift OOD dataset. The top images are examples of the camera's recording. The length of confidence intervals (in light blue) is equal to $2 \sqrt{u(x)}$.}
\label{citycam-weather-quali}
\end{figure}

\subsection{OSR-OOD detection benchmark on classification datasets}

\subsubsection{Setup}

We consider the Open-Set-Recognition (OSR) and Out-of-Distribution detection extensive benchmark (OpenOOD), developed in \citet{Yang2022openoodBenchmark} which compares more than $30$ OSR and OOD detection methods on various classification datasets. The source code for the MaxWEnt experiments, conducted within the OpenOOD benchmark, is available on GitHub.\footnote{\url{https://github.com/antoinedemathelin/OpenOOD}} We focus on the OSR and OOD detection experiments:

\begin{itemize}
    \item \textbf{Open-Set-Recognition}: For the OSR benchmark, each dataset is divided into two parts by removing the instances corresponding to some classes from the training set. The goal is to detect whether an instance comes from a training class or a removed one. Each experiment is repeated five times with random selections of the training classes. Four datasets are considered: MNIST \citep{deng2012mnist}, CIFAR10, CIFAR100 \citep{krizhevsky2009cifar10} and TinyImageNet \citep{torralba2008tinyImageNet}.
    \item \textbf{Out-Of-Distribution Detection}: For the OOD detection benchmark, data coming from all classes are used at training time. The goal is then to discriminate between the test set and data coming from other datasets (with no overlapping classes). Two types of OOD datasets are considered: \textbf{Far-OOD} which corresponds to images very different from the training instances (e.g., CIFAR10 vs MNIST) and \textbf{Near-OOD} which corresponds to images close to the training instances (e.g., CIFAR10 vs CIFAR100). This last type of OOD detection is considered more challenging and is closely related to the OSR setting \citep{Yang2022openoodBenchmark}. Three datasets are considered: MNIST, CIFAR10 and CIFAR100.
\end{itemize}

A summary of the datasets used in each experiment is presented in Table \ref{table-openood-setup}. The AUROC score is used to evaluate the discrimination accuracy between test and OOD datasets. To compute the ``OOD scores'', a variety of algorithms are considered. They can be classified into two main categories:

\begin{itemize}
    \item \textbf{post-hoc Methods}, defined as methods that can be applied ``directly'' on a pretrained single network, independently of the training process. These methods are considered practical and model-agnostic \citep{Yang2022openoodBenchmark}. Among them, we can further distinguish the methods that do not require the training data: 
    MSP \citep{hendrycksbaseline2017MSP}, MLS \citep{hendrycks2022MLS}, ODIN \citep{liang2017ODIN}, EBO \citep{liu2020EBO}, GradNorm \citep{huang2021GradNorm}, ReAct \citep{sun2021react}, KLM \citep{hendrycks2022MLS} and TempScale \citep{guo2017TempScale} and the methods that uses the training set: 
    OpenMax \citep{bendale2016OpenMax}, MDS \citep{lee2018MDS}, Gram \citep{sastry2020Gram}, VIM \citep{wang2022vim}, KNN \citep{Sun2022KNNOOD}, DICE \citep{sun2022dice}. Notice that, except for MSP and MLS, all post-hoc methods at least require the use of a validation dataset to fine-tune their hyper-parameters.
    
    \item \textbf{Non post-hoc Methods}, including all methods which do not belong to the previous category, essentially because they require a specific training process (in terms of training loss or data augmentation for instance). This category of methods includes anomaly detection approaches: DeepSVDD \citep{ruff2018deepSVDD}, CutPaste \citep{li2021cutpaste}, DRAEM \citep{zavrtanik2021draem}; OOD detection methods with specific training process: ConfBranch \citep{devries2018ConfBranch}, G-ODIN \citep{hsu2020GODIN}, CSI \citep{tack2020csi}, ARPL \citep{chen2021ARPL}, MOS \citep{huang2021mos}, OpenGAN \citep{kong2021opengan}, VOS \citep{duvos2022VOS}, LogitNorm \citep{wei2022LogitNorm}; uncertainty-based approaches: MCdropout \citep{gal2016MCdropout}, Deep Ensemble \citep{Lakshminarayanan2017DeepEnsemble}; and data augmentation methods: MixUp \citep{thulasidasan2019mixup}, CutMix \citep{yun2019cutmix}, PixMix \citep{hendrycks2022pixmix}.
\end{itemize}

\begin{table}[ht]
\footnotesize
\centering
\begin{tabularx}{\linewidth}{>{\lratio{0.5}}X|X|>{\lratio{1.5}}X}
 Experiment &  In-Distribution Dataset & Out-Of-Distribution Datasets \\
\toprule
\multirow{4}{*}{\textbf{OSR}} & MNIST6 & MNIST4 \\
& CIFAR6 & CIFAR4 \\
& CIFAR50 & CIFAR50  \\
& TIN20 & TIN180 \\
\midrule
\multirow{3}{*}{\textbf{Near-OOD}} & MNIST & NOTMNIST,  FashionMNIST  \\
& CIFAR10 & CIFAR100, TIN200 \\
& CIFAR100 & CIFAR10, TIN200 \\
\midrule
\multirow{3}{*}{\textbf{Far-OOD}} & MNIST & CIFAR10, TIN200, Texture, Places-365  \\
& CIFAR10 & MNIST, SVHN, Texture, Places-365 \\
& CIFAR100 & MNIST, SVHN, Texture, Places-365 \\
\bottomrule
\end{tabularx}
\caption{OpenOOD Experiments Summary.}
\label{table-openood-setup}
\end{table}

According to \citet{Yang2022openoodBenchmark}, a fair comparison between methods should be done among each category, as non-post-hoc methods may benefit from their specific training process. Notice that this classification is not perfect. Post-hoc methods are considered model-agnostic, as they can generally be ``plugged'' into any pretrained network. However, most post-hoc methods generally require the end-layer of the network to produce logits. Post-hoc methods are considered practical because they generally require less computational time than the non post-hoc methods. This computational efficiency is mainly due to the training process economy. It should be mentioned, however, that inference time for some post-hoc methods may become important for large training datasets. For instance, KNN computes the distance between test data and all training data in the penultimate network layer. This may lead to important memory and computational burden if the training dataset is very large.

The MaxWEnt algorithm can be plugged directly into a pretrained neural network $h_{\overline{w}}$. It may not be totally considered as post-hoc, as it requires the additional training of the scale parameters $\phi$. However, this training may be done with a few epochs and also on a small extract of the training dataset. For our experiments, we trained MaxWEnt with the Adam optimizer \citep{Kingma2014Adam} with learning rate $5 \cdot 10^{-4}$ and $20$ epochs. We also consider an ensemble of five MaxWEnt networks. For inference, we use $P = 10$ predictions.

\subsubsection{Results}

The results are reported in Figure \ref{openood-results}, we compare AUROC scores between MaxWEnt (x1) and MaxWEnt (x5) (in red) to the previously mentioned methods (in blue). Note that we do not include OOD detection methods that require auxiliary OOD datasets during training for the comparison, as MaxWEnt does not use this kind of additional information. post-hoc methods are marked with a dagger $\dagger$. We group all experiments in the three main categories: OSR, NearOOD and FarOOD as described in Table \ref{table-openood-setup}. The reported AUROC scores are averaged over all experiments inside each category and over five different random seeds. We observe that MaxWEnt (x1) is ranked $3^{\text{rd}}$, $8^{\text{th}}$ and $2^{\text{nd}}$ for respectively the OSR, FarOOD and NearOOD experiments compared to all methods. When restricting the comparison to post-hoc methods, the MaxWEnt (x1) rankings become $1^{\text{st}}$, $3^{\text{rd}}$ and $1^{\text{st}}$ which demonstrates the effectiveness of the approach. It should be underlined that MaxWEnt (x1) is outperforming all other methods in the particular setting OSR and Near-OOD which are known to be the more challenging. For these two experiments, the MaxWEnt (x1) performances closely match those of Deep Ensemble, which requires the training of five neural networks and thus more computational resources. The ensemble of MaxWEnt networks provides an additional gain of around $2$ points of AUROC scores and is then ranked $1^{\text{st}}$, $3^{\text{rd}}$ and $1^{\text{st}}$ compared to all methods. However, this improvement requires the training of five networks, which increases the computational time.

\begin{figure}[!htbp]
    \centering
    \begin{minipage}{0.9\linewidth}
    \begin{minipage}{0.03\linewidth}
        \begin{tabular}{c}
            \rotatebox[origin=c]{90}{\small AUROC}
        \end{tabular}
    \end{minipage}
    \begin{minipage}{0.95\linewidth}
        \centering
        \includegraphics[width=\linewidth]{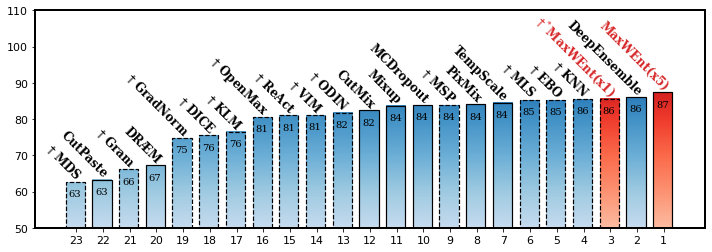}
    \end{minipage}
    \source{(a) OSR: MNIST6 + CIFAR6 + CIFAR50 + TIN20}
    \end{minipage} \\

    \begin{minipage}{0.9\linewidth}
    \begin{minipage}{0.03\linewidth}
        \begin{tabular}{c}
            \rotatebox[origin=c]{90}{\small AUROC}
        \end{tabular}
    \end{minipage}
    \begin{minipage}{0.95\linewidth}
        \centering
        \includegraphics[width=\linewidth]{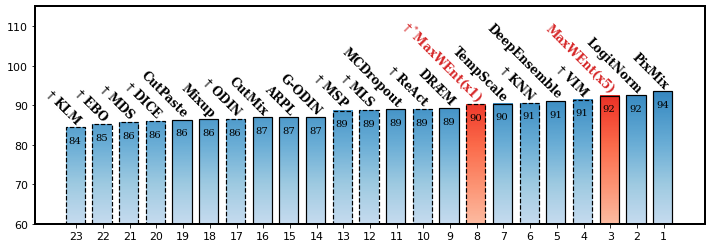}
    \end{minipage}
    \source{(b) FarOOD: MNIST + CIFAR10 + CIFAR100}
    \end{minipage} \\

    \begin{minipage}{0.9\linewidth}
    \begin{minipage}{0.03\linewidth}
        \begin{tabular}{c}
            \rotatebox[origin=c]{90}{\small AUROC}
        \end{tabular}
    \end{minipage}
    \begin{minipage}{0.95\linewidth}
        \centering
        \includegraphics[width=\linewidth]{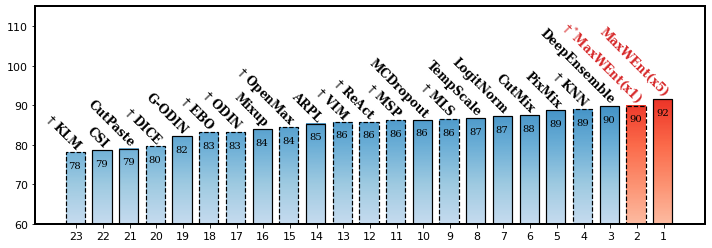} 
    \end{minipage}
    \source{(c) NearOOD: MNIST + CIFAR10 + CIFAR100}
    \end{minipage} \\
    
    \caption{OpenOOD benchmark ranking. Each method is ranked according to the average AUROC score computed for the three ``global'' experiments: OSR, Far-OOD, Near-OOD. Each experiment is performed on 3 or 4 different datasets (cf. Table \ref{table-openood-setup}). The top $23$ scores among the $32$ competitors are reported. post-hoc methods are marked with daggers, MaxWEnt(x1) can be considered as post-hoc as it applies on a pretrained network, although it requires additional training steps to learn the scaling parameters $\phi$.}
    \label{openood-results}
\end{figure}

\subsection{Implementation choices}
\label{impl-choice}

We present hereafter the implementation choices that we consider as ``good practices'' for MaxWEnt:


\subsubsection{Initialization}

In our proposed setup, the weight mean ${\mathbb{E}}_{q_{\phi}}[w] = \overline{w}$ is frozen during the MaxWEnt optimization and independent of the parameters $\phi$. The weight vector $\overline{w}$ is derived from a pretrained network $h_{\overline{w}}$ fitted on the training data. The $\phi$ parameters are initialized with a small constant value $C \ll 1$. Therefore, the weight distribution $q_{\phi}$ is initialized as a peaked distribution around $\overline{w}$, which already provides low empirical risk. Notice that the use of pretrained weights to initialize the mean of $q_{\phi}$ is similar to the common practice in Laplace approximation \citep{ritter2018scalableLaplace}, where the mean of the posterior distribution is set to the maximum a posteriori estimation (MAP). Moreover, in the case where a pretrained network is already available, the use of pretrained weights reduces the computational time. Note, finally, that we also consider a ``softplus'' activation of the $\phi$ parameters to smooth the increase of the weight entropy in earlier stages: $\phi = \log(1 + \exp(u))$.

\subsubsection{Entropy function}

In our implementation, we use a proxy for the entropy given by $H(\phi) \propto \sum_{k=1}^d |\phi_k|$. This choice smooths the increase of $\phi$ during the initial stages of training, offering improved optimization stability compared to the logarithmic function $ \log(\phi^2_k) $, which can over-penalize small components and cause excessive initial gradients (cf. Section \ref{entropy-func-discuss}).

\subsubsection{Trade-off parameter}

The MaxWEnt optimization (\ref{tradeoff-optim}) involves a trade-off between empirical risk minimization and entropy maximization, which is controlled by the trade-off parameter $\lambda$. A small $\lambda$ penalizes larger average risks, while a large $\lambda$ favors the weight distribution expansion. Obviously, the learner has to accept to penalize the empirical risk to offer room for the weight distribution to expand. In this perspective, we do not recommend selecting the trade-off parameter based on validation risk minimization. The $\lambda$ value should be selected large enough to speed up the increase of the weight entropy, while not too large to avoid optimization instabilities. In all our experiments, we choose to consider a fixed trade-off $\lambda = 10$ for simplicity.\footnote{In practice the entropy is scaled by the number of parameters such that $\lambda = 10/d$.}

An extended analysis of the impact of $\lambda$ on the training and the OOD detection performances is reported in Appendix \ref{app-hyp-impact}. This analysis reveals that OOD performances and entropy generally increase with higher $\lambda$ values, but when $\lambda$ is too large, the training becomes unstable and the validation loss quickly increases. As a result, a recommended approach is to use a stopping criterion to maintain a reasonable validation loss and select $\lambda$ based on the resulting weight entropy. A well-performing model is characterized by high entropy and reasonable validation loss. Both metrics can be measured without using OOD data, which enables a practical approach to select the $\lambda$ parameter.




\subsubsection{Stopping criterion}
\label{imp-stopping-criterion}

In standard training of neural networks, a sufficiently large number of epochs is generally performed until the full convergence of the training loss. Then the learner restores the weights of the network for the epoch which provides the best validation risk. Of course, we cannot consider such a technique for the MaxWEnt optimization, as increasing the weight entropy generally induces a small degradation of the validation risk. We then propose to save the network weights if the validation risk is below a threshold $\tau$ computed at the beginning of the optimization. We propose to estimate the performance threshold $\tau$ by the validation risk of the pretrained network $h_{\overline{w}}$ plus a statistical error:
\begin{equation*}
    \tau = \mathcal{L}_{\mathcal{S}_{\text{val}}}(\overline{w}) + \frac{2}{n_{\text{val}}} \sqrt{ \sum_{(x, y) \in \mathcal{S}_{\text{val}}} \left(\ell(h_{\overline{w}}(x), y) -\mathcal{L}_{\mathcal{S}_{\text{val}}}(\overline{w}) \right)^2 }.
\end{equation*}
The second term is proportional to the standard deviation of the errors on the validation dataset. 

The hyperparameter impact study presented in Appendix \ref{app-hyp-impact} reveals that, in some cases, a significant variance in OOD detection metrics is observed for late epochs, coinciding with an important increase in the validation loss. In these cases, the stopping criterion effectively selects an epoch with lower variance in OOD metrics and a reasonable validation loss. This indicates that the stopping criterion heuristic provides a relevant guideline to end the training, which does not add further hyperparameter tuning.



\subsubsection{Ensemble}

It should be underlined that the proposed parameterizations (\ref{weight-reparam}) and (\ref{SVD-param}) limit the range of the weight distribution around a neighborhood of $\overline{w}$. A straightforward improvement would be to apply Algorithm (\ref{alg-training}) on a set of weights $\overline{w}^{(j)}$ coming from a pretrained deep ensemble \citep{Lakshminarayanan2017DeepEnsemble}. Conceptually, this comes down to describing $q_{\phi}$ as a mixture with, for any $j \in [|1, m|]$, $\phi^{(j)} \in \mathbb{R}^d$, $z^{(j)} \sim \mathcal{Z}$ and $\pi \sim \mathcal{U}(\{1, ..., m \})$:
\begin{equation*}
q_{\phi} \sim \sum_{j=1}^m \mathbb{1}(\pi=j) \, \omega(\phi^{(j)}, z^{(j)}),
\end{equation*}
with $\omega(\phi^{(j)}, z^{(j)}) = \overline{w}^{(j)} + \phi^{(j)} \odot z^{(j)}$ or $\omega(\phi^{(j)}, z^{(j)}) = \overline{w}^{(j)} + V \left( \phi^{(j)} \odot z^{(j)} \right)$. In practice, we apply Algorithm (\ref{alg-training}) to each of the pretrained networks with the scaling parameterization $\omega(\phi^{(j)}, z^{(j)})$. Notice that, if there is no overlap between the mixture components, the ensemble parameterization necessarily results in a weight distribution of higher entropy for the same empirical risk level, and then leads to a more efficient parameterization than the single network setting. A guideline to choose the centers $\overline{w}^{(j)}$ is then to avoid overlapping, which can be achieved with centers distant from each other. Thus, combining MaxWEnt with techniques such as RDE \citep{Angelo2021RepulsiveDeepEnsemble} and AnchorNet \citep{pearce2018AnchorNetwork} may offer increased performances.


\section{Limitations and Perspectives}

In this work, we develop the MaxWEnt algorithm to improve OOD detection with stochastic neural networks. The main goal of MaxWEnt is to produce samples with larger weight diversity compared to standard Bayesian and ensemble methods. Our experiments show that MaxWEnt fulfills its promise, it increases the weight entropy and provides better OOD detection results. Moreover, we show that the more the weight entropy, the better the OOD detection (for the same level of average empirical error).

\begin{itemize}
    \item \textbf{Increasing the weight entropy}: The weight entropy increase is strongly conditioned by the weight parameterization. We show that the use of the SVD-parameterization is already an important improvement compared to the use of independent scaling parameters. However, more efficient parameterization may be obtained with other techniques such as normalizing flows \citep{louizos2017multiplicativeNormFlow} or weight subspaces \citep{izmailov2020subspaceInferenceBNN}. Nevertheless, the maximum entropy framework provides a general guideline for selecting the weight parameterization: an efficient stochastic model should enable a large increase of the weight entropy in low empirical risk regions of the weight space. 

\item \textbf{Penalized performances in-distribution}: We have seen that increasing the entropy penalizes the in-distribution performances. However, this negative result can be mitigated by the use of ``shrunk'' weight distribution obtained through weight clipping (cf. Sections \ref{synthetic-discuss} and \ref{uci-results-sec}). The user can use the MaxWEnt uncertainties to discriminate between ID and OOD data, and then use the prediction obtained with ``shrunk'' weight distribution for the ID data.

\item \textbf{SVD-parameterization for Convolutions}: For now, the SVD-parameterization is only developed for fully connected neural networks, but it may also be applied to convolutional layers. Convolutions apply the same kernel to multiple windows of one channel. To use the SVD-parameterization in this context, one idea is to concatenate all the windows on which the kernel is applied for all training data and then compute the SVD decomposition of the resulting dataset.

\item \textbf{General Bayesian and ensemble limitations}: The developed MaxWEnt approach improves upon Bayesian and ensemble methods in terms of weight diversity. However, it still inherits the other limitations of these approaches, which principally include the computational burden in training and inference. Future work will then consider the use of ``Laplace-like'' approximation to reduce the computational time of MaxWEnt (cf. Section \ref{discuss-svd}).

\end{itemize}

\section{Conclusion}

In this work, we tackle the over-confidence issue encountered with standard Bayesian and ensemble methods outside the training domain. Building on the maximum entropy principle, we show that penalizing the empirical average error with the weight entropy leads to larger hypothesis diversity and, then, to improved OOD detection. Theoretical analysis shows that the behavior of the developed MaxWEnt approach is related to the amplitude of the neuron activation on the training data. In MaxWEnt neural networks, weakly activated neurons play a more important role in the OOD detection in comparison to vanilla probabilistic networks. Motivated by this quest of entropy maximization and by the outcomes of our theoretical analysis, we propose the SVD parameterization to take advantage of correlations between weights with limited additional complexity. Numerical experiments show the benefit of the method and highlight the link between weight entropy and OOD detection performances. We show that the maximum entropy framework offers a guideline to rank two weight distributions with the same empirical risk, the one with the largest entropy should be preferred to improve OOD detection. Moreover, we advocate for the use of stochastic models that foster the increase of the weight entropy, as the SVD parameterization. We are convinced that this approach is a step forward in the safety of deep learning. Although many challenges have to be resolved such as the training and inference computational time.

\clearpage

\appendix

\section{Proofs}

\subsection{Proof of Proposition \ref{entropy-closeform}}
\label{proof-entropy-closeform}

Let's consider a matrix $A \in \mathbb{R}^{d \times d}$ and a vector $\phi \in \mathbb{R}^d$, $\phi > 0$, such that the weights $w$ are written:
\begin{equation*}
\begin{split}
    w = \overline{w} + A (\phi \odot z),
\end{split}
\end{equation*}
with $z \sim \mathcal{Z}$ following either a multivariate normal or a uniform distribution. We demonstrate Proposition (\ref{entropy-closeform}) for any orthogonal matrix $A$. Indeed, the weight parameterizations (\ref{weight-reparam}) and (\ref{SVD-param}) correspond respectively to the specific cases $A = \text{Id}_d$ and $A = V$ which are both orthogonal matrices. 


\subsubsection{Gaussian Case}

To demonstrate the result in the Gaussian case $z \sim \mathcal{N}(0, \text{Id}_d)$, we first derive the two following preliminary results:
\begin{itemize}
    \item $z \sim \mathcal{N}(0, \text{Id}_d) \implies A (z \odot \phi) \sim \mathcal{N}(0, A^T \text{diag}(\phi^2) A)$, with $\text{diag}(\phi^2)$ the diagonal matrix of diagonal values $\phi^2$ (cf. Lemma \ref{lemma1}).
    \item The entropy of a multivariate Gaussian $\mathcal{N}(0, \Sigma)$ is written $C + \frac{1}{2} \log(|\text{det}(\Sigma)|)$ with $C>
    0$ a constant (independent of $\Sigma$) and $\text{det}(\Sigma)$ the determinant of $\Sigma$ (cf. Lemma \ref{lemma2}).
\end{itemize}

\begin{lemma}
\label{lemma1}
For any $A \in \mathbb{R}^{d \times d}$ and any $\phi \in \mathbb{R}^d$, we have:
\begin{equation*}
z \sim \mathcal{N}(0, \textnormal{Id}_d) \implies A (z \odot \phi) \sim \mathcal{N}(0, A \, \textnormal{diag}(\phi^2) A^T).
\end{equation*}
\end{lemma}
\begin{proof}
    We first notice that linear combinations of Gaussian variables are Gaussians. Then, it appears that:
    \begin{equation*}
    \begin{split}
         \mathbb{E}[A (z \odot \phi)] & = A (\mathbb{E}[z] \odot \phi) = 0,
    \end{split}
    \end{equation*}
    and:
    \begin{equation*}
    \begin{split}
         \mathbb{V}[A (z \odot \phi)] & = \mathbb{E}[\left( A (z \odot \phi) \right) \left( A (z \odot \phi) \right)^T ]  \\
         & = \mathbb{E}[ A (z \odot \phi) (z \odot \phi)^T A^T ]  \\
         & = A \, \mathbb{E}[(z \odot \phi) (z \odot \phi)^T] A^T  \\
         & = A \, \mathbb{V}[z \odot \phi] A^T  \\
         & = A \, \textnormal{diag}(\phi^2) A^T .
    \end{split}
    \end{equation*}
    From which we conclude that $A (z \odot \phi) \sim \mathcal{N}(0, A \, \textnormal{diag}(\phi^2) A^T)$
    
\end{proof}

\begin{lemma}
\label{lemma2}
The entropy of a multivariate Gaussian $\mathcal{N}(0, \Sigma)$ is written $C + \frac{1}{2} \log(|\textnormal{det}(\Sigma)|)$ with $C>0$ a constant (independent of $\Sigma$) and $\textnormal{det}(\Sigma)$ the determinant of $\Sigma$.
\end{lemma}
\begin{proof}
Let's consider the multivariate Gaussian variable $Z \sim \mathcal{N}(0, \Sigma)$ with $\Sigma \in \mathbb{R}^{d \times d}$. We denote $p_Z(z)$ its probability density function such that, for any $z \in \mathbb{R}^d$:
\begin{equation*}
    \begin{split}
         p_Z(z) = \frac{1}{\sqrt{(2 \pi)^d |\textnormal{det}(\Sigma)|}} \exp \left(- \frac{1}{2} z^T \Sigma^{-1} z \right).
    \end{split}
    \end{equation*}
Then,
\begin{equation*}
\begin{split}
     -2 \log(p_Z(z)) = d \log(2 \pi) + \log(|\textnormal{det}(\Sigma)|) + z^T \Sigma^{-1} z.
\end{split}
\end{equation*}
We now consider the eigen-decomposition of $\Sigma^{-1}$, such that $\Sigma^{-1} = Q^T \textnormal{diag}(1/\lambda) Q$ with $Q \in \mathbb{R}^{d \times d}$ an orthogonal matrix and $\lambda$ the vector of eigenvalues of $\Sigma$. The following equality holds:
\begin{equation*}
\begin{split}
     z^T \Sigma^{-1} z = (Q z)^T \textnormal{diag}(1/\lambda) (Q z) = u^T \textnormal{diag}(1/\lambda) u = \sum_{k=1}^d \frac{u_k^2}{\lambda_k} \, .
\end{split}
\end{equation*}
Moreover, for any $z \sim \mathcal{N}(0, \Sigma)$, the variable $u = Q z$ follows the distribution $\mathcal{N}(0, Q \Sigma Q^T) = \mathcal{N}(0, \textnormal{diag}({\lambda}))$. We then deduce that:
\begin{equation*}
\begin{split}
     \mathbb{E}[z^T \Sigma^{-1} z] = \sum_{k=1}^d \frac{\mathbb{E}[u_k^2]}{\lambda_k^2} = \sum_{k=1}^d \frac{\lambda_k^2}{\lambda_k^2} = d .
\end{split}
\end{equation*}
Finally, we can derive the following formula for the entropy of $Z$:
\begin{equation*}
\begin{split}
     -\mathbb{E}[\log(p_Z(z))] = C + \frac{1}{2} \log(|\textnormal{det}(\Sigma)|),
\end{split}
\end{equation*}
with $C \in \mathbb{R}$ verifying: $C = \frac{d}{2} \log(2 \pi) + \frac{d}{2}$

\end{proof}

Let's now consider the variable $z \sim \mathcal{N}(0, \textnormal{Id}_d)$. According to Lemma (\ref{lemma1}), the variable $A (z \odot \phi)$ follows the distribution $\mathcal{N}(0, A \, \textnormal{diag}(\phi^2) A^T)$. Then, according to Lemma (\ref{lemma2}) and by invariance of the entropy by translation, the entropy of the distribution $q_{\phi}(w) \sim \overline{w} + A (z \odot \phi)$ is written:
\begin{equation*}
\begin{split}
    H(\phi) & = -\mathbb{E}[\log(q_{\phi}(w))] \\
    & = C + \frac{1}{2} \log(|\textnormal{det}(A \, \textnormal{diag}(\phi^2) A^T)|) \\
    & = C + \frac{1}{2} \log(|\textnormal{det}(A) \, \textnormal{det}(\textnormal{diag}(\phi^2))  \textnormal{det}(A^T)|),
\end{split}
\end{equation*}
with $C \in \mathbb{R}$ a constant. Then, as $A$ is an orthogonal matrix, we have $|\textnormal{det}(A)| = |\textnormal{det}(A^T)| = 1$ and:
\begin{equation*}
\begin{split}
    H(\phi) & = C + \frac{1}{2} \log( |\textnormal{det}(\textnormal{diag}(\phi^2))|) \\
    & = C + \frac{1}{2} \log(|\prod_{k=1}^d \phi_k^2|) \\
    & = C + \frac{1}{2} \sum_{k=1}^d 
 \log(\phi_k^2) .
\end{split}
\end{equation*}

\subsubsection{Uniform Case}

The probability density function $p_{Z}(z)$ of a uniform distribution defined over the parallelotope $\mathcal{P}$ described by the matrix $\Sigma \in \mathbb{R}^{d \times d}$ is written:
\begin{equation*}
    p_{Z}(z) = \begin{cases} 
      1 / \mathcal{V}(\mathcal{P}) & z \in \mathcal{P} \\
      0 & z \notin \mathcal{P}
   \end{cases},
\end{equation*}
with $\mathcal{P}$ the subset of $\mathbb{R}^d$ defined as $\mathcal{P} = \{ \Sigma x \, ; \; x \in [0, 1]^{d} \}$ and $\mathcal{V}(\mathcal{P})$ the volume of $\mathcal{P}$ which verifies $\mathcal{V}(\mathcal{P}) = |\textnormal{det}(\Sigma)|$.

Let's now consider the variable $Z$ of probability density function $p_{Z}(z)$, the entropy of $Z$ is then written:
\begin{equation*}
   \mathbb{E}[-\log(p_{Z}(z))] = \log(|\textnormal{det}(\Sigma)|).
\end{equation*}

We notice that, if $z \sim \mathcal{U}([-\sqrt{3}, \sqrt{3}]^d)$, then the variable $A (z \odot \phi) = A \, \textnormal{diag}(\phi) z$ is defined as the uniform distribution over the parallelotope $\mathcal{P} = \{ A \, \textnormal{diag}(\phi) x \, ; \; x \in [-\sqrt{3}, \sqrt{3}]^{d} \}$, which is well-defined as $\textnormal{det}(A \, \textnormal{diag}(\phi)) \neq 0$. As the volume of a subset is invariant by translation, we have $\mathcal{V}(\mathcal{P}) = \mathcal{V}(\tilde{\mathcal{P}})$ with $\tilde{\mathcal{P}}$ the parallelotope defined as $\tilde{\mathcal{P}} = \{ A \, \textnormal{diag}(\phi) x \, ; \; x \in [0, 2 \sqrt{3}]^{d} \} = \{ 2 \sqrt{3} A \, \textnormal{diag}(\phi) x \, ; \; x \in [0, 1]^{d} \}$. We then deduce that the entropy of $q_{\phi}(w) \sim \overline{w} + A (z \odot \phi)$ 
\begin{equation*}
\begin{split}
   H(\phi) & = \mathbb{E}[-\log(q_{\phi}(w))] \\ 
   & = \log(|\textnormal{det}(2 \sqrt{3} A \, \textnormal{diag}(\phi))|) \\
   & = \log(|\textnormal{det}(A)| \, |\textnormal{det}(2 \sqrt{3} \, \textnormal{diag}(\phi))|).
\end{split}
\end{equation*}
Finally, as $A$ is an orthogonal matrix, we have $|\textnormal{det}(A)| = 1$ and:
\begin{equation*}
\begin{split}
   H(\phi) & = \log(|\textnormal{det}(2 \sqrt{3} \, \textnormal{diag}(\phi))|) \\
   & = 2^{d-1} \sqrt{3}^d \sum_{k=1}^b \log(\phi_k^2).
\end{split}
\end{equation*}

\subsection{Proof of Proposition \ref{thm-lin-indep}}
\label{appendix-thm-lin-indep}


\begin{proof}
Let's consider $\phi \in \mathbb{R}^b$ and $z \sim \mathcal{Z}$. The training risk for the weight $w = \overline{w} + \phi \odot z$ can be written as follows:
\begin{equation*}
\begin{split}
     || X (\overline{w} + \phi \odot z) - y ||_2^2 & = || X (\phi \odot z) + X \overline{w} - y ||^2_2 \\
    & =  || X (\phi \odot z) ||^2_2 + \langle X (\phi \odot z), X \overline{w} - y \rangle + || X \overline{w} - y ||_2^2.
\end{split}
\end{equation*}
When averaging over $z \sim \mathcal{Z}$, considering that $\mathbb{E}[z] = 0$, we obtain:
\begin{equation}
\label{appendix-eq-lin-indep}
\begin{split}
    {\mathbb{E}}_{\mathcal{Z}}\left[ || X (\overline{w} + \phi \odot z) - y ||_2^2 \right] - || X \overline{w} - y ||_2^2 & = {\mathbb{E}}_{\mathcal{Z}}\left[ || X (\phi \odot z) ||^2_2 \right] \\
    & = \sum_{i=1}^n {\mathbb{E}}_{\mathcal{Z}}\left[ \left( \sum_{k = 1}^b X_{ik} \phi_k z_k \right)^2 \right] \\
    & = \sum_{i=1}^n \sum_{k=1}^b X_{ik}^2 \phi^2_k \\
    & = n \sum_{k=1}^b a_k^2 \phi^2_k
\end{split}
\end{equation}
The objective function of Problem (\ref{linear-optim}) can then be written, for any $\phi \in \mathbb{R}^b$:
\begin{equation*}
\begin{split}
    G(\phi) = \sum_{k=1}^b \left( a_k^2 \phi_k^2 - \lambda \log(\phi_k^2) \right) .
\end{split}
\end{equation*}
The objective function of Problem (\ref{linear-optim}) is convex and admits a solution. Moreover, the partial derivative of the objective with respect to $\phi_k^2$ is written:
\begin{equation*}
    \frac{\partial G(\phi)}{\partial \phi_k^2} = a_k^2 - \frac{\lambda}{\phi_k^2} \, .
\end{equation*}
As a consequence, the gradient of $G$ is null if and only if
\begin{equation*}
    \phi_k^2 = \frac{\lambda}{a_k^2},
\end{equation*}
which is well-defined when assuming $a_k^2 > 0$.

\end{proof}

\subsection{Proof of Proposition \ref{thm-lin-notindep}}
\label{appendix-thm-lin-notindep}

\begin{proof} 
Let's consider $\phi \in \mathbb{R}^b$, $V$ the matrix of eigenvectors of $\frac{1}{n} X^T X$ with $s^2$ the corresponding vector of eigenvalues and $z \sim \mathcal{Z}$. The average training risk for the weight $w = \overline{w} + V (\phi \odot z)$ can be written as follows:
\begin{equation*}
\begin{split}
    {\mathbb{E}}_{\mathcal{Z}}\left[\frac{1}{n} || X (\overline{w} + V (\phi \odot z)) - y ||_2^2 \right] = {\mathbb{E}}_{\mathcal{Z}}\left[\frac{1}{n} || X V (\phi \odot z) ||^2_2 \right] + \frac{1}{n} || X \overline{w} - y ||_2^2 .
\end{split}
\end{equation*}
We notice that:
\begin{equation*}
\begin{split}
     \frac{1}{n} || X V (\phi \odot z) ||^2_2 & = \frac{1}{n} || X V \, \textnormal{diag}(\phi) z ||^2_2 \\
     & = z^T \textnormal{diag}(\phi)^T V^T \left( \frac{1}{n} X^T X \right) V \textnormal{diag}(\phi) z \\
     & = z^T \textnormal{diag}(\phi)^T \textnormal{diag}(s^2) \textnormal{diag}(\phi) z \\
     & = z^T \textnormal{diag}(s^2 \phi^2) z \\
     & = \sum_{k=1}^b s_k^2 \phi_k^2 z_k^2 .
\end{split}
\end{equation*}
Then,
\begin{equation}
\label{appendix-eq-lin-dep}
\begin{split}
    {\mathbb{E}}_{\mathcal{Z}}\left[\frac{1}{n} || X (\overline{w} + V (\phi \odot z)) - y ||_2^2 \right] = \sum_{k=1}^b s_k^2 \phi_k^2 + \frac{1}{n} || X \overline{w} - y ||_2^2 .
\end{split}
\end{equation}
The continuation of the proof is similar to the proof in Appendix (\ref{appendix-thm-lin-indep}) with $s_k^2$ instead of $a_k^2$.

\end{proof}

\subsection{Proof of Proposition \ref{thm-wf}}
\label{appendix-thm-compar}

\begin{proof}
Let $q^{(1)}_{\phi^*}$, $q^{(2)}_{\phi^*}$ be the respective optimal weight distributions for the scaling and the SVD parameterization. Then,
\begin{gather*}
    q^{(1)}_{\phi^*} \sim \overline{w} + \frac{\lambda}{a} \odot z \\
    q^{(2)}_{\phi^*} \sim \overline{w} + V ( \frac{\lambda}{s} \odot z), 
\end{gather*}
with $z \sim \mathcal{Z}$. Considering Equations (\ref{appendix-eq-lin-indep}) and (\ref{appendix-eq-lin-dep}), both average empirical losses are written:
\begin{gather*}
    {\mathbb{E}}_{q^{(1)}_{\phi^*}}\left[ \mathcal{L}_{\mathcal{S}}(w) \right] = \sum_{k=1}^b \frac{\lambda a_k^2}{a_k^2} + \frac{1}{n} || X \overline{w} - y ||_2^2 \\
    {\mathbb{E}}_{q^{(2)}_{\phi^*}}\left[ \mathcal{L}_{\mathcal{S}}(w) \right] = \sum_{k=1}^b \frac{\lambda s_k^2}{s_k^2} + \frac{1}{n} || X \overline{w} - y ||_2^2 .
\end{gather*}
Then,
\begin{equation*}
    {\mathbb{E}}_{q^{(1)}_{\phi^*}}\left[ \mathcal{L}_{\mathcal{S}}(w) \right] = {\mathbb{E}}_{q^{(2)}_{\phi^*}}\left[ \mathcal{L}_{\mathcal{S}}(w) \right] = \lambda \, b + \frac{1}{n} || X \overline{w} - y ||_2^2 .
\end{equation*}
Moreover, both entropy can be written as:
\begin{gather*}
    {\mathbb{E}}_{q^{(1)}_{\phi^*}}\left[ -\log(q^{(1)}_{\phi^*}) \right] = 
    - \sum_{k=1}^b \log(a_k^2) + b \log(\lambda)
    \\ {\mathbb{E}}_{q^{(2)}_{\phi^*}}\left[ - \log(q^{(2)}_{\phi^*}) \right] = 
    - \sum_{k=1}^b \log(s_k^2) + b \log(\lambda) .
\end{gather*}
Let's denote $M = \frac{1}{n} X^T X$, by definition, we have the following equalities:
\begin{gather}
    \label{appendix-diago-eq}
    M = V^T \, \textnormal{diag} (s^2) V \\
    \label{appendix-diag-elem-eq}
    M_{ii} = a_i^2 \; \forall i \in [|1, b|] .
\end{gather}
Equation (\ref{appendix-diago-eq}) implies that $M = U U^T \; \textnormal{with} \; U = V^T \, \textnormal{diag} (s) V$. For any $i \in [|1, b|]$, we denote $u_i \in \mathbb{R}^b$ the $i^{th}$ row vector of the matrix $U$ and $||u_i||_2 = \sqrt{\sum_{j=1}^b U_{ij}^2}$ its corresponding Euclidean norm.

Applying the Hadamard inequality to the matrix $U$, we obtain that:
\begin{equation*}
    \textnormal{det}(U) \leq \prod_{i=1}^b ||u_i||_2 .
\end{equation*}
Then, the formula $U = V^T \, \textnormal{diag} (s) V$ implies that $\textnormal{det}(U) = \prod_{i=1}^b s_i$ and the equality $M = U U^T$ implies that $M_{ii} = \sum_{j=1}^b U_{ij}^2 = ||u_i||_2^2$. Considering Equation (\ref{appendix-diag-elem-eq}), we then deduce that:
\begin{equation*}
    \prod_{i=1}^b s^2_i \leq \prod_{i=1}^b a_i^2 .
\end{equation*}
From which we conclude that:
\begin{equation*}
\begin{split}
    -\log \left( \prod_{i=1}^b s^2_i \right)  \geq -\log \left( \prod_{i=1}^b a_i^2 \right) & \implies
    - \sum_{i=1}^b \log(s_i^2)  \geq - \sum_{i=1}^b \log(a_i^2) \\
    & \implies
    {\mathbb{E}}_{q^{(2)}_{\phi^*}}\left[ - \log(q^{(2)}_{\phi^*}) \right]  \geq {\mathbb{E}}_{q^{(2)}_{\phi^*}}\left[ - \log(q^{(1)}_{\phi^*}) \right] .
\end{split}
\end{equation*}

\end{proof}

\subsection{Proof of Proposition \ref{thm-multi-closeform}}
\label{proof-thm-multi-closeform}

The proof consists in first rewriting the optimization problem (\ref{multi-layer}) as a maximum entropy problem with a constraint over the average empirical risk. Then, we show that $\phi^*$ is solution of the optimization problem (OP) augmented with additional equality constraints in the hidden layers. We then remove the constraint over the average empirical risk and show that the solution $\phi^{\dagger}$ of the resulting OP provides a distribution with higher entropy than $\phi^*$. By splitting the OP in sub-optimization problems by hidden layer, we show that $\phi^{\dagger}$ verifies Equation (\ref{thm-multi-cond-phi-eq}). Then, using recursively Assumption (\ref{activ-func}) on the activation function, we show that, for any layer, the first and second moments of the neuron activation are the same for both distribution $q_{\phi^{\dagger}}$ and $q_{\phi^*}$. We then prove the equality of empirical risk for $q_{\phi^{\dagger}}$ and $q_{\phi^*}$, leading to show that $\phi^{\dagger}$ is solution of Problem (\ref{multi-layer}), from which we conclude that $\phi^{\dagger} = \phi^*$, as the solution is unique.

\begin{proof}
    Let's consider $\overline{w} \in \mathbb{R}^d$ and, for any $\phi \in \mathbb{R}^d$, the distribution $q_{\phi} \sim \overline{w} + \phi \odot z$ with $z \sim \mathcal{Z}$ such that $\mathcal{Z} \sim \mathcal{U}([-\sqrt{3}, \sqrt{3}]^d)$ or $\mathcal{Z} \sim \mathcal{N}(0, \textnormal{Id}_d)$. The optimization problem (\ref{multi-layer}) is written:
\begin{equation}
\label{app-multi-optim1}
    \min_{\phi \in \mathbb{R}^d} {\mathbb{E}}_{q_{\phi}}\left[\mathcal{L}_{\mathcal{S}}(w) \right] - \lambda \sum_{k=1}^d \log(\phi_k^2) .
\end{equation}
It is assumed that the above optimization problem has a unique solution, denoted $\phi^* \in \mathbb{R}^d$. Then, there exists $\tau \in \mathbb{R}_+$ such that $\phi^*$ verifies the following optimization problem:
\begin{equation}
\label{app-multi-optim2}
\begin{gathered}
    \max_{\phi \in \mathbb{R}^d} \quad \sum_{k=1}^d \log(\phi_k^2) \\
    \textnormal{subject to} \quad
     {\mathbb{E}}_{q_{\phi}}\left[\mathcal{L}_{\mathcal{S}}(w) \right] \leq \tau .
\end{gathered}
\end{equation}
Indeed, for $\tau = {\mathbb{E}}_{q_{\phi^*}}\left[\mathcal{L}_{\mathcal{S}}(w) \right]$, if we denote $\phi^{**} \in \mathbb{R}^d$ the solution of problem (\ref{app-multi-optim2}), then \linebreak $\sum_{k=1}^d \log({\phi^{**}_k}^2) \geq \sum_{k=1}^d \log({\phi^*_k}^2)$ and $ {\mathbb{E}}_{q_{\phi^{**}}}\left[\mathcal{L}_{\mathcal{S}}(w) \right] \leq \tau$ which implies that:
\begin{equation*}
    {\mathbb{E}}_{q_{\phi^{**}}}\left[\mathcal{L}_{\mathcal{S}}(w) \right] - \lambda \sum_{k=1}^d \log({\phi^{**}_k}^2) \leq {\mathbb{E}}_{q_{\phi^*}}\left[\mathcal{L}_{\mathcal{S}}(w) \right] - \lambda \sum_{k=1}^d \log({\phi^*_k}^2)  .
\end{equation*}
From which we deduce that $\phi^{**} = \phi^*$, as the solution of Problem (\ref{app-multi-optim1}) is assumed unique. Moreover, $\phi^*$ is the unique solution of Problem (\ref{app-multi-optim2}).

\smallskip

\noindent For each layer, we define the amplitude of the input neuron activation on average over the training data:
\begin{gather*}
\label{app-multi-ampl}
a_{(l, k)}^2 = \frac{1}{n} \sum_{i=1}^n {\mathbb{E}}_{q_{\phi^*}}[\psi_{(l, k)}(x_i)^2]   \quad \forall \, l \in [|0, L|]; k \in [|1, b|] .
\end{gather*}
We also define the quantities $\sigma_{(l, j)}^2$, related to the variance of the output neurons, before activation, on average over the training data:
\begin{gather*}
\label{app-multi-sigma}
\sigma_{(l, j)}^2 = \frac{1}{n} \sum_{i=1}^n {\mathbb{V}}_{q_{\phi^*}}\left[\psi_{(l)}(x_i)^T \left(w_{(l, j)} - \overline{w}_{(l, j)} \right) \right] \quad \forall \, l \in [|0, L|]; j \in [|1, b_l|] ,
\end{gather*}
with $b_l = 1$ if $l = L$ and $b_l=b$ otherwise.

\noindent Let's now take $l \in [|0, L|]$ and $j \in [|1, b_l|]$, considering the independence between $\psi_{(l)}$ and $z_{(l, j)}$, we have:
\begin{equation*}
\label{app-eq-cov1}
\begin{split}
    n \sigma_{(l, j)}^2 & = \sum_{i=1}^n {\mathbb{V}}_{q_{\phi^*}}[\psi_{(l)}(x_i) (\phi_{(l, j)}^* \odot z_{(l, j)})]\\
    & = \sum_{i=1}^n {\mathbb{V}}_{q_{\phi^*}} \left[ \sum_{k=1}^b \psi_{(l, k)}(x_i) \phi_{(l, j, k)}^*  z_{(l, j, k)} \right] \\
    & = \sum_{i=1}^n \sum_{u=1}^b \sum_{v=1}^b \phi_{(l, j, u)}^* \phi_{(l, j, v)}^* \textnormal{Cov}\left(\psi_{(l, u)}(x_i)  z_{(l, j, u)}, \psi_{(l, v)}(x_i) z_{(l, j, v)} \right) \\
    & = \sum_{i=1}^n \sum_{u=1}^b \sum_{v=1}^b \phi_{(l, j, u)}^* \phi_{(l, j, v)}^* \mathbb{E}_{q_{\phi^*}}\left[  \psi_{(l, u)}(x_i) \psi_{(l, v)}(x_i) \right] \mathbb{E}_{q_{\phi^*}}\left[ z_{(l, j, u)} z_{(l, j, v)} \right] .
\end{split}
\end{equation*}
For $u \neq v$, $z_{(l, j, u)} \independent z_{(l, j, v)}$ and $\mathbb{E}_{q_{\phi^*}}\left[ z_{(l, j, u)} z_{(l, j, v)} \right] = \mathbb{E}_{q_{\phi^*}}\left[ z_{(l, j, u)} \right] \mathbb{E}_{q_{\phi^*}}\left[ z_{(l, j, v)} \right] = 0$, then:
\begin{equation*}
\begin{split}
    \sigma_{(l, j)}^2 & = \frac{1}{n} \sum_{i=1}^n \sum_{k=1}^b {\mathbb{E}}_{q_{\phi^*}}[\psi_{(l, k)}(x_i)^2] {\phi_{(l, j , k)}^*}^2 \\
    & = \sum_{k=1}^b a_{(l, k)}^2 {\phi_{(l, j , k)}^*}^2 .
\end{split}
\end{equation*}

\noindent The optimization problem (\ref{app-multi-optim2}) is then equivalent to:
\begin{equation}
\label{app-multi-optim3}
\begin{gathered}
    \max_{\phi \in \mathbb{R}^d} \quad \sum_{k=1}^d \log(\phi_k^2) \\
    \textnormal{subject to:} \begin{cases} 
      {\mathbb{E}}_{q_{\phi}}\left[\mathcal{L}_{\mathcal{S}}(w) \right] \leq \tau \\
      \sum_{k=1}^b a_{(l, k)}^2 {\phi_{(l, j , k)}}^2 = \sigma_{(l, j)}^2 \quad \forall \, l \in [|0, L|]; j \in [|1, b_l|] .
   \end{cases}
\end{gathered}
\end{equation}
Indeed, as problem (\ref{app-multi-optim3}) includes more constraints than problem (\ref{app-multi-optim2}), its solution necessarily provides a distribution of lower or equal entropy than $q_{\phi^*}$. However, as the additional constraints are verified by $\phi^*$, $\phi^*$ is the unique solution of problem (\ref{app-multi-optim3}).

We now remove the constraint over the average empirical risk and consider the following alternative optimization problem:
\begin{equation}
\label{app-multi-optim4}
\begin{gathered}
    \max_{\phi \in \mathbb{R}^d} \quad \sum_{k=1}^d \log(\phi_k^2) \\
    \textnormal{subject to:}  \quad
      \sum_{k=1}^b a_{(l, k)}^2 {\phi_{(l, j , k)}}^2 = \sigma_{(l, j)}^2 \quad \forall \, l \in [|0, L|]; j \in [|1, b_l|] .
\end{gathered}
\end{equation}
Considering a similar argument as before, the solution $\phi^{\dagger}$ of problem (\ref{app-multi-optim4}) necessarily provides a distribution of larger or equal entropy than $\phi^*$, i.e.,
\begin{equation}
\label{app-multi-entropy}
    \sum_{k=1}^d \log({\phi^*_k}^2) \leq \sum_{k=1}^d \log({\phi^{\dagger}_k}^2) .
\end{equation}
Moreover, the optimization problem (\ref{app-multi-optim4}) can be decomposed in multiple sub-problems such that:
\begin{equation*}
    \phi^{\dagger} = \overset{L}{\underset{l=0}{\bigotimes}} \overset{b_l}{\underset{j=1}{\bigotimes}} \, \phi^{\dagger}_{(l, j)} , 
\end{equation*}
with $\phi^{\dagger}_{(l, j)} \in \mathbb{R}^b$ for any $l \in [|0, L|]$, $j \in [|1, b_l|]$. The operator $\bigotimes$ is the concatenation operator. Each vector $\phi^{\dagger}_{(l, j)}$ is a solution of the following optimization sub-problem:
\begin{equation*}
\label{app-multi-optim5}
\begin{gathered}
    \max_{\phi_{(l, j)} \in \mathbb{R}^b} \; \sum_{k=1}^b \log(\phi_{(l, j, k)}^2) \\
    \textnormal{subject to:} \quad
     \sum_{k=1}^b a_{(l, k)}^2 {\phi_{(l, j , k)}}^2 = \sigma_{(l, j)}^2 .
\end{gathered}
\end{equation*}
Then, by writing the Karush–Kuhn–Tucker conditions of the above optimization problem we get the following expression for the solution:
\begin{equation*}
    {\phi^{\dagger}_{(l, j, k)}}^2 = \frac{\sigma^2_{(l, j)}}{b \, a^2_{(l, k)}} \quad \forall \, k \in [|1, b|] .
\end{equation*}
Thus, $\phi^{\dagger}$ verifies Equation (\ref{thm-multi-cond-phi-eq}).

We now need to show that $\phi^{\dagger}$ provides the same empirical risk than $\phi^*$. For this purpose, we consider $l \in [|0, L-1|]$ and assume that the first and the second moments of the neuron activation in layer $l$ are the same for $\phi^*$ and $\phi^{\dagger}$, we will then show that this property is true in layer $l+1$. Let's then assume that:
\begin{gather}
\label{app-assum-mom1}
\sum_{i=1}^n {\mathbb{E}}_{q_{\phi^*}}\left[ \psi_{(l, j)}(x_i) \right] = \sum_{i=1}^n {\mathbb{E}}_{q_{\phi^{\dagger}}}\left[ \psi_{(l, j)}(x_i) \right] \quad \forall \, j \in [|1, b|] \\
\label{app-assum-mom2}
\sum_{i=1}^n {\mathbb{E}}_{q_{\phi^*}}\left[ \psi_{(l)}(x_i) \psi_{(l)}(x_i)^T \right] = \sum_{i=1}^n {\mathbb{E}}_{q_{\phi^{\dagger}}}\left[ \psi_{(l)}(x_i) \psi_{(l)}(x_i)^T \right] .
\end{gather}
Let's define $U_i = (U_{i1}, ..., U_{ib})$ with $U_{ij} = \psi_{(l)}(x_i)^T w_{(l, j)} \; \forall i \in [|1, n|], \; \forall j \in [|1, b|]$. Considering Equation (\ref{app-assum-mom1}), for any $j \in [|1, b|]$, we have:
\begin{equation}
\label{app-mean-eq}
\begin{split}
    \sum_{i=1}^n {\mathbb{E}}_{q_{\phi^{\dagger}}} \left[ U_{ij} \right]
    = \sum_{i=1}^n {\mathbb{E}}_{q_{\phi^{\dagger}}} \left[ \psi_{(l)}(x_i) \right]^T \overline{w}_{(l, j)}
     = \sum_{i=1}^n {\mathbb{E}}_{q_{\phi^*}} \left[ \psi_{(l)}(x_i) \right]^T \overline{w}_{(l, j)}
    = \sum_{i=1}^n {\mathbb{E}}_{q_{\phi^*}} \left[ U_{ij} \right] .
\end{split}
\end{equation}
Moreover, for any $k, j \in [|1, b|]$ such that $k \neq j$, we have:
\begin{equation}
\label{app-ij-cov-equality}
\begin{split}
     \sum_{i=1}^n {\mathbb{E}}_{q_{\phi^{\dagger}}} \left[ U_i U_i^T \right]_{k j} & =  \sum_{i=1}^n {\mathbb{E}}_{q_{\phi^{\dagger}}} \left[ U_{ik} U_{ij} \right] \\
    & =  \sum_{i=1}^n {\mathbb{E}}_{q_{\phi^{\dagger}}} \left[ \left( \psi_{(l)}(x_i)^T w_{(l, k)} \right) \left(  \psi_{(l)}(x_i)^T w_{(l, j)} \right) \right] \\
    & =  \sum_{i=1}^n \sum_{u = 1}^b \sum_{v = 1}^b {\mathbb{E}}_{q_{\phi^{\dagger}}} \left[ \psi_{(l, u)}(x_i) \psi_{(l, v)}(x_i) \right] {\mathbb{E}}_{q_{\phi^{\dagger}}} \left[w_{(l, k, u)}  w_{(l, j, v)} \right] \\
    & =  \sum_{u = 1}^b \sum_{v = 1}^b \left( \sum_{i=1}^n {\mathbb{E}}_{q_{\phi^{\dagger}}} \left[ \psi_{(l, u)}(x_i) \psi_{(l, v)}(x_i) \right] \right) \overline{w}_{(l, k, u)}  \overline{w}_{(l, j, v)} \\
    & =  \sum_{i=1}^n \overline{w}_{(l, k)}^T \left( \sum_{i=1}^n {\mathbb{E}}_{q_{\phi^{\dagger}}} \left[ \psi_{(l)}(x_i) \psi_{(l)}(x_i)^T \right] \right) \overline{w}_{(l, j)} \\
    & =  \sum_{i=1}^n \overline{w}_{(l, k)}^T \left( \sum_{i=1}^n {\mathbb{E}}_{q_{\phi^*}} \left[ \psi_{(l)}(x_i) \psi_{(l)}(x_i)^T \right] \right) \overline{w}_{(l, j)} \quad (\textnormal{considering Equation (\ref{app-assum-mom2})}) \\
    & =  \sum_{i=1}^n {\mathbb{E}}_{q_{\phi^*}} \left[ U_i U_i^T \right]_{k j} .
\end{split}
\end{equation}
Then, for any $j \in [|1, b|]$, we have:
\begin{equation*}
\begin{split}
     \sum_{i=1}^n {\mathbb{E}}_{q_{\phi^{\dagger}}} \left[ U_i U_i^T \right]_{j j} & = \sum_{i=1}^n {\mathbb{E}}_{q_{\phi^{\dagger}}} \left[ \left( \psi_{(l)}(x_i)^T w_{(l, j)} \right)^2 \right] \\
    & =  \sum_{i=1}^n \left( {\mathbb{V}}_{q_{\phi^{\dagger}}} \left[\psi_{(l)}(x_i)^T (w_{(l, j)} -\overline{w}_{(l, j)}) \right] + {\mathbb{E}}_{q_{\phi^{\dagger}}} \left[ \left( \psi_{(l)}(x_i)^T \overline{w}_{(l, j)} \right)^2 \right] \right) \\
    & =  \sum_{i=1}^n \left( \sum_{k=1}^b {\mathbb{E}}_{q_{\phi^{\dagger}}}[\psi_{(l, k)}(x_i)^2] {\phi_{(l, j , k)}^{\dagger}}^2 + \overline{w}_{(l, j)}^T {\mathbb{E}}_{q_{\phi^{\dagger}}} \left[ \psi_{(l)}(x_i) \psi_{(l)}(x_i)^T  \right] \overline{w}_{(l, j)} \right) \\
    & =  \sum_{i=1}^n \left( \sum_{k=1}^b {\mathbb{E}}_{q_{\phi^*}}[\psi_{(l, k)}(x_i)^2] {\phi_{(l, j , k)}^{\dagger}}^2 + \overline{w}_{(l, j)}^T {\mathbb{E}}_{q_{\phi^*}} \left[ \psi_{(l)}(x_i) \psi_{(l)}(x_i)^T  \right] \overline{w}_{(l, j)} \right) .
\end{split}
\end{equation*}
Where the last equality is deducted from Equation (\ref{app-assum-mom2}). Moreover, the first term can be developed as follows:
\begin{equation*}
\begin{split}
    \sum_{i=1}^n \sum_{k=1}^b {\mathbb{E}}_{q_{\phi^*}}[\psi_{(l, k)}(x_i)^2] {\phi_{(l, j , k)}^{\dagger}}^2 & = \sum_{k=1}^b  n \, a^2_{(l, k)} {\phi_{(l, j , k)}^{\dagger}}^2 \\
    & = \sum_{k=1}^b  n \, a^2_{(l, k)} \frac{\sigma_{(l, j)}^2}{b \, a^2_{(l, k)}} \quad \textnormal{by definition of $\phi^{\dagger}$} \\
    & =  n \, \sigma_{(l, j)}^2 \\
    & = \sum_{i=1}^n {\mathbb{V}}_{q_{\phi^*}}\left[\psi_{(l)}(x_i)^T \left(w_{(l, j)} - \overline{w}_{(l, j)} \right) \right] .
\end{split}
\end{equation*}
We then deduce that:
\begin{equation}
\label{app-ii-cov-equality}
\sum_{i=1}^n {\mathbb{E}}_{q_{\phi^{\dagger}}} \left[ U_i U_i^T \right]_{j j} = \sum_{i=1}^n {\mathbb{E}}_{q_{\phi^*}} \left[ U_i U_i^T \right]_{j j}  .
\end{equation}
Equations (\ref{app-ij-cov-equality}) and (\ref{app-ii-cov-equality}) implies that $\sum_{i=1}^n {\mathbb{E}}_{q_{\phi^{\dagger}}} \left[ U_i U_i^T \right] = \sum_{i=1}^n {\mathbb{E}}_{q_{\phi^*}} \left[ U_i U_i^T \right]$. Considering this last equality, Equation (\ref{app-mean-eq}) and Assumption (\ref{activ-func}), we then conclude that:
\begin{gather}
\label{app-assum-1-equiv}
    \sum_{i=1}^n  {\mathbb{E}}_{q_{\phi^{\dagger}}}\left[ \zeta \left( U_i \right) \right] = \sum_{i=1}^n  {\mathbb{E}}_{q_{\phi^*}}\left[  \zeta \left( U_i \right) \right] \\
\label{app-assum-2-equiv}
    \sum_{i=1}^n  {\mathbb{E}}_{q_{\phi^{\dagger}}}\left[ \zeta \left( U_i \right) \zeta \left( U_i \right)^T \right] = \sum_{i=1}^n  {\mathbb{E}}_{q_{\phi^*}}\left[  \zeta \left( U_i \right) \zeta \left( U_i \right)^T \right] .
\end{gather}
Where,
\begin{equation*}
\zeta \left( U_i \right) = \left(\zeta \left( \psi_{(l)}(x_i) w_{(l, 1)}  \right), ..., \zeta \left( \psi_{(l)}^T(x_i) w_{(l, p)} \right) \right) = \psi_{(l+1)}(x_i) .
\end{equation*}
Then Equations (\ref{app-assum-1-equiv}) and (\ref{app-assum-2-equiv}) are equivalent to the moments' equality in Equations (\ref{app-assum-mom1}) and (\ref{app-assum-mom2}) applied to layer $l+1$. As these equations are true for $l=0$, then, by recurrence, we have Equations (\ref{app-assum-mom1}) and (\ref{app-assum-mom2}) for $l = L+1$, then:
\begin{gather*}
    \sum_{i=1}^n  {\mathbb{E}}_{q_{\phi^{\dagger}}}\left[ h(x_i) \right] = \sum_{i=1}^n  {\mathbb{E}}_{q_{\phi^*}}\left[  h(x_i) \right] \quad \textnormal{and} \quad \sum_{i=1}^n  {\mathbb{E}}_{q_{\phi^{\dagger}}}\left[ h(x_i)^2 \right] = \sum_{i=1}^n  {\mathbb{E}}_{q_{\phi^*}}\left[  h(x_i)^2 \right] .
\end{gather*}
Moreover, by developing the empirical risk, we have:
\begin{gather*}
    \mathcal{L}_{\mathcal{S}}(w) = \sum_{i=1}^n  \left( h(x_i) - y_i \right)^2 = \sum_{i=1}^n  \left( h(x_i)^2 - 2 h(x_i) y_i + y_i^2 \right) .
\end{gather*}
From which we deduce that:
\begin{gather*}
    {\mathbb{E}}_{q_{\phi^{\dagger}}}\left[ \mathcal{L}_{\mathcal{S}}(w) \right] = {\mathbb{E}}_{q_{\phi^*}}\left[ \mathcal{L}_{\mathcal{S}}(w) \right] .
\end{gather*}
Then, considering Equation (\ref{app-multi-entropy}) and the uniqueness of the solution of Problem (\ref{app-multi-optim2}), we conclude that $\phi^{\dagger} = \phi^*$.

\end{proof}

\section{Hyperparameters Impact Analysis}
\label{app-hyp-impact}


\begin{figure}[!htb]
\center
\includegraphics[width=0.85\linewidth]{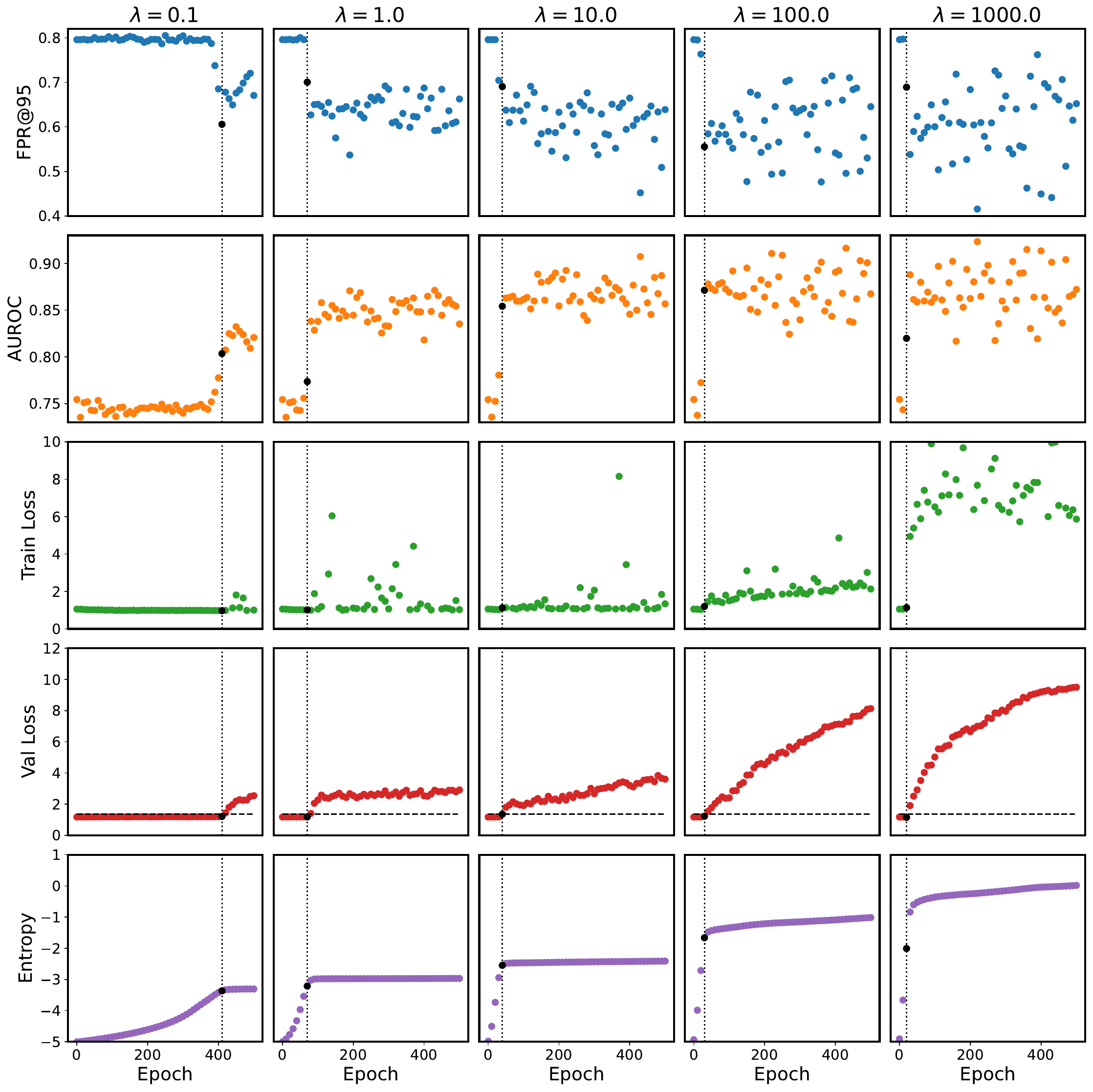}
\caption{Hyperparameter analysis of MaxWEntSVD for the Weather-Shift experiment. Evolution of the five metrics: FPR@95, AUROC, training loss, validation loss, and weight entropy across epochs for different values of $\lambda$ (note that a scaling factor of $1/d$ is further applied to $\lambda$). The horizontal dashed lines in the validation loss plots represent the threshold $\tau$ and the dotted vertical lines indicate the stopping epoch (after which the validation loss exceeds $\tau$). We observe that higher $\lambda$ values lead to better FPR and AUROC, along with higher entropy. At a certain epoch, entropy increases rapidly, coinciding with improved OOD detection performance but also an increase in validation loss. For high $\lambda$ values and late epochs, the variance in OOD detection metrics between epochs becomes large. The stopping criterion ensures that training stops when AUROC is high but with reduced variance. For $\lambda = 1000$, the training loss is large and fluctuating, with high variance in OOD detection metrics, suggesting that the training loss stability should be taken into account when selecting the $\lambda$ value.}
\label{app-citycam-weather-svd}
\end{figure}

\clearpage

\begin{figure}[!htb]
\center
\includegraphics[width=\linewidth]{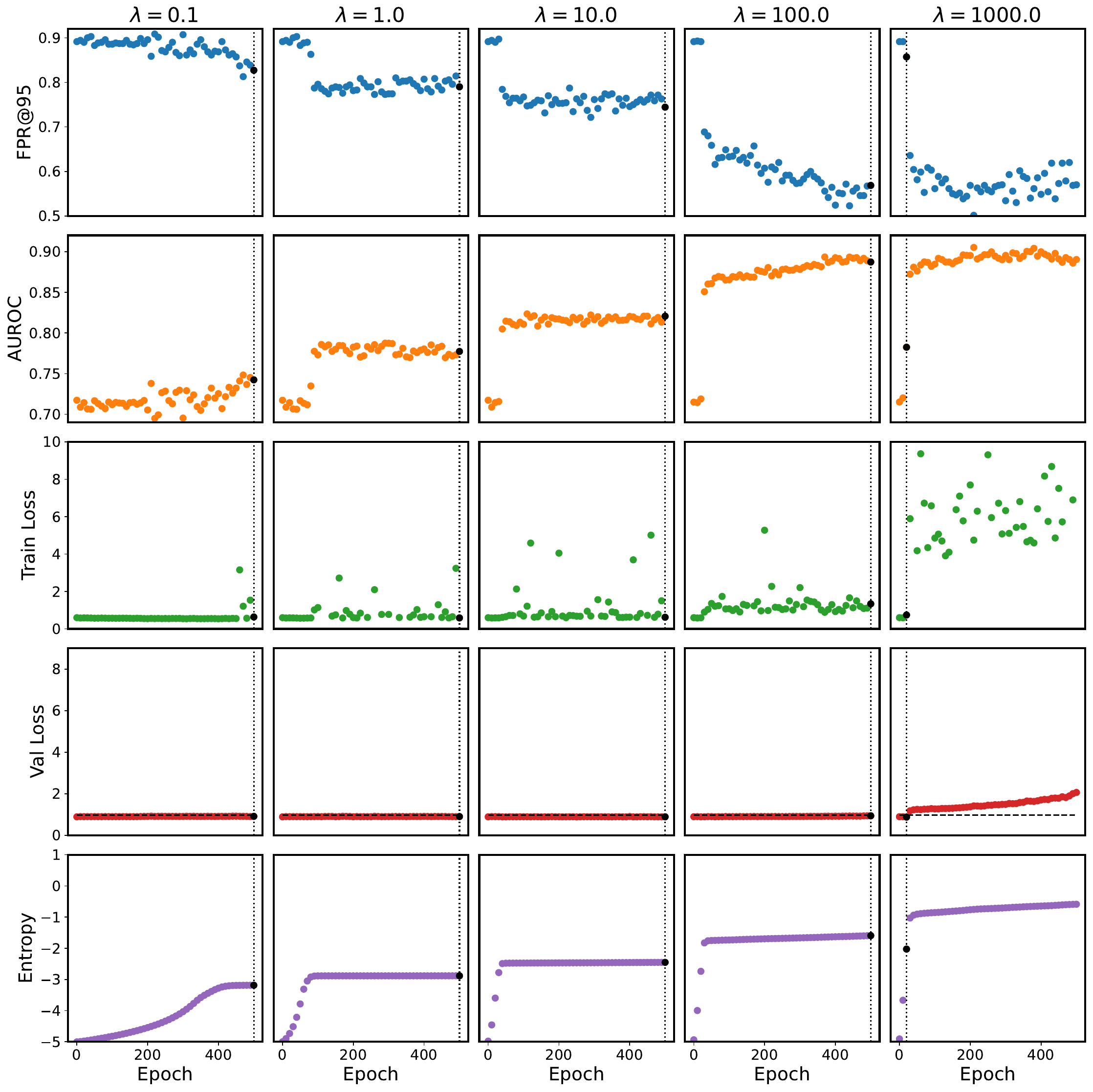}
\caption{Hyperparameter analysis of MaxWEntSVD for the BigBus-Shift experiment. Evolution of the five metrics: FPR@95, AUROC, training loss, validation loss, and entropy across epochs for different values of $\lambda$ (note that a scaling factor of $1/d$ is further applied to $\lambda$). The horizontal dashed lines in the validation loss plots represent the threshold $\tau$ and the dotted vertical lines indicate the stopping epoch (after which the validation loss exceeds $\tau$). We observe that higher $\lambda$ values lead to better FPR and AUROC, along with higher entropy. At a certain epoch, entropy increases rapidly, coinciding with improved OOD detection performance. The increase in validation loss is only observed for $\lambda = 1000$. In this setting, the OOD detection performance remains stable for late epochs although the training loss is unstable.}
\label{app-citycam-bigbus-svd}
\end{figure}

\clearpage

\begin{figure}[ht]
\center
\includegraphics[width=\linewidth]{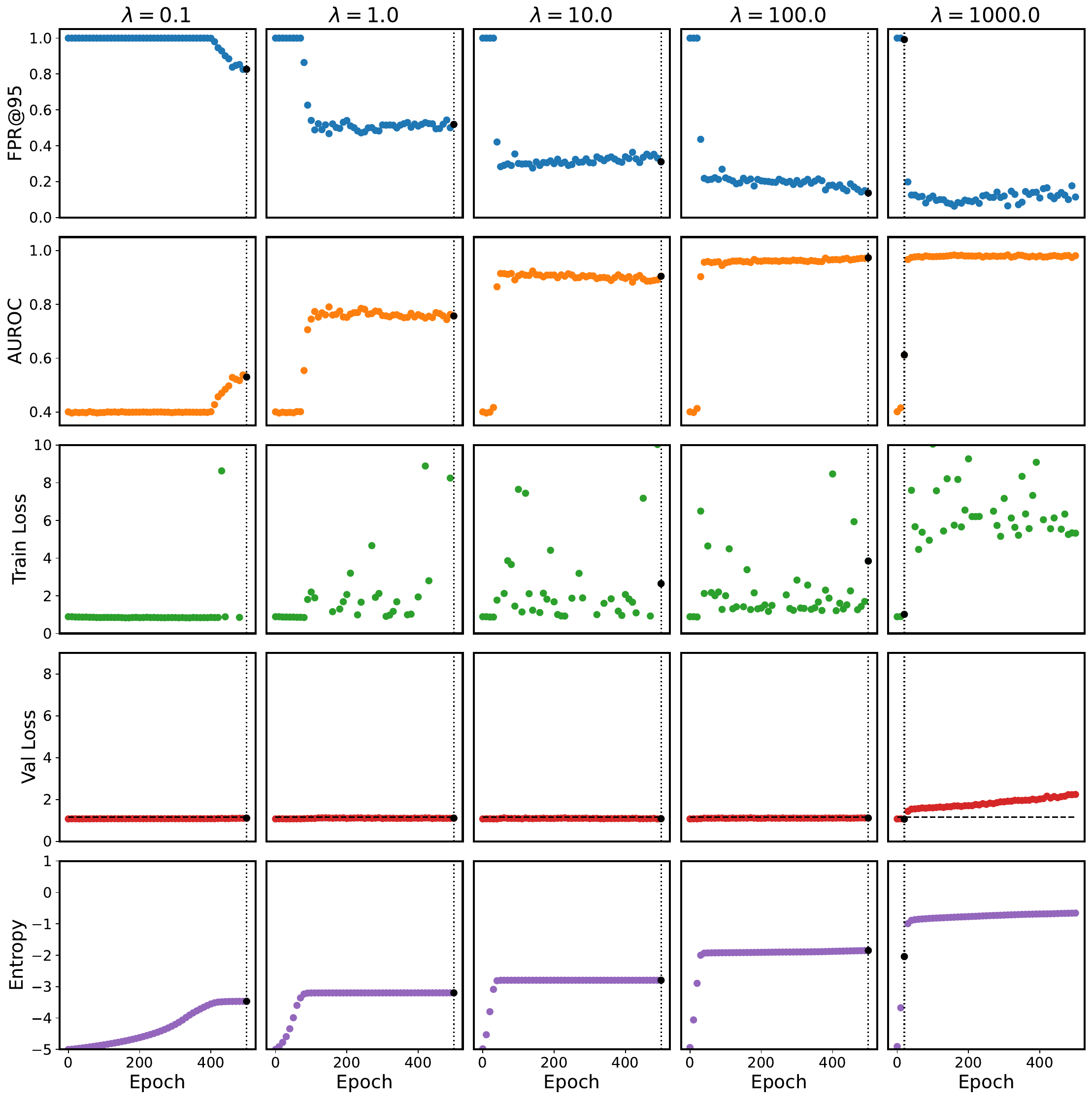}
\caption{Hyperparameter analysis of MaxWEntSVD for the Cameras-Shift experiment. Evolution of the five metrics: FPR@95, AUROC, training loss, validation loss, and entropy across epochs for different values of $\lambda$ (note that a scaling factor of $1/d$ is further applied to $\lambda$). The horizontal dashed lines in the validation loss plots represent the threshold $\tau$ and the dotted vertical lines indicate the stopping epoch (after which the validation loss exceeds $\tau$). We observe that higher $\lambda$ values lead to better FPR and AUROC, along with higher entropy. At a certain epoch, entropy increases rapidly, coinciding with improved OOD detection performance. The increase in validation loss is only observed for $\lambda = 1000$. In this setting, the OOD detection performance remains stable for late epochs although the training loss is unstable.}
\label{app-citycam-cameras-svd}
\end{figure}

\clearpage

The analysis of the hyperparameter impact across the three experiments Weather, BigBus, and Camera Shift (Figures \ref{app-citycam-weather-svd}, \ref{app-citycam-bigbus-svd}, \ref{app-citycam-cameras-svd}) demonstrates that larger values of $\lambda$ generally result in improved OOD detection performances. In the Weather Shift experiment, however, a significant variance in OOD detection metrics is observed for high $\lambda$ values during later epochs. Furthermore, for $\lambda = 1000$, the training loss is highly unstable across all three experiments, indicating that while larger $\lambda$ values tend to improve performance, excessively large values can lead to instability during training. Additionally, it should be noted that, in the Weather Shift experiment, the stopping criterion effectively selects an epoch with lower variance in OOD metrics and a reasonable validation loss.

The detailed metrics obtained at the stopping epoch are reported for all $\lambda$ values in Table \ref{tab:metrics:summary}. The results reveal that weight entropy is closely aligned with OOD performance: higher entropy corresponds to better AUROC and lower FPR. Although $\lambda = 1000$ yields the highest entropy after 500 epochs, the stopping criterion stops the training at an epoch where the entropy is lower than that for $\lambda = 100$. This observation suggests that hyperparameter selection should be guided by both entropy and validation loss. Specifically, users should select a $\lambda$ value that produces a model with the highest entropy, while ensuring the validation loss remains below the threshold $\tau$. Importantly, it should be highlighted that the lowest validation loss does not correspond to the best OOD detection performance.

In summary, while OOD metrics tend to improve with more epochs, the variance between epochs also increases in some cases. OOD performance and entropy generally increase with higher $\lambda$ values, but when $\lambda$ is too large, training becomes unstable and validation loss quickly exceeds the threshold. Therefore, the recommended approach is to use the stopping criterion to maintain a reasonable validation loss and select $\lambda$ based on the resulting entropy. A well-performing model is characterized by high entropy and reasonable validation loss. Both metrics can be measured without using OOD data, which enables a practical approach to select the hyperparameter $\lambda$.

\begin{table}[ht]
\centering
\footnotesize
\begin{tabular}{l|cccc|cccc|cccc}
\toprule
\multirow{2}{*}{\backslashbox{$\lambda$}{Exp}} & \multicolumn{4}{c|}{BigBus} & \multicolumn{4}{c|}{Cameras} & \multicolumn{4}{c}{Weather} \\
 & AUC & FPR95 & Ent. & Val & AUC & FPR95 & Ent. & Val & AUC & FPR95 & Ent. & Val \\
\midrule
0.1 & 0.74 & 0.83 & -3.19 & 0.91 & 0.53 & 0.83 & -3.47 & 1.11 & 0.8 & 0.61 & -3.36 & 1.21 \\
1.0 & 0.78 & 0.79 & -2.88 & 0.9 & 0.76 & 0.52 & -3.2 & 1.11 & 0.77 & 0.7 & -3.21 & 1.18 \\
10.0 & 0.82 & 0.74 & -2.45 & 0.89 & 0.9 & 0.31 & -2.8 & 1.09 & 0.85 & 0.69 & -2.55 & 1.35 \\
100.0 & \textbf{0.89} & \textbf{0.57} & \textbf{-1.59} & 0.94 & \textbf{0.97} & \textbf{0.14} & \textbf{-1.85} & 1.12 & \textbf{0.87} & \textbf{0.56} & \textbf{-1.66} & 1.22 \\
1000.0 & 0.78 & 0.86 & -2.03 & \textbf{0.88} & 0.61 & 0.99 & -2.04 & \textbf{1.07} & 0.82 & 0.69 & -2.01 & \textbf{1.15} \\
\bottomrule
\end{tabular}
\caption{Metrics for different values of $\lambda$ obtained at the stopping epoch. The table presents the AUC, FPR95, weight entropy (Ent.), and validation NLL (Val) for the three datasets BigBus, Cameras, and Weather. Lower FPR95 values and higher AUC scores indicate better OOD detection performance, while lower validation NLL account for better ID performance. The entropy measures the diversity of the weight distribution. We observe that the $\lambda$ providing the highest weight entropy at the stopping epoch achieves the best AUC and FPR. In contrast, the $\lambda$ corresponding to the lowest validation loss does not yield the best overall performance.}
\label{tab:metrics:summary}
\end{table}


\clearpage

\section{Achieving MaxWEnt's Behavior in Bayesian Neural Networks}
\label{app-bnn-setting}

In this section, we conduct an ablation study to identify the key components that enable a Bayesian Neural Network (BNN) with a Gaussian prior to exhibit behaviors similar to those of MaxWEnt and MaxWEnt-SVD. We identified three critical factors: frozen mean, large prior variance, and SVD parameterization. Experiments were performed on the synthetic classification dataset and the eight UCI regression datasets introduced in Section \ref{impl-choice}.

In this study, we consider a non-MC-Dropout BNN trained with stochastic variational inference \citep{hoffman2013stochasticVI} using the reparameterization trick \citep{Kingma2013VAE}. We use the KL-weighted ELBO optimization involving a $\lambda$ parameter to weight the KL term as follows:
\begin{equation*}
    \max_{\theta \in \mathbb{R}^D} \; {\mathbb{E}}_{q_{\theta}}\left[ \sum_{(x, y) \in \mathcal{S}} \log(p(y | h_w(x)) \right] - \lambda D_{\textnormal{KL}} \left( q_{\theta}(w), p(w) \right).
\end{equation*}
This formulation is commonly adopted in practice for the variational Bayes approach to Bayesian neural networks \citep{wenzel2020howgoodBayesPosterior}. The use of a trade-off parameter enables the direct comparison with the MaxWEnt objective. In the following, we adopt the same trade-off factor as for MaxWEnt, such that $\lambda = 10 \frac{n}{d}$. We consider a Gaussian isotropic prior $p(w)$ of variance $\sigma^2$ and a Gaussian posterior $q_{\theta}(w) \sim \mathcal{N}(\mu, \Sigma)$, with $\theta = (\mu, \Sigma)$ the posterior parameters.


The following settings are examined:

\begin{itemize} 
    \item \textbf{SVD vs. No SVD}: In both settings, the posterior variance $\Sigma$ is parameterized by the vector $\phi \in \mathbb{R}^d$, such that $\Sigma = \text{diag}(\phi^2)$ (No SVD) or $\Sigma = V \text{diag}(\phi^2) V^T$ (SVD), using the matrix $V$ defined in Section \ref{svd-param-section}.
    
    \item \textbf{Frozen vs. Unfrozen Mean}: For the frozen mean case, the posterior mean $\mu$ is trained along with $\Sigma$; the prior is centered on $0$ such that $p(w) \sim \mathcal{N}(0, \sigma^2 \text{Id}_d)$, which corresponds to the usual setting for BNNs. For the unfrozen mean case, the posterior mean is equal to the weights of a pretrained neural network, such that $\mu = \overline{w}$ is fixed during training, as implemented in MaxWEnt. Accordingly, the prior is set to $\mathcal{N}(\overline{w}, \sigma^2 \text{Id}_d)$.
    
    \item \textbf{Large vs. Small Prior Variance $\boldsymbol{\sigma}^2$}: We have highlighted in Section \ref{bayesian-discussion} that the MaxWEnt optimization is related to the BNN training with Gaussian prior of large variance. To showcase the impact of the prior variance we consider both setting $\sigma^2 \ll 1$ ($\sigma^2=0.1)$ and $\sigma^2 \gg 1$ ($\sigma^2 = 10^{10}$).
\end{itemize}

It should be noted that the SVD parameterization requires a frozen posterior mean. Hence, the implementation choices above results in six different setup combinations. For all experiments, all networks are trained during 500 epochs and implement the stopping criterion presented in Section \ref{imp-stopping-criterion}, providing a consistent basis for comparison.

\subsection{Synthetic Classification}

\begin{table}[!ht]
\small
\centering
\begin{tabular}{|l|l|c|c|c|}
\hline
 \multicolumn{2}{|c|}{} & \textbf{Low prior variance} & \textbf{Large prior variance} \\
 \hline
\multirow{2}{*}[-12.5ex]{\rotatebox[origin=c]{90}{\textbf{No SVD}}} & \rotatebox[origin=c]{90}{\textbf{Unfrozen Mean}} & \adjustbox{valign=c, margin=2mm}{\includegraphics[width=0.37\linewidth]{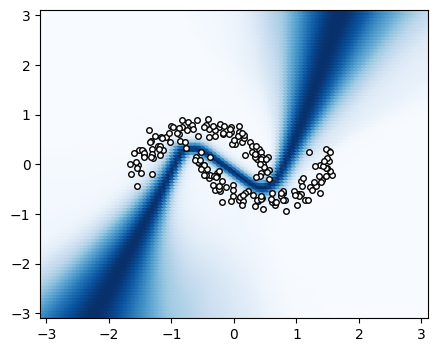}} & \adjustbox{valign=c, margin=2mm}{\includegraphics[width=0.37\linewidth]{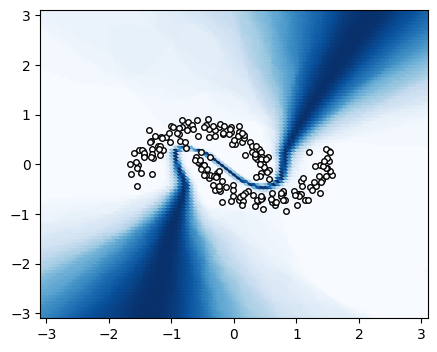}}  \\ \cline{2-4}
& \rotatebox[origin=c]{90}{\textbf{Frozen Mean}} & \adjustbox{valign=c, margin=2mm}{\includegraphics[width=0.37\linewidth]{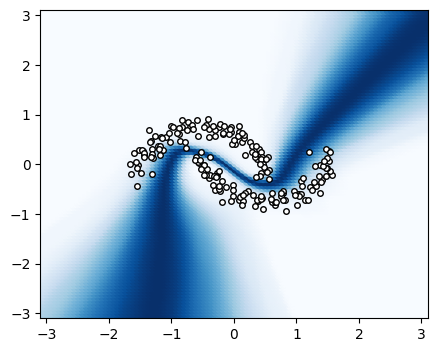}} & \adjustbox{valign=c, margin=2mm}{
\begin{tikzpicture}
    \node[inner sep=0pt] (image) at (0,0) {\includegraphics[width=0.37\linewidth]{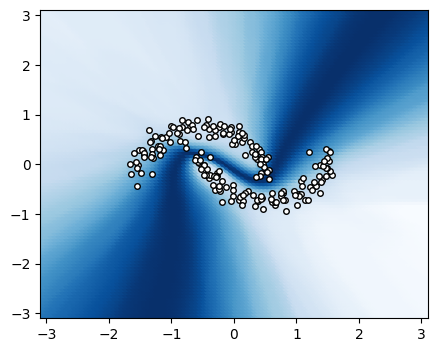}};
    \node[fill=white, draw=black, rounded corners, text width=4cm, align=center] at (0.2, 1.9) {MaxWEnt};
\end{tikzpicture}
} \\
\hline
\rotatebox[origin=c]{90}{\textbf{SVD}} & \rotatebox[origin=c]{90}{\textbf{Frozen Mean}} & \adjustbox{valign=c, margin=2mm}{\includegraphics[width=0.37\linewidth]{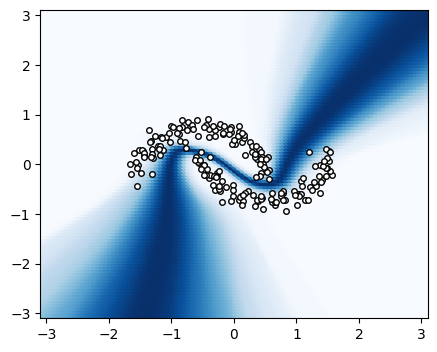}} & \adjustbox{valign=c, margin=2mm}{
\begin{tikzpicture}
    \node[inner sep=0pt] (image) at (0,0) {\includegraphics[width=0.37\linewidth]{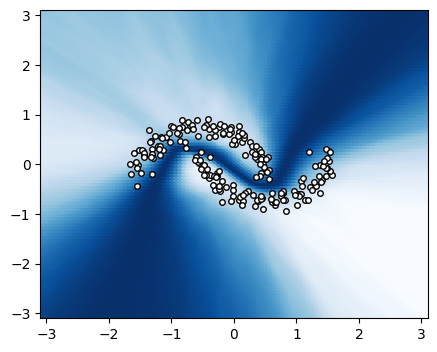}};
    \node[fill=white, draw=black, rounded corners, text width=4cm, align=center] at (0.2, 1.9) {MaxWEnt-SVD};
\end{tikzpicture}
} \\
\hline
\end{tabular}
\caption{Comparison of BNN settings on the synthetic classification dataset. White points represent training data, and shades of blue indicate uncertainty.}
\label{toy-abla-rez}
\end{table}

We consider the synthetic classification dataset described in Section \ref{synth-expe}. The results for the six setting combinations are reported in Figure \ref{toy-abla-rez}. The left and right columns respectively present the results obtained with low and large variance prior. We observe that for all settings enlarging the prior variance increases the contrast between in-distribution and out-of-distribution uncertainty. Additionally, the results highlight that the two settings ``frozen mean'' and SVD-parameterization, used for the MaxWEnt and MaxWEnt-SVD approaches, improve the epistemic uncertainty quantification.

\subsection{UCI datasets}

{\renewcommand{\arraystretch}{1.3}
\begin{table}[ht]
\centering
\small
\begin{tabular}{|l|l|c|c|c|c|c|c|}
\cline{3-8}
\multicolumn{2}{c|}{} & \multicolumn{4}{c|}{No SVD} & \multicolumn{2}{c|}{SVD} \\
\cline{3-8}
\multicolumn{2}{c|}{} & \multicolumn{2}{c|}{Unfrozen mean} & \multicolumn{2}{c|}{Frozen mean} & \multicolumn{2}{c|}{Frozen mean}  \\
\cline{3-8}
\multicolumn{2}{c|}{} & $\sigma^2 \ll 1$ & $\sigma^2 \gg 1$ & $\sigma^2 \ll 1$ & $\sigma^2 \gg 1$ & $\sigma^2 \ll 1$ & $\sigma^2 \gg 1$ \\
\hline
\multirow{2}{*}[-.5ex]{\rotatebox[origin=c]{90}{AUC}} & extrapol & 0.659 & 0.692 & 0.761 & 0.775 & 0.795 & \textbf{0.843} \\
& interpol & 0.450 & 0.556 & 0.524 & 0.559 & 0.656 & \textbf{0.783} \\
\hline
\multirow{2}{*}[-.5ex]{\rotatebox[origin=c]{90}{FPR}} & extrapol & 0.698 & 0.639 & 0.594 & 0.576 & 0.503 & \textbf{0.407} \\
& interpol & 0.950 & 0.897 & 0.895 & 0.852 & 0.765 & \textbf{0.535} \\
\hline
\multirow{2}{*}[-.5ex]{\rotatebox[origin=c]{90}{Val}} & extrapol & -0.318 & -0.100 & -0.638 & -0.623 & \textbf{-0.649} & -0.62 \\
& interpol & -0.169 & 0.633 & -0.359 & -0.342 & \textbf{-0.409} & -0.371 \\
\hline
\end{tabular}

\caption{Comparison of the BNN settings for the UCI experiments. The table presents the average OOD detection metrics (AUC, FPR) and Validation NLL (Val) across the 8 UCI datasets for interpolation (interpol) and extrapolation (extrapol) settings. Higher AUC and lower FPR are better. Columns compare BNN settings with and without SVD parameterization and ``frozen mean'', under small ($\sigma^2 \ll 1$) and large ($\sigma^2 \gg 1$) prior variances.}
\label{tab:uci:bnn}
\end{table}
}

Table \ref{tab:uci:bnn} presents the OOD detection performance for different BNN settings on the 8 UCI datasets. We report our observations as follows:
\begin{itemize}
    \item \textbf{Effectiveness of large variance prior.} For frozen and unfrozen mean, SVD and No SVD settings, the use of a large variance prior instead of a low variance significantly improves the OOD detection performances (e.g., for the interpolation experiments with the SVD parameterization the FPR drops from 0.765 to 0.535).
    \item \textbf{Impact of frozen mean.} Freezing the mean improves the OOD detection performances compared to the unfrozen mean setting. For instance, with large prior variance in the extrapolation experiments, the FPR decreases from 0.639 to 0.576 when using a frozen mean. A similar effect is observed in the synthetic experiment, where the uncertainty between the two classes is lower for the unfrozen mean setting. A possible explanation to this phenomenon is that using a frozen mean shifts the training focus to the scale parameters, which are directly related to weight entropy, hence promoting an increase in entropy.    
    \item \textbf{Impact of SVD parameterization.} Incorporating the SVD parameterization to the BNN significantly increase the OOD detection performances, even with low prior variance. Further improvements are obtained when enlarging the prior variance.
    \item \textbf{The validation NLL is smaller when using the low variance prior.} The use of a large prior consistently degrades the validation NLL. For example, in the Unfrozen mean case, NLL increases from -0.169 to 0.633 in the interpolation setting. This observation highlights that validation NLL is not a suitable metric for promoting weight entropy and improving OOD performances. Interestingly, the best validation NLL is obtained with the SVD parameterization and low prior variance, suggesting that traditional BNNs could benefit from SVD parameterization to improve in-distribution performance.
\end{itemize}
In conclusion, enlarging the prior variance, and then the weight entropy, effectively improves OOD detection performances in all settings. While the SVD parameterization with small prior variance already improves OOD detection performance, increasing the variance further yields a significant improvement, especially in the challenging interpolation setting.

\end{document}